\documentclass{article} 
\usepackage{hyperref}
\usepackage{url}

\usepackage{graphicx} 
\usepackage{subfigure}
\usepackage{xspace}
\usepackage{caption}

\usepackage{wrapfig}
\usepackage{lipsum}

\usepackage[numbers]{natbib}

\usepackage{algorithm}
\usepackage{algorithmic}

\usepackage{fullpage}
\usepackage{authblk}

\usepackage{enumitem}

\usepackage{Definitions}
\usepackage{color}

\renewcommand{\leq}{\leqslant}
\renewcommand{\le}{\leqslant}
\renewcommand{\geq}{\geqslant}
\renewcommand{\ge}{\geqslant}

\usepackage{color}

\usepackage[flushleft]{threeparttable}

\title{Provable Bayesian Inference via Particle Mirror Descent}

\author{
    Bo Dai$^1$, Niao He$^2$, Hanjun Dai$^1$, Le Song$^1$\\
    $^1$ Georgia Institute of Technology\\
    \{bodai, hanjundai\}@gatech.edu, lsong@cc.gatech.edu\\
    $^2$ University of Illinois at Urbana-Champaign\\
    niaohe@illinois.edu\\
}

\begin{document}

\maketitle

\begin{abstract} 
Bayesian methods are appealing in their flexibility in modeling complex data and ability in capturing uncertainty in parameters. However, when Bayes' rule does not result in tractable closed-form, most approximate inference algorithms lack either scalability or rigorous guarantees. To tackle this challenge, we propose a simple yet provable algorithm, \emph{Particle Mirror Descent} (PMD), to iteratively approximate the posterior density.  PMD is inspired by stochastic functional mirror descent where one descends in the density space using a small batch of data points at each iteration, and by particle filtering where one uses samples to approximate a function. We prove result of the first kind that, with $m$ particles, PMD provides a posterior density estimator that converges in terms of $KL$-divergence to the true posterior in rate $O(1/\sqrt{m})$. We demonstrate competitive empirical performances of PMD compared to several approximate inference algorithms in mixture models, logistic regression, sparse Gaussian processes and latent Dirichlet allocation on large scale datasets. 
\end{abstract}

\clearpage
\newpage

\tableofcontents
\addtocontents{toc}{\protect\setcounter{tocdepth}{2}} 

\clearpage
\newpage

\section{Introduction}\label{sec:intro}

Bayesian methods are attractive because of their ability in modeling complex data and capturing uncertainty in parameters. The crux of Bayesian inference is to compute the posterior distribution, $p(\theta|X) \propto p(\theta) \prod\nolimits_{n=1}^N p(x_n|\theta)$, of a parameter $\theta \in \RR^d$ given a set of $N$ data points $X = \cbr{x_n}_{n=1}^N$ from $\RR^D$, with a prior distribution $p(\theta)$ and a model of data likelihood $p(x|\theta)$. For many non-trivial models from real-world applications, the prior might not be conjugate to the likelihood or might contain hierarchical structure. Therefore, computing the posterior often results in intractable integration and poses computational challenges. Typically, one resorts to approximate inference such as sampling, \eg, MCMC~\cite{Neal93} and SMC~\cite{DouFreGor01}, or variational inference~\cite{JorGhaJaaSau98, WaiJor08}. 

Two longstanding challenges in approximate Bayesian inference are i) \emph{provable convergence} and  ii) \emph{data-intensive computation} at each iteration. MCMC is a general algorithm known to generate samples from distribution that converges to the true posterior. However, in order to generate a single sample at every iteration, it requires a complete scan of the dataset and evaluation of the likelihood at each data point, which is computationally expensive. To address this issue, approximate sampling algorithms have been proposed which use only a small batch of data points at each iteration~\citep[\eg][]{Chopin02, BalMad06, WelTeh11, MacAda14}. \citet{Chopin02, BalMad06} extend the sequential Monte Carlo~(SMC) to Bayesian inference on static models. However, these algorithms rely on Gaussian distribution or kernel density estimator as transition kernel for efficiency, which breaks down the convergence guarantee of SMC. On the other hand, the stochastic Langevin dynamics algorithm~(SGLD)~\cite{WelTeh11} and its derivatives~\cite{AhnKorWel12, ChenFoxGue14, DinFanBabCheetal14} combine ideas from stochastic optimization and Hamiltonian Monte Carlo, and are proven to converge  in terms of integral approximation, as recently shown in~\cite{TehThiVol2014a,TehThiVol2015}. Still, it is unclear whether the \emph{dependent} samples generated  reflects convergence to the true posterior. FireflyMC~\cite{MacAda14}, introduces auxiliary variables to switch on and off data points to save computation for likelihood evaluations, but this algorithm requires the knowledge of lower bounds of  likelihood that is model-specific and may be hard to calculate. 

In another line of research, the variational inference algorithms~\cite{JorGhaJaaSau98,WaiJor08, Minka01} attempt to approximate the entire posterior density by optimizing information divergence~\cite{Minka05}. The recent derivatives~\cite{HofBleWanPai13} avoid examination of all the data in each update. However, the major issue for these algorithms is the absence of theoretical guarantees. This is due largely to the fact that variational inference algorithms typically choose a parametric family to approximate the posterior density, which can be far from the true posterior, and require to solve a highly non-convex optimization problem. In most cases, these algorithms optimize over simple exponential family for tractability. More flexible variational families have been explored but largely restricted to mixture models~\cite{JaaJor98, GerHofBle12}. In these cases, it is often difficult to quantify the approximation and optimization error at each iteration, and analyze how the error accumulates across the iterations. Therefore, a provably convergent variational inference algorithm is still needed. 

In this paper, we present such a simple and provable nonparametric inference algorithm, \emph{Particle Mirror Descent}~(PMD), to iteratively approximate the posterior density. PMD relies on the connection that Bayes' rule can be expressed as the solution to a convex optimization problem over the density space~\cite{Williams80, Zellner88, ZhuCheXin14}. However, directly solving the optimization will lead to both computational and representational issues: one scan over the entire dataset at each iteration is needed, and the exact function update has no closed-form. To address these issues, we draw inspiration from two sources: ({\bf i}) stochastic mirror descent, where one can instead descend in the density space using a small batch of data points at each iteration; and ({\bf ii}) particle filtering and kernel density estimation, where one can maintain a tractable approximate representation of the density using samples. In summary, PMD possesses a number of desiderata:

\noindent  {\bf Simplicity.} PMD applies to many probabilistic models, even with \emph{non-conjugate priors}. The algorithm is summarized in just a few lines of codes, and only requires the value of likelihood and prior, unlike other approximate inference techniques~\citep[\eg]{WelTeh11, GerHofBle12, PaiBleJor12, HofBleWanPai13}, which typically require their first and/or second-order derivatives.

\noindent  {\bf Flexibility.} Different from other variational inference algorithms, which sacrifice the model flexibility for tractability, our method approximates the posterior by particles or kernel density estimator. The flexibility of nonparametric model enables PMD to capture multi-modal in posterior.

\noindent  {\bf Stochasticity.} At iteration $t$, PMD only visits a mini-batch of data to compute the stochastic functional gradient, and samples $O(t)$ points from the solution. Hence, it avoids scanning over the whole dataset in each update.

\noindent  {\bf Theoretical guarantees.} We show the density estimator provided by PMD converges in terms of both integral approximation and $KL$-divergence to the true posterior density in rate $O(1/\sqrt{m})$ with $m$ particles. To our best knowledge, these results are the first of the kind in Bayesian inference for estimating posterior.

In the remainder, we will introduce the optimization view of Bayes' rule before presenting our algorithm, and then we provide both theoretical and empirical supports of PMD. 

Throughout this paper, we denote $KL$ as the  Kullback-Leibler divergence, function $q(\theta)$ as $q$, a random sequence as $\theta_{[t]}:=[\theta_1,\ldots,\theta_t]$, integral $f(\cdot)$ w.r.t. some measure $\mu(\theta)$ over support $\Omega$ as $\int f(\theta)\mu(d\theta)$, or $\int f(\theta)d\theta$ without ambiguity, $\langle\cdot,\cdot\rangle_{L_2}$ as the $L_2$ inner product, and $\|\cdot\|_p$ as the $L_p$ norm for $1\leq p\leq\infty$.

\section{Optimization View of Bayesian Inference}\label{sec:background}

Our algorithm stems from the connection between Bayes' rule and optimization. \citet{Williams80, Zellner88, ZhuCheXin14} showed that Bayes' rule
\vspace{-2mm}
$$
p(\theta |  {X}) = \frac{p(\theta)\prod_{n=1}^Np( {x}_n|\theta)}{p( {X})}
$$
where $ p( {X}) = \int p(\theta)\prod_{n=1}^Np( {x}_n|\theta) d\theta$, 
 can be obtained by solving the optimization problem
\vspace{-2mm}
\begin{align}\label{eq:bayes_opt}
  \min_{q(\theta)\in\mathcal{P}}L(q) :=- \sum_{n=1}^N\bigg[\int q(\theta)\log p( {x}_n|\theta)\, d\theta\bigg]+KL(q(\theta)\,||\,p(\theta)),
\end{align}
where $\mathcal{P}$ is the valid density space. The objective, $L(q)$, is continuously differentiable with respect to $q \in \mathcal{P}$ and one can further show that 
\vspace{-2mm}
\begin{lemma}\label{lem:strongconvexity}
Objective function $L(q)$ defined on $q(\theta)\in \mathcal{P}$ is $1$-strongly convex w.r.t. $KL$-divergence.\vspace{-2mm}
\end{lemma}
Despite of the closed-form representation of the optimal solution, it can be challenging to compactly represent, tractably compute, or efficiently sample from the solution. The normalization, $p( {X}) = \int p(\theta)\prod_{n=1}^Np( {x}_n|\theta) d\theta$, involves high dimensional integral and typically does not admit tractable closed-form computation. Meanwhile, the product in the numerator could be arbitrarily complicated, making it difficult to represent and sample from. However, this optimization perspective provides us a way to tackle these challenges by leveraging recent advances from optimization algorithms.

\subsection{Stochastic Mirror Descent in Density Space}\label{subsec:exact_smd}

We will resort to stochastic optimization to avoid scanning the entire dataset for each gradient evaluation. The stochastic mirror descent~\cite{NemJudLanSha09} expands the usual stochastic gradient descent scheme to problems with non-Euclidean geometries, by applying unbiased stochastic subgradients and Bregman distances as prox-map functions. We now explain in details, the stochastic mirror descent algorithm in the context of Bayesian inference. 

At $t$-th iteration, given a data point $x_t$ drawn randomly from the dataset, the stochastic functional gradient of $L(q)$ with respect to $q(\theta)\in L_2$ is $g_t(\theta) = \log(q(\theta)) - \log(p(\theta)) - N\log p(x_t|\theta)$. The stochastic mirror descent iterates over the prox-mapping step
$ q_{t+1}= \mathbf{P}_{q_t}(\gamma_t g_t)$, where $\gamma_t>0$ is the stepsize and
\vspace{-1mm}
\begin{eqnarray*}
  \mathbf{P}_{q}(g):=\argmin\nolimits_{\qhat(\theta) \in \Pcal} \, \{\inner{\qhat}{g}_{L_2} + KL(\qhat\| q)\}.  
\end{eqnarray*}
Since the domain is  density space, $KL$-divergence is a natural choice for the prox-function. The prox-mapping therefore admits the closed-form
\vspace{-2mm}
\begin{eqnarray}\label{eq:solve_exact_prox}
q_{t+1}(\theta)&=&q_{t}(\theta)\exp(-\gamma_t g_t(\theta))/Z\\\nonumber
&=& q_t(\theta)^{1-\gamma_t}p(\theta)^{\gamma_t} p(x_t|\theta)^{N\gamma_t}/Z,
\end{eqnarray}
where $Z:=\int q_{t}(\theta)\exp(-\gamma_t g_t(\theta))\, d\theta$ is the normalization.  This update is similar the Bayes' rule. However, an important difference here is that the posterior is updated using the \emph{fractional power} of the previous solution, the prior and the likelihood. Still computing $q_{t+1}(\theta)$  can be  intractable due to the normalization $Z$.    

\subsection{Error Tolerant Stochastic Mirror Descent}\label{approx:exact_smd}
To handle the intractable integral normalization at each prox-mapping step, we will consider a modified
version of the stochastic mirror descent algorithm which can tolerate additional error in the prox-mapping step. 
Given $\epsilon\geqslant 0$ and $g\in L_2$, we define the  $\epsilon$-prox-mapping of $q$ as the set
\vspace{-1mm}
\begin{align}\label{eq:inexact_prox11}
\mathbf{P}_{q}^\epsilon(g):=\{\qhat\in \mathcal{P}\, :\, KL(\qhat||q)+\langle g, \qhat\rangle_{L_2} \leqslant \min\nolimits_{\qhat\in\mathcal{P}}  \{KL(\qhat||q)+\langle g, \qhat\rangle_{L_2}\}+\epsilon\}, 
\end{align}
and consider the update $\qtil_{t+1}(\theta)\in \mathbf{P}_{\qtil_t}^{\epsilon_t}(\gamma_t g_t)$. When $\epsilon_t=0, \forall t$, this reduces to the usual stochastic mirror descent algorithm. The classical results regarding the convergence rate can also be extended as below
\begin{theorem}\label{thm:main1}
\vspace{-2mm}
Let $q^*=\argmin_{q\in\mathcal{P}} L(q)$, stochastic mirror descent with inexact prox-mapping after $T$ steps gives the recurrence: 
\vspace{-3mm}
$$
\EE[KL(q^* ||  \qtil_{t+1})] \leqslant \epsilon_t+(1-\gamma_t) \EE[KL(q^* ||  \qtil_t)] + \frac{\gamma_t^2}{2}\EE\|g_t\|^2_\infty
$$ 
\end{theorem}
\vspace{-3mm}

\noindent{\bf Remark 1:} 
As shown in the classical analysis of stochastic mirror descent, we could also provide a non-asymptotic convergence results in terms of objective error at average solutions, \eg, simple average $\bar q_T=\sum_{t=1}^T\gamma_t  \qtil_t/\sum_{t=1}^T\gamma_t$ in Appendix~\ref{appendix:inexact_prox}.

\noindent{\bf Remark 2:} For simplicity, we present the algorithm with stochastic gradient estimated by a single data point. The mini-batch trick is also applicable to reduce the variance of stochastic gradient, and convergence remains the same order but with an improved constant.

Allowing error in each step gives us room to design more flexible algorithms. Essentially, this implies that we can \emph{approximate} the intermediate density by some \emph{tractable representation}. As long as the approximation error is not too large, the algorithm will still converge; and if the approximation does not involve costly computation, the overall algorithm will still be efficient. 

\section{Particle Mirror Descent Algorithm}\label{sec:algorithm}

We  introduce two efficient strategies to approximate prox-mappings, one based on weighted particles and the other based on weighted kernel density estimator. The first strategy is designed for the situation when the prior is a ``good" guess of the true posterior, while the second strategy works for general situations. Interestingly, these two methods resemble particle reweighting and rejuvenation respectively in sequential Monte Carlo yet with notable differences.

\subsection{Posterior Approximation Using Weighted Particle}\label{subsec:weight_particle}

We first consider the situation when we are given a ``good" prior, such that  $p(\theta)$ has the same support as the true posterior $q^*(\theta)$, \ie, $0\le q^*(\theta)/p(\theta)\le C$. We will simply maintain a set of samples (or particles) from $p(\theta)$, and utlize them to estimate the intermediate prox-mappings. Let $\cbr{\theta_i}_{i=1}^m\sim p(\theta)$ be a set of fixed \iid\,samples. We approximate $q_{t+1}(\theta)$ as a set of weighted particles 
\begin{eqnarray}\label{eq:weighted_particle}
&\tilde q_{t+1}(\theta) = \sum\nolimits_{i=1}^m \alpha^{t+1}_i\, \delta(\theta_i),\\
&\alpha^{t+1}_i:=\frac{\alpha^{t}_i \exp(-\gamma_t g_t(\theta_i))}{\sum_{i=1}^m \alpha^{t}_i \exp(-\gamma_t g_t(\theta_i))},  \forall t\geq 1 .\nonumber
\end{eqnarray}
The update is derived from the closed-form solution to the \emph{exact} prox-mapping step~\eq{eq:solve_exact_prox}. Since the normalization is a constant common to all components, one can simply update the set of working variable $\alpha_i$ as
\vspace{-1mm}
\begin{align}
  \alpha_i &\leftarrow \alpha_i^{1-\gamma_t} p(x_t|\theta_i)^{N\gamma_t},\forall i\\\nonumber
  \alpha_i &\leftarrow \frac{\alpha_i}{\sum_{i=1}^m \alpha_i}.
\end{align}
We show that the one step approximation~(\ref{eq:weighted_particle}) incurs a dimension-\emph{independent} error when estimating the integration of a function. 
\begin{theorem}\label{thm:discrete_approximation_error}
\vspace{-2mm} 
For any bounded and integrable function $f$,  
$
\EE\sbr{\abr{\int \qtil_t(\theta)f(\theta)d\theta - \int q_t(\theta)f(\theta)d\theta}}\le \frac{2C\|f\|_\infty}{\sqrt{m}}. 
$
\vspace{-2mm}
\end{theorem}

\paragraph{Remark.} Please refer to the Appendix~\ref{appendix:integral_convgence} for details. When the model has several latent variables $\theta = (\xi, \zeta)$ and some parts of the variables have closed-form update in~\eq{eq:solve_exact_prox}. \eg, sparse GPs and LDA~(refer to Appendix~\ref{appendix:GP_LDA}), we could incorporate such structure information into algorithm by decomposing the posterior $q(\theta) = q(\xi)q(\zeta|\xi)$. When $p(\xi)$ satisfies the condition, we could sample $\{\xi_i\}_{i=1}^m\sim p(\xi)$ and approximate the posterior with summation of several functions, i.e., in the form of $q(\theta) \approx \sum{\alpha_i}q(\zeta|\xi_i)$.

\subsection{Posterior Approximation Using Weighted Kernel Density Estimator}\label{subsec:weight_kde}

In general, sampling from prior $p(\theta)$ that are not so ``good" will lead to particle depletion and inaccurate estimation of the posterior. To alleviate particle degeneracy, we propose to estimate the prox-mappings via weighted kernel density estimator~(KDE). The weighted KDE prevents particles from dying out, in a similar fashion as kernel smoothing variant SMC~\citep{DouFreGor01} and one-pass SMC~\citep{BalMad06}, but with guarantees. 

More specifically, we approximate $q_{t+1}(\theta)$ via a weighted kernel density estimator 
\vspace{-2mm}
\begin{align}\label{eq:weighted_kde}
&\qtil_{t+1}(\theta) =\sum\nolimits_{i=1}^m \alpha_i\, K_h(\theta-\theta_i),\\
&\alpha_i:=\frac{\exp(-\gamma_t g_t(\theta_i))}{\sum_{i=1}^m \exp(-\gamma_t g_t(\theta_i))},~~\cbr{\theta_i}_{i=1}^m\overset{\iid}{\sim} \qtil_t(\theta), \nonumber
\end{align}
where $h>0$ is the bandwidth parameter and $K_h(\theta):=\frac{1}{h^d}K(\theta/h)$ is a smoothing kernel. The update serves as an $\epsilon$-prox-mapping~(\ref{eq:inexact_prox11}) based on the closed-form solution to the \emph{exact} prox-mapping step~\eq{eq:solve_exact_prox}. Unlike the first strategy, the particle location in this case is sampled from the previous solution $\qtil_{t}(\theta)$. The idea here is that $\qtil^+_{t}(\theta)=\qtil_{t}(\theta) \exp(-\gamma_t g_t(\theta)) / Z$ can be viewed as an importance weighted version of $\qtil_{t}(\theta)$ with weights equal to $\exp(-\gamma_t g_t(\theta)) / Z$. If we want to approximate $\qtil^+_{t}(\theta)$, we can sample $m$ locations from $\tilde q_{t}(\theta)$ and associate each location the normalized weight $\alpha_i$. To obtain a density for re-sampling in the next iteration, we place a kernel function $K_h(\theta)$ on each sampled location. Since $\alpha_i$ is a ratio, we can avoid evaluating the normalization factor $Z$ when computing $\alpha_i$. In summary, we can simply update the set of working variable $\alpha_i$ as
\vspace{-1mm}
\begin{align}
  \alpha_i &\leftarrow \tilde q_{t}(\theta_i)^{-\gamma_t}p(\theta_i)^{\gamma_t} p(x_t|\theta_i)^{N\gamma_t},\forall i\\\nonumber
  \alpha_i &\leftarrow \frac{\alpha_i}{\sum_{i=1}^m \alpha_i}.
\end{align}
Intuitively, the sampling procedure gradually adjusts the support of the intermediate distribution towards that of the true posterior, which is similar to ``rejuvenation'' step. The reweighting procedure gradually adjusts the shape of the intermediate distribution on the support. Same as the mechanism in~\citet{DouFreGor01, BalMad06}, the weighted KDE could avoid particle depletion.

We demonstrate that the estimator in~\eq{eq:weighted_kde} in one step possesses similar estimation properties as standard KDE for densities~(for details, refer to the Appendix~\ref{appendix:bound_weighted_kde}).
\begin{theorem}\label{thm:error1}
Let $q_t$ be a $(\beta;\mathcal{L})$-H\"{o}lder density function, and $K$ be a $\beta$-valid density kernel, and the kernel bandwidth chosen as $h=O(m^{-\frac{1}{d+2\beta}})$. Then, under some mild conditions, $\EE \nbr{\qtil_t(\theta) - q_t(\theta)}_1 = O(m^{-\frac{\beta}{d+2\beta}})$.
\end{theorem}

A kernel function $K(\cdot)$ is called $\beta$-valid, if $\int z^sK(z)dz=0$ holds true for any $s=(s_1,\ldots,s_d)\in\NN^d$ with $|s|\leq \lfloor{\beta}\rfloor$. Notice that all spherically symmetric and product kernels  satisfy the condition. For instance, the Gaussian kernel $K(\theta) = (2\pi)^{-d/2}\exp(-\nbr{\theta}^2/2)$ satisfies the condition with $\beta = 1$, and it is used throughout our experiments. Theorem \ref{thm:error1} implies that the weighted KDE achieves the \emph{minmax} rate for density estimation in $({\beta}; {\mathcal{L}})$-H\"{o}lder function class~\cite{DelJud95}, where $\beta$ stands for the smoothness parameter and $\mathcal{L}$ is the corresponding Lipschitz constant. With further assumption on the smoothness of the density, the weighted KDE can achieve even better rate. For instance, if $\beta$ scales linearly with dimension, the error of weighted KDE can achieve a rate independent of the dimension. 

Essentially, the weighted KDE step  provides an $\epsilon$-prox-mapping $\mathbf{P}_{\qtil_t}^{\epsilon_t}(\gamma_t g_t)$~(\ref{eq:inexact_prox11}) in density space as we discussed in Section~\ref{sec:background}. The inexactness is therefore determined by the number of samples $m_t$ and kernel bandwidth $h_t$ used in the weighted KDE.

\subsection{Overall Algorithm}
\vspace{-6mm}
\begin{figure}[H]
\begin{minipage}{0.5\textwidth}
We present the overall algorithm, Particle Mirror Descent~(PMD), in Algorithm~\ref{alg:pmd}. The algorithm is based on stochastic mirror descent incorporated with  two strategies from section~\ref{subsec:weight_particle} and \ref{subsec:weight_kde} to compute prox-mapping.  PMD takes as input $N$ samples $X=\cbr{x_n}_{n=1}^N$, a prior $p(\theta)$ over the model parameter and the likelihood $p(x|\theta)$, and outputs the posterior density estimator $\qtil_{T}(\theta)$ after $T$ iterations. At each iteration, PMD takes the stochastic functional gradient information and computes an inexact prox-mapping $\qtil_t(\theta)$ through either weighted particles or weighted kernel density estimator.  Note that as discussed in Section~\ref{sec:background}, we can also take a batch of points at each iteration to compute the stochastic gradient in order to reduce variance.

In Section~\ref{sec:theory}, we will show that, with proper setting of stepsize $\gamma$, Algorithm~\ref{alg:pmd} converges in rate $O({1}/{\sqrt{m}})$ using $m$ particles, in terms of either integral approximation or $KL$-divergence, to the true posterior.
\end{minipage}
~~~~~
\begin{minipage}{0.45\textwidth}
\setlength{\abovedisplayskip}{4pt}
\setlength{\abovedisplayshortskip}{1pt}
\setlength{\belowdisplayskip}{4pt}
\setlength{\belowdisplayshortskip}{1pt}
\setlength{\textfloatsep}{4pt}
\begin{algorithm}[H]
  \caption{Particle Mirror Descent Algorithm}\label{alg:pmd}
  \begin{algorithmic}[1]
  \STATE {\bf Input}: Data set $X=\cbr{x_n}_{n=1}^N$, prior $p(\theta)$ 
  \STATE {\bf Output}: posterior density estimator $\tilde q_{T}(\theta)$ 
  \STATE Initialize $\qtil_1(\theta) = p(\theta)$
  \FOR{$t=1,2,\ldots,T-1$} 
    \STATE $x_t \overset{unif.}{\sim} X$
    \IF {Good $p(\theta)$ is provided}
      \STATE $\{\theta_i\}_{i=1}^{m_t} \overset{\iid}{\sim} p(\theta)$ when $t=1$
      \STATE $\alpha_i \leftarrow \alpha_i^{1-\gamma_t} p(x_t|\theta_i)^{N\gamma_t},\forall i$
      \STATE $\alpha_i \leftarrow \frac{\alpha_i}{\sum_{i=1}^{m_t} \alpha_i},\forall i$
      \STATE $\qtil_{t+1}(\theta) = \sum\nolimits_{i=1}^{m_t} \alpha_i\, \delta(\theta_i)$ 
    \ELSE
      \STATE $\{\theta_i\}_{i=1}^{m_t} \overset{\iid}{\sim} \qtil_{t}(\theta)$
      \STATE $\alpha_i \leftarrow \tilde q_t(\theta_i)^{-\gamma_t}p(\theta_i)^{\gamma_t} p(x_t|\theta_i)^{N\gamma_t},\forall i$
      \STATE $\alpha_i \leftarrow \frac{\alpha_i}{\sum_{i=1}^{m_t} \alpha_i},\forall i$
      \STATE $\qtil_{t+1}(\theta) = \sum\nolimits_{i=1}^{m_t} \alpha_i K_{h_t}\rbr{\theta-\theta_i}$
    \ENDIF
   \ENDFOR
  \end{algorithmic}
  \end{algorithm}
\end{minipage}

\end{figure}

In practice, we could combine the proposed two algorithms to reduce the computation cost. In the beginning stage, we adopt the second strategy. The computation cost is affordable for small number of particles. After we achieve a reasonably good estimator of the posterior, we could switch to the first strategy using large size particles to get better rate.

\section{Theoretical Guarantees}\label{sec:theory}

In this section, we show that PMD algorithm  ({\bf i}) given good prior $p(\theta)$, achieves a dimension independent, sublinear rate of convergence in terms of integral approximation; and ({\bf ii}) in general cases, achieves a  dimension dependent, sublinear rate of convergence in terms of $KL$-divergence with proper choices of stepsizes. 

\subsection{Weak Convergence of PMD}

The weighted particles approximation, $\tilde q_t(\theta) = \sum_{i=1}^m \alpha_i\, \delta(\theta_i)$, returned by Algorithm~\ref{alg:pmd} can be used directly for Bayesian inference. That is, given a function $f$, $\int q^*(\theta) f(\theta) d\theta$ can be approximated as $\sum_{i=1}^m \alpha_i f(\theta_i)$. We will analyze its ability in approximating integral, which is commonly used in sequential Monte Carlo for dynamic models~\cite{CriDou02} and stochastic Langevin dynamics~\cite{TehThiVol2015}. For simplicity, we may write $\sum_{i=1}^m \alpha_i f(\theta_i)$ as $\int \tilde q_t(\theta) f(\theta) d\theta$, despite of the fact that $\tilde q_t(\theta)$ is not exactly a density here. We show a sublinear rate of convergence \emph{independent} of the dimension exists.
\begin{theorem}[Integral approximation]\label{thm:integral_convergence}
\vspace{-2mm}
Assume $p(\theta)$ has the same support as the true posterior $q^*(\theta)$, \ie, $0\le q^*(\theta)/p(\theta)\le C$. Assume further model $\|p(x|\theta)^N\|_\infty\le \rho, \forall x$. Then $\forall f(\theta)$ bounded and integrable, the $T$-step PMD algorithm with stepsize $\gamma_t = \frac{\eta}{t}$ returns $m$ weighted particles such that 
\vspace{-1mm}
\begin{eqnarray*}
&&\EE\sbr{\abr{\int \qtil_T(\theta)f(\theta)d\theta - \int q^*(\theta)f(\theta)d\theta}}\le \frac{2\sqrt{\max\{C, \rho e^M)\}}\|f\|_\infty}{\sqrt{m}} + \max\bigg\{\sqrt{KL(q^*||p)}, \frac{\eta M}{\sqrt{2\eta - 1}}\bigg\}\frac{\|f\|_\infty}{\sqrt{T}}
\end{eqnarray*}
\vspace{-2mm}
where $M := \max_{t = 1,\ldots, T}\|g_t\|_\infty$.
\end{theorem}
\vspace{-2mm}
\paragraph{Remark.} The condition for the models, $\|p(x|\theta)^N\|_\infty\le \rho, \forall x$, is mild, and there are plenty of models satisfying such requirement. For examples, in binary/multi-class logistic regression, probit regression, as well as latent Dirichlet analysis, $\rho\le 1$. Please refer to details in Appendix~\ref{appendix:integral_convgence}. The proof combines the results of the weighted particles for integration, and convergence analysis of mirror descent. One can see that the error consists of two terms, one from integration approximation and the other from optimization error. To achieve the best rate of convergence, we need to balance the two terms. That is, when the number particles, $m$, scales linearly with the number of  iterations, we obtain an overall convergence rate of $O(\frac{1}{\sqrt{T}})$. In other words, if the number of particles is fixed to $m$, we could achieve the convergence rate $O(\frac{1}{\sqrt{m}})$ with $T = O(m)$ iterations.

\subsection{Strong Convergence of PMD}

In general, when the weighted kernel density approximation scheme is used, we show that PMD enjoys a much stronger convergence, \ie, the $KL$-divergence between the generated density and the true posterior converges sublinearly. Throughout this section, we merely assume that 
\begin{itemize}
\item The prior and likelihood belong to $(\beta;\mathcal{L})$-H\"{o}lder class. \\[-6mm]
\item Kernel $K(\cdot)$ is a $\beta$-valid density kernel with a compact support and there exists $\mu,\nu, \delta>0$ such that $\int K(z)^2\,dz\leq \mu^2$, $\int \|z\|^\beta|K(z)|dz\leq\nu$. \\[-6mm]
\item There exists a bounded support $\Omega$ such that  $\qtil_t^+$ almost surely bounded awary from $\Delta^{-1}>0$.\\[-6mm]
\end{itemize} 

Note that the above assumptions are more of a brief characteristics of the commonly used kernels and inferences problems in practice rather than an exception. The second condition clearly holds true when the logarithmic of the prior and likelihood belongs to $C_\infty$ with bounded derivatives of all orders, as assumed in several literature~\cite{TehThiVol2014a, TehThiVol2015}. The third condition is for characterizing the estimator over its support. These assumptions automatically validate all the conditions required to apply Theorem~\ref{thm:error1} and the corresponding high probability bounds (stated in Corollary \ref{cor:error4} in appendix). Let the kernel bandwidth $h_t = m_t^{-1/(d+2\beta)}$, we immediately have that with high probability,
$$\nbr{\qtil_{t+1}-\mathbf{P}_{\qtil_t}(\gamma_t g_t)}_1\leq O( m_t^{-\beta/(d+2\beta)}).$$ 
Directly applying Theorem~\ref{thm:main1}, and solving the recursion following~\cite{NemJudLanSha09}, we establish the convergence results in terms of KL-divergence. 
%
\begin{theorem}[KL-divergence]\label{thm:final} 
Based on the above assumptions, when setting $\gamma_t =\min\{\frac{2}{t+1}, \frac{\Delta}{Mm_t^{\beta/(d+2\beta)}}\}$,
\vspace{-1mm}
\begin{eqnarray*}
\EE[KL(q^* || \qtil_T)] &\leq& \frac{2\max\left\{D_1, M^2\right\}}{T} +  \mathcal{C}_1\frac{\sum_{t=1}^Tt^2 m_t^{-\frac{2\beta}{d+2\beta}}}{T^2}+ \Ccal_2m_T^{-\frac{\beta}{d+2\beta}}
\end{eqnarray*}

where $M := \max_{t = 1,\ldots, T}\|g_t\|_\infty$, $D_1=KL(q^*|| \qtil_1)$,  $\mathcal{C}_1:=O(1)(\mu+\nu\mathcal{L})^2\mu^2\Delta$, and $\Ccal_2 := O(1)M(\mu+\nu\mathcal{L})$ with $O(1)$ being a constant. 
\vspace{-3mm}
\end{theorem}

\paragraph{Remark.} 
Unlike Theorem~\ref{thm:integral_convergence}, the convergence results are established in terms of the $KL$-divergence, which is a stronger criterion and can be used to derive the convergence under other divergences~\cite{GibSu02}. To our best knowledge, these results are the first of its kind for estimating posterior densities in literature. One can immediately see that the final accuracy is essentially determined by two sources of errors, one from noise in applying stochastic gradient, the other from applying weighted kernel density estimator. For the last iterate, an overall $O(\frac{1}{T})$ convergence rate can be achieved when $m_t=O(t^{2+d/\beta})$. There is an explicit trade-off between the overall rate and the total number of particles: the more particles we use at each iteration, the faster algorithm converges. One should also note that in our analysis, we explicitly characterize the effect of the smoothness of model controlled by $\beta$, which is assumed to be infinite in existing analysis of SGLD. When the smoothness parameter $\beta>>d$, the number of particles is no longer depend on the dimension. That means, with memory budget $O(dm)$, \ie, the number of particles is set to be $O(m)$, we could achieve a $O(1/\sqrt{m})$ rate.

\paragraph{Open question.} It is worth mentioning that in the above result, the $O(1/T)$ bound corresponding to the stochasticity is tight (see \citet{NemJudLanSha09}), and the $O(m^{-\frac{\beta}{d+2\beta}})$ bound for KDE estimation is also tight by itself (see \cite{BarYan95}). An interesting question here is whether the overall complexity provided here is indeed optimal? This is out of the scope of this paper, and we will leave it as an open question.

\section{Related Work}\label{sec:related_work}

\begin{table*}[!t]
\centering
\caption{Summary of the related inference methods}\label{table:methods_survey}
\vspace{-3mm}
\begin{threeparttable}
    \begin{tabular}{ll|c|c|c|c|c|c}
      \hline
      \hline
      &Methods &Provable &Convergence  & Convergence  &\multicolumn{2}{|c|}{Cost} &Black \\
      \cline{6-7}
      &      &         &Criterion    &Rate          &Computation  &Memory &Box\\
      &      &         &                      &                  &per Iteration  & &\\
      \hline
      &SVI &No &$-$ &$-$ &$\Omega(d)$ &$O(d)$ &No \\
      &NPV &No &$-$ &$-$ &$\Omega(dm^2N + d^2N)$ &$O(dm)$ &No \\
      &{Static SMC} &No &$-$ &$-$ & $\Omega(dm)$ &$O(dm)$ &Yes\\
      &{SGLD} &Yes &$|\langle q-q^*, f \rangle |$ &$O(m^{-\frac{1}{3}})$ &$\Omega(d)$ &$O(dm)$ &Yes \\ 
      &PMD &Yes &$|\langle q-q^*, f \rangle |$ &$O(m^{-\frac{1}{2}})$ &$\Omega(dm)$ &$O(dm)$ &Yes\\
      &    &    &$KL(q^*||q)$ &$O(m^{-\frac{1}{2}})$ &$\Omega(dm^2)$ &$O(dm)$ &\\
      \hline
    \end{tabular}
\end{threeparttable}
\end{table*}
PMD connects stochastic optimization, Monte Carlo approximation and functional analysis to Bayesian inference. Therefore, it is closely related to two different paradigms of inference algorithms derived based on either optimization or Monte Carlo approximation.
\vspace{-3mm}
\paragraph{Relation to SVI.} From the optimization point of view, the proposed algorithm shares some similarities to stochastic variational inference~(SVI)~\cite{HofBleWanPai13}--both algorithms utilize stochastic gradients to update the solution. However, SVI optimizes a \emph{surrogate of the objective}, the evidence lower bound~(ELBO), with respect to a \emph{restricted parametric} distribution\footnote{Even in \citep{GerHofBle12}, ``nonparametric variational inference''~(NPV) uses the mixture of Gaussians as variational family which is still parametric.}; while the PMD directly optimizes the objective over \emph{all valid densities} in a nonparametric form. Our flexibility in density space eliminates the bias and leads to favorable convergence results.
\vspace{-3mm}
\paragraph{Relation to SMC.} From the sampling point of view, PMD and the particle filtering/sequential Monte Carlo~(SMC)~\cite{DouFreGor01} both rely on importance sampling. In the framework of SMC sampler~\cite{DelDouJas06}, the static SMC variants proposed in~\cite{Chopin02, BalMad06} bares some resemblances to the proposed PMD. However, their updates come from completely different origins: the static SMC update is based on Monte Carlo approximation of Bayes' rule, while the PMD update based on inexact prox-mappings. On the algorithmic side, ({\bf i}) the static SMC re-weights the particles with likelihood while the PMD re-weights based on functional gradient, which can be fractional power of the likelihood; and ({\bf ii}) the static SMC only utilizes each datum once while the PMD allows multiple pass of the datasets. Most importantly, on the theoretical side, PMD is guaranteed with convergence in terms of both $KL$-divergence and integral approximation for \emph{static model}, while SMC is only rigoriously justified for \emph{dynamic models}. It is unclear whether the convergence still holds for these extensions in~\cite{Chopin02, BalMad06}.
\vspace{-3mm}
\paragraph{Summary of the comparison.} We summarize the comparison between PMD and static SMC, SGLD, SVI and NPV in Table~\ref{table:methods_survey}. For the connections to other inference algorithms, including Annealed IS~\cite{Neal01}, general SMC sampler~\cite{DelDouJas06}, stochastic gradient dynamics family~\cite{WelTeh11, AhnKorWel12, DinFanBabCheetal14, ChenFoxGue14}, and nonparametric variational inference~\cite{SudIhlFreWil03,IhlMcA09,SonGreBicLowGue11,GerHofBle12,LieTehDou15}, please refer to Appendix~\ref{appendix:more_related_work}. Given dataset $\{x_i\}_{i=1}^N$, the model $p(x|\theta), \, \theta\in \RR^d$ and prior $p(\theta)$, whose value and gradient could be computed, we set PMD, static SMC, SGLD and NPV to keep $m$ samples/components, so that they have the same memory cost and comparable convergence rate in terms of $m$. Therefore, SGLD runs $O(m)$ iterations. Meanwhile, by balancing the optimization error and approximation in PMD, we have PMD running $O(m)$ for integal approximation and $O(\sqrt{m})$ for $KL$-divergence. For static SMC, the number of iteration is $O(N)$. From Table~\ref{table:methods_survey}, we can see that there exists a delicate trade-off between computation, memory cost and convergence rate for the approximate inference methods.

\begin{enumerate}
\item The static SMC uses simple normal distribution~\citep{Chopin02} or kernel density estimation~\citep{BalMad06} for rejuvenation. However, such moving kernel is purely heuristic and it is unclear whether the convergence rate of SMC for dynamic system~\citep{CriDou02, GlaOud04} still holds for static models. To ensure the convergence of static SMC, MCMC is needed in the rejuvenation step. The MCMC step requires to browse all the previously visited data, leading to extra computation cost $\Omega(dmt)$ and memory cost $O(dt)$, and hence violating the memory budget requirement. We emphasize that even using MCMC in static SMC for rejuvenation, the conditions required for static SMC is more restricted. We discuss the conditions for convergence of SMC and PMD using particles approximation in Appendix~\ref{appendix:integral_convgence}. 

\item Comparing with SGLD, the cost of PMD at each iteration is higher. However, PMD converges in rate of $O(m^{-\frac{1}{2}})$, faster than SGLD, $O(m^{-\frac{1}{3}})$, in terms of integral approximation and $KL$-divergence which is more stringent if \emph{all the orders} of derivatives of stochastic gradient is bounded. Moreover, even for the integral approximation, SGLD converges only when $f$ having weak Taylor series expansion, while for PMD, $f$ is only required to be bounded. The SGLD also requires the stochastic gradient satisfying several extra conditions to form a Lyapunov system, while such conditions are not needed in PMD.

\end{enumerate}

\section{Experiments}\label{sec:experiment}

We conduct experiments on mixture models, logistic regression, sparse Gaussian processes and latent Dirichlet allocation to demonstrate the advantages of PMD in capturing multiple modes, dealing with non-conjugate models and incorporating special structures, respectively. 
\vspace{-3mm}
\paragraph{Competing algorithms.} For the mixture model and logistic regression, we compare our algorithm with five general approximate Bayesian inference methods, including three sampling algorithms, \ie, one-pass sequential Monte Carlo (one-pass SMC)~\cite{BalMad06} which is an improved version of the SMC for Bayesian inference~\cite{Chopin02}, stochastic gradient Langevin dynamics~(SGD Langevin)~\cite{WelTeh11} and Gibbs sampling, and two variational inference methods, \ie, stochastic variational inference~(SVI)~\cite{HofBleWanPai13} and stochastic variant of nonparametric variational inference~(SGD NPV)~\cite{GerHofBle12}. For sparse GP and LDA, we compare with the existing large-scale inference algorithms designed specifically for the models. 
\vspace{-3mm}
\paragraph{Evaluation criterion.} For the synthetic data generated by mixture model, we could calculate the true posterior, Therefore, we evaluate the performance directly through total variation and $KL$-divergence (cross entropy). For the experiments on logistic regression, sparse GP and LDA on real-world datasets, we use indirect criteria which are widely used~\citep{ChenFoxGue14, DinFanBabCheetal14, HenFusLaw13,PatTeh2013, HofBleWanPai13} because of the intractability of the posterior. We keep the same memory budget for Monte Carlo based algorithms if their computational cost is acceptable. To demonstrate the efficiency of each algorithm in utilizing data, we use the number of data visited cumulatively as x-axis. 

For the details of the model specification, experimental setups, additional results and algorithm derivations for sparse GP and LDA, please refer to the Appendix~\ref{appendix:exp}.

\vspace{-3mm}
\paragraph{Mixture Models.} We conduct comparison on a simple yet interesting mixture model~\cite{WelTeh11}, the observations $x_i \sim p\Ncal(\theta_1, \sigma_x^2) + (1-p)\Ncal(\theta_1+\theta_2, \sigma_x^2)$ and $\theta_1\sim \Ncal(0, \sigma_1^2)$, $\theta_2\sim \Ncal(0, \sigma_2^2)$, where $(\sigma_1, \sigma_2) = (1, 1)$, $\sigma_x = 2.5$ and $p = 0.5$. The means of two Gaussians are tied together making $\theta_1$ and $\theta_2$ correlated in the posterior. We generate $1000$ data  from the model with $(\theta_1, \theta_2)= (1, -2)$. This is one mode of the posterior, there is another equivalent mode at $(\theta_1, \theta_2)= (-1, 2)$. 
\begin{figure*}[!t]
\centering
\setlength{\tabcolsep}{0pt}
  \small
  \begin{tabular}{cccc}
  \hspace{-2mm}
  \includegraphics[width=0.24\textwidth]{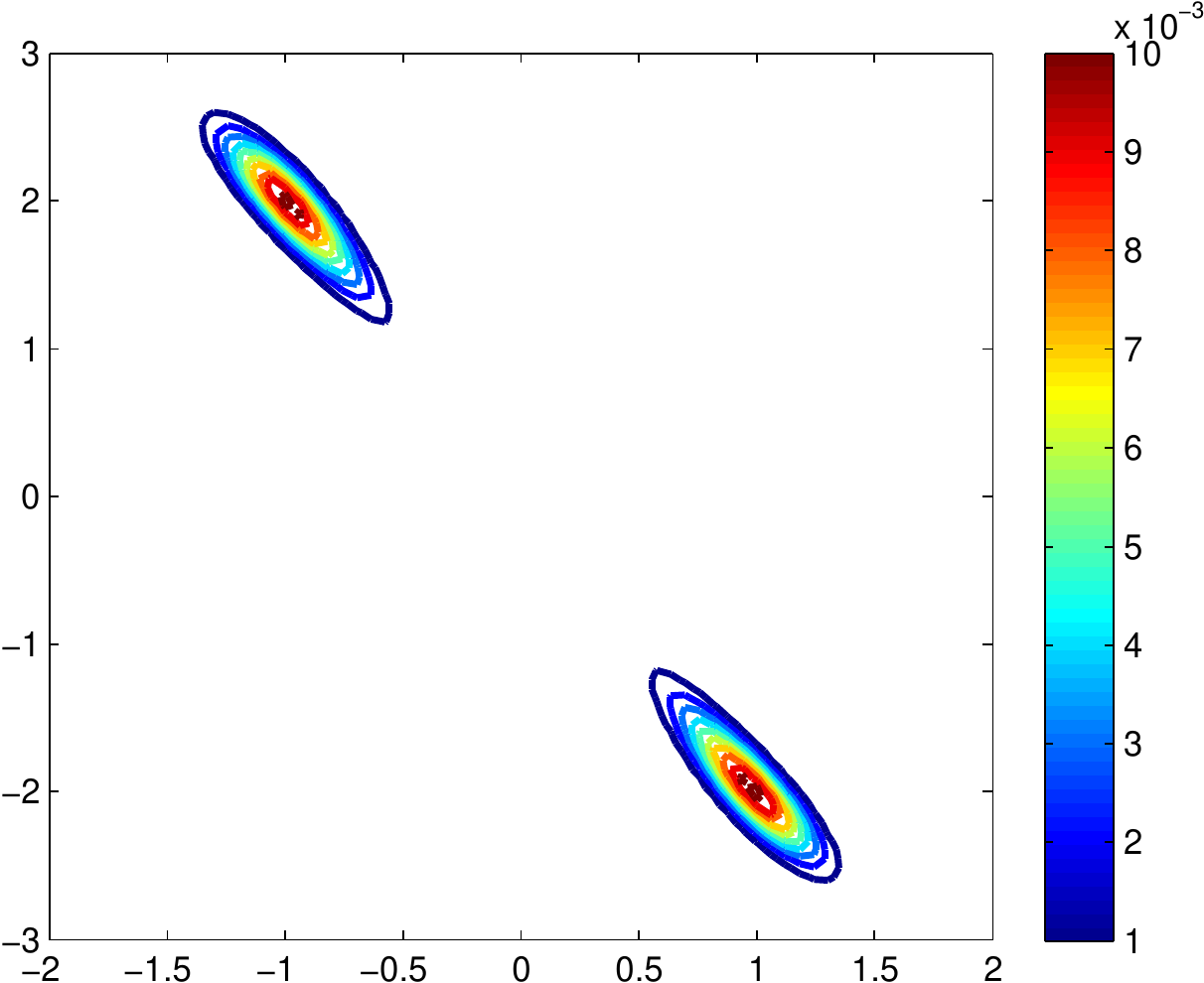} &
  \includegraphics[width=0.24\textwidth]{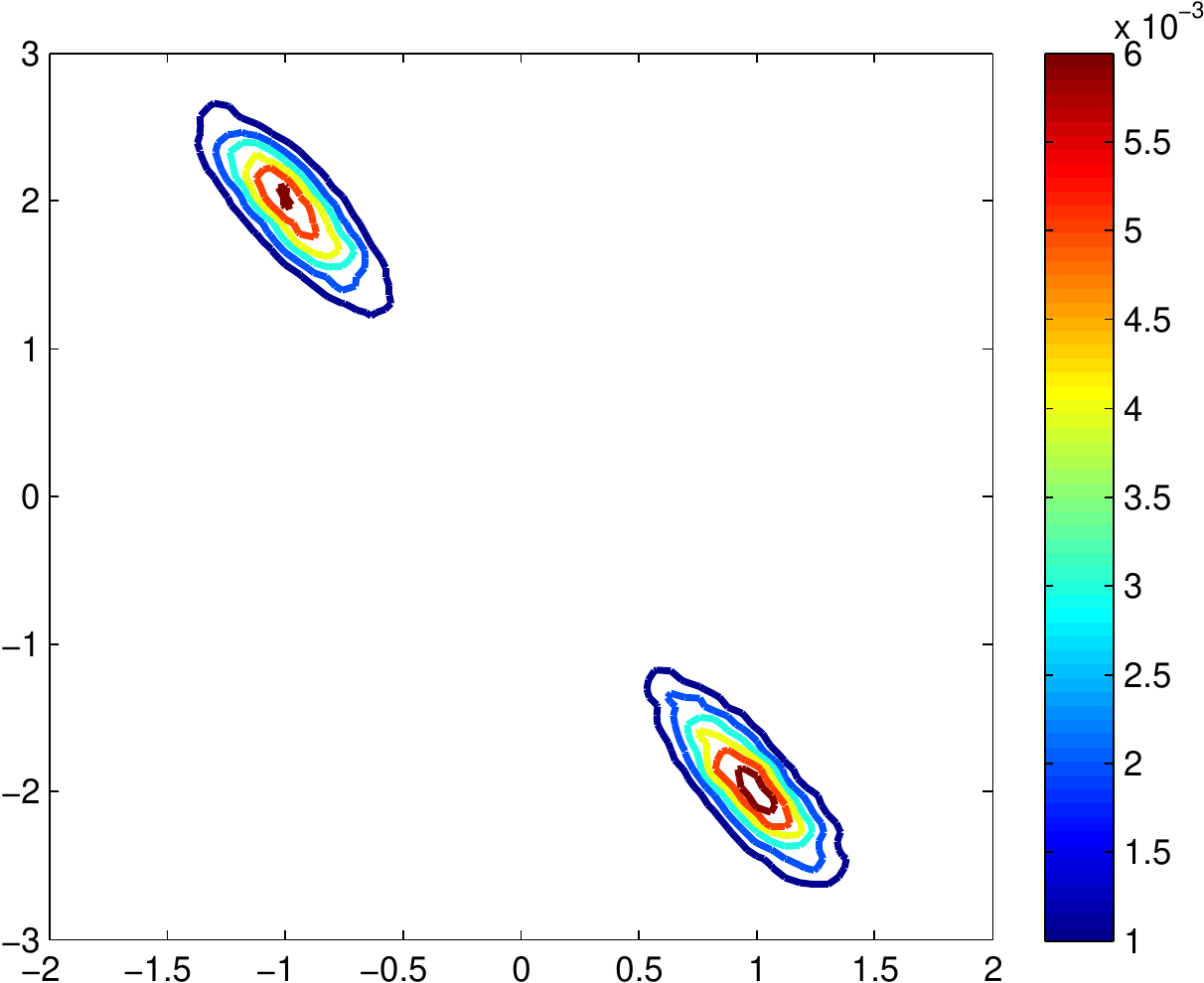} &
  \includegraphics[width=0.25\textwidth]{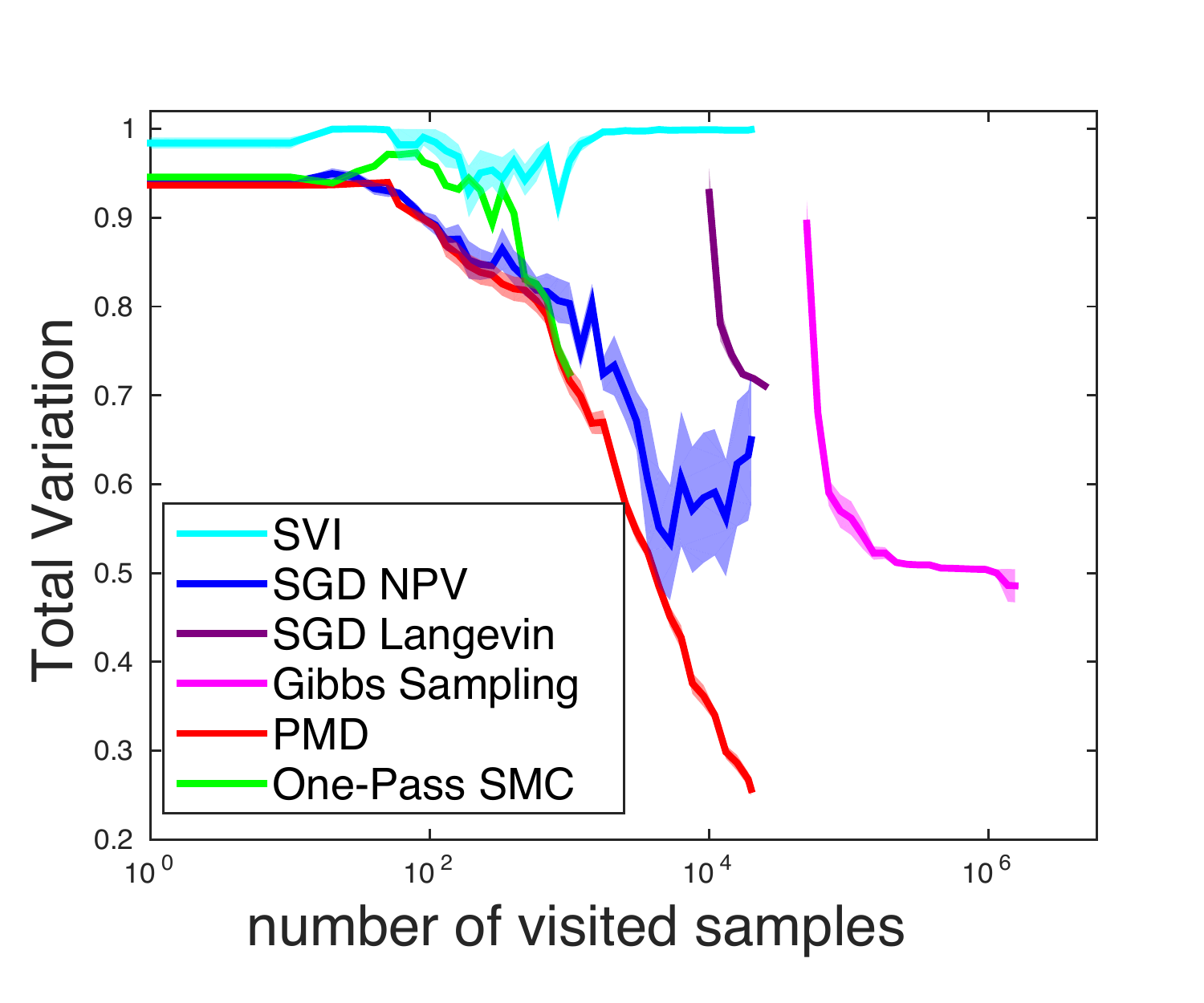} & 
  \includegraphics[width=0.25\textwidth]{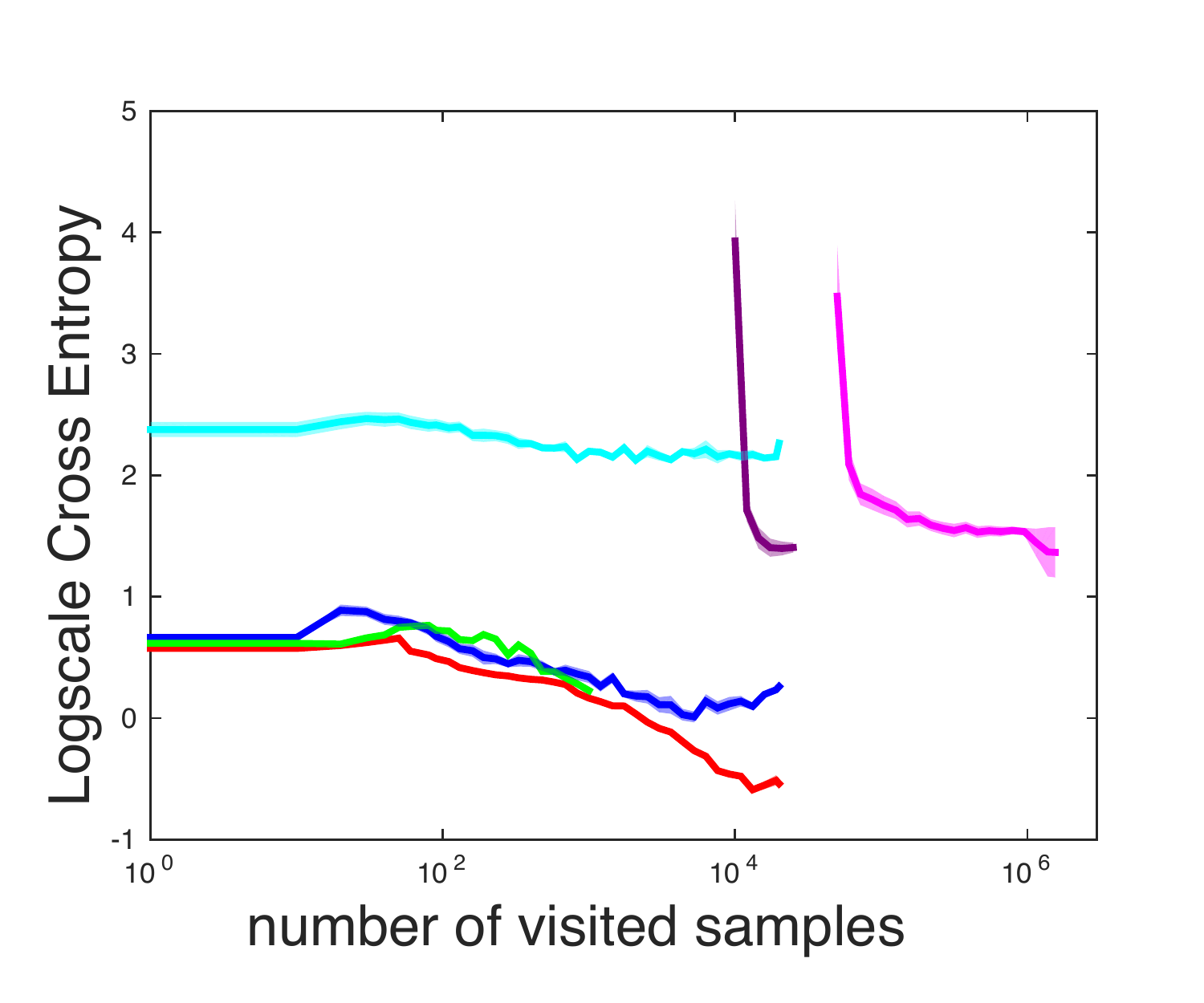}\\
  (1) True Posterior & (2) Estimated Posterior & (3) Total Variation & (4) Cross Entropy
  \end{tabular}
  \vspace{-3mm}
  \caption{Experimental results for mixture model on synthetic dataset.}
  \label{fig:syn_data}
\end{figure*}
\begin{figure*}[ht]
\centering
\setlength{\tabcolsep}{0pt}
\small
  \begin{tabular}{ccc}
  \hspace{-3mm}
  \includegraphics[width=0.32\textwidth]{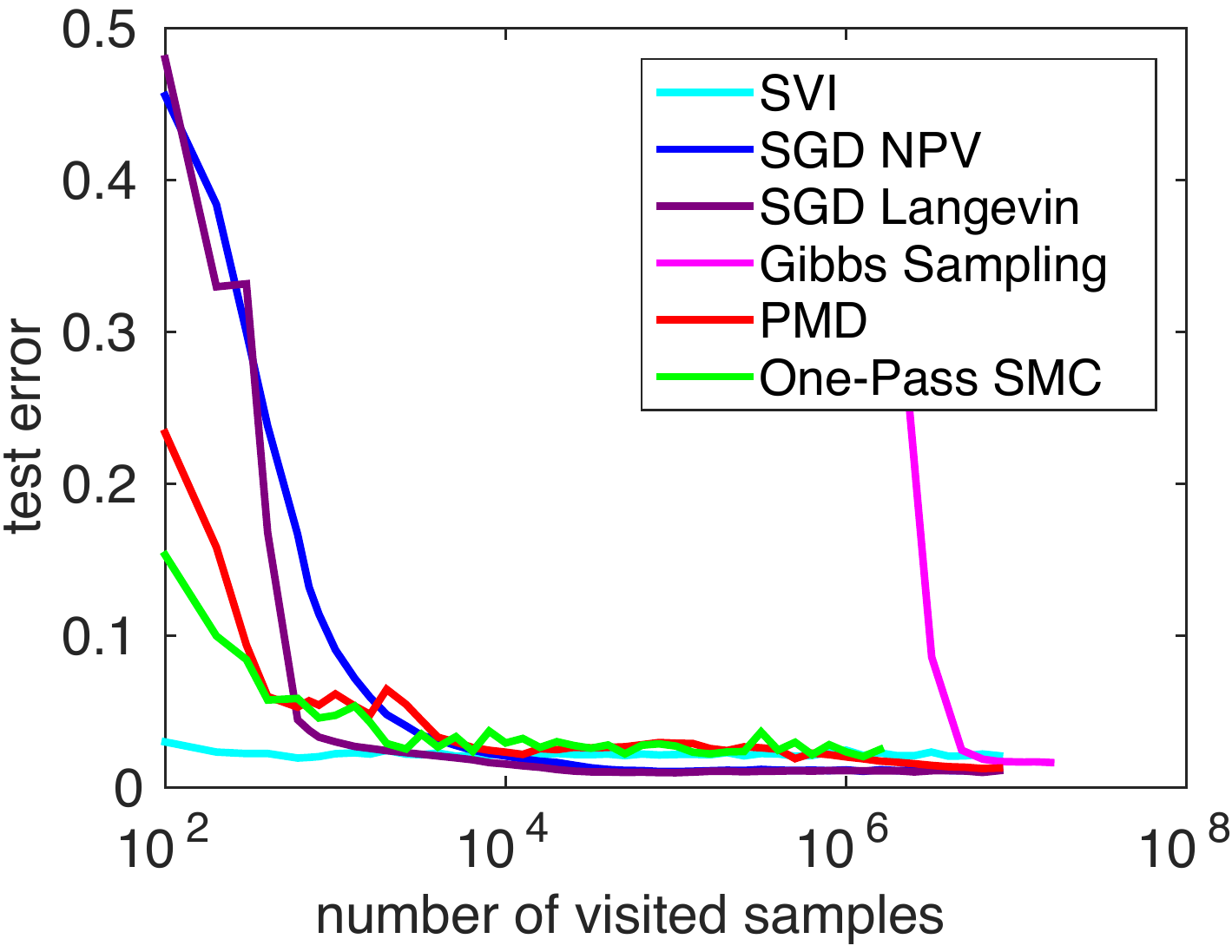} &
  \includegraphics[width=0.32\textwidth]{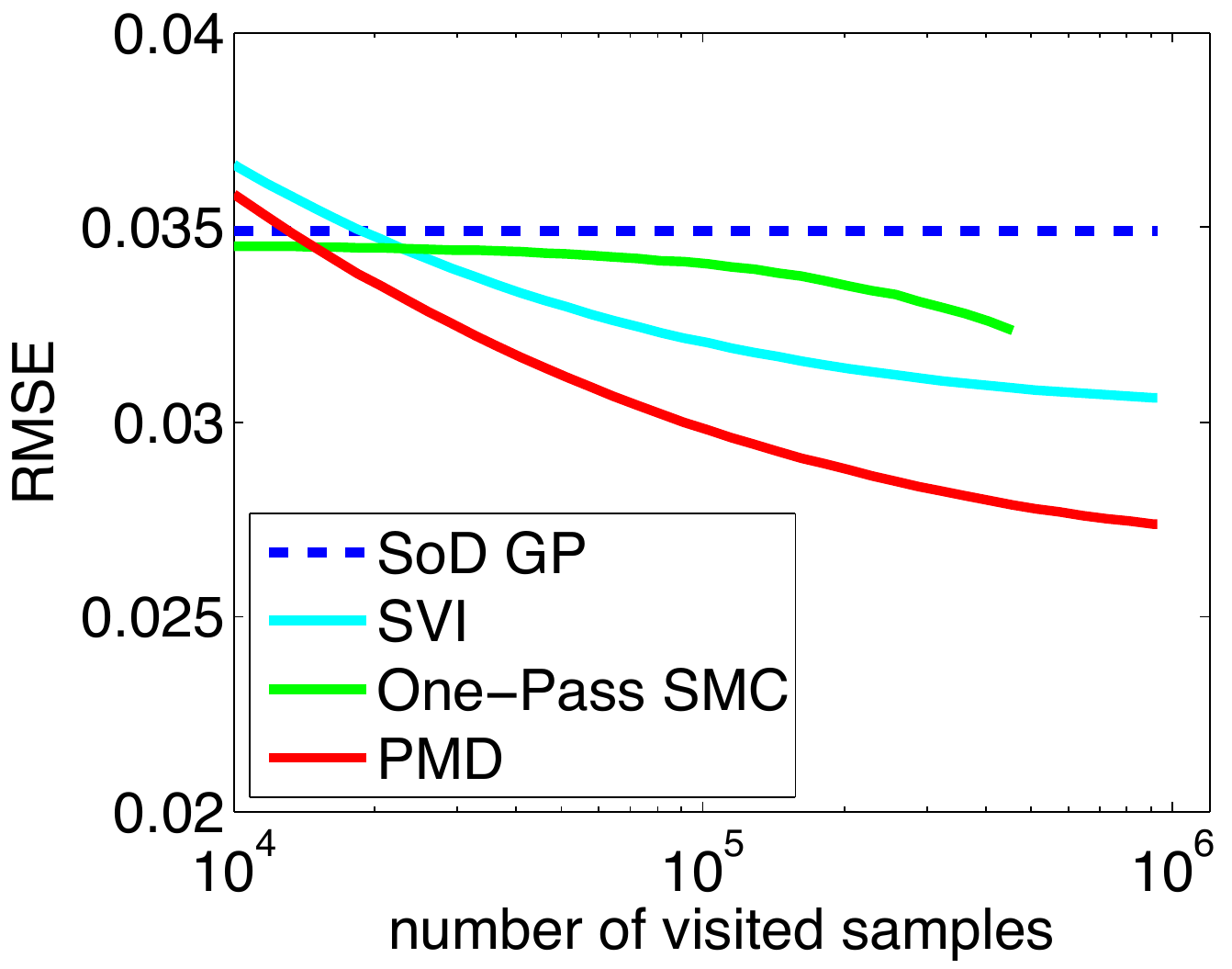} &
  \includegraphics[width=0.32\textwidth]{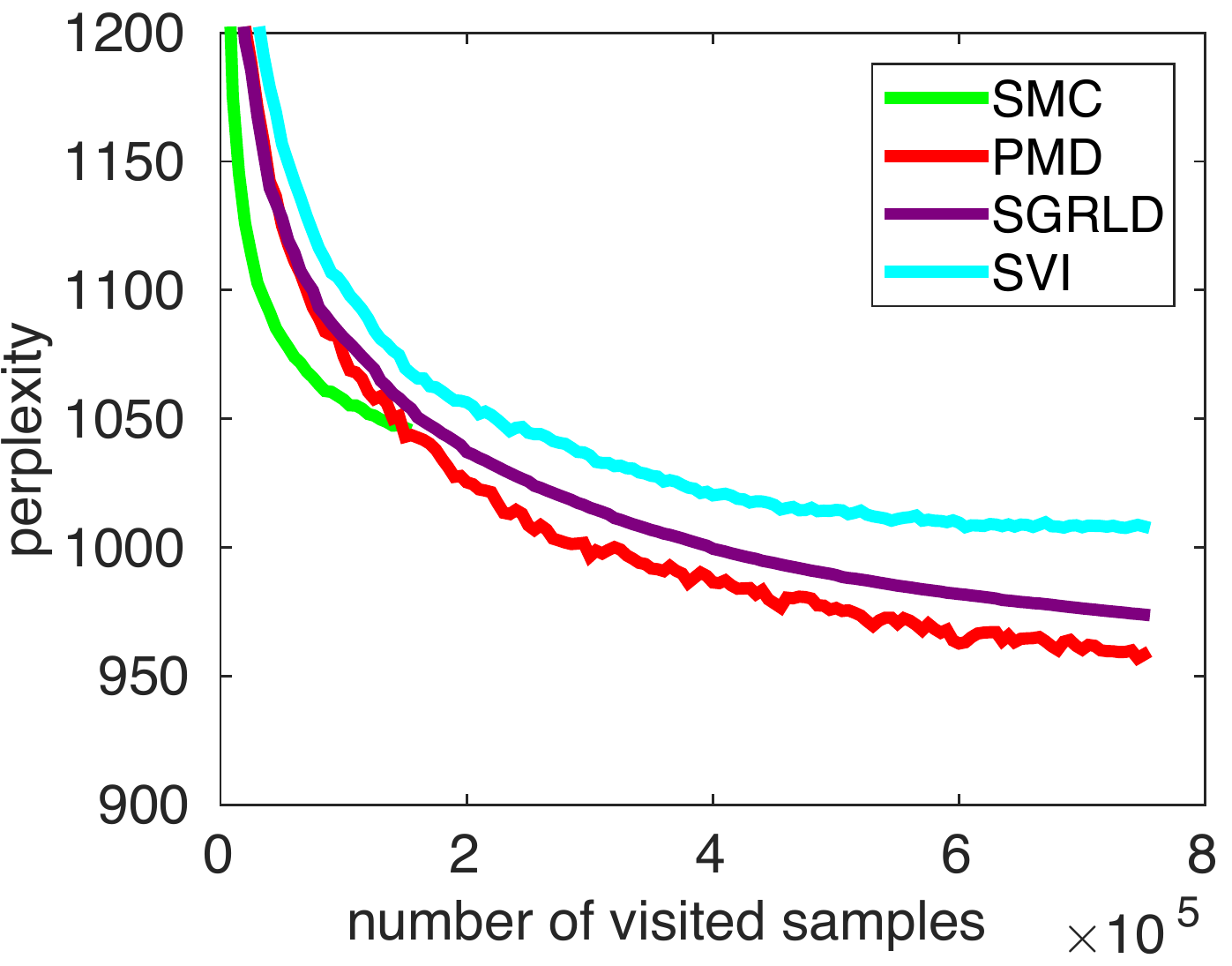} \\
  (1) Logistic regression on MNIST & (2) Sparse GP on music data & (3) LDA on wikipedia data
  \end{tabular}
  \vspace{-3mm}
  \caption{Experimental results on several different models for real-world datasets.}
  \vspace{-3mm}
  \label{fig:real_data}
\end{figure*}
We initialize all algorithms with prior on $(\theta_1, \theta_2)$. We repeat the experiments $10$ times and report the average results. We keep the same memory for all except SVI. The true posterior and the one generated by our method is illustrated in Figure~\ref{fig:syn_data} (1)(2). PMD fits both modes well and recovers nicely the posterior while other algorithms either miss a mode or fail to fit the multimodal density. For the competitors' results, please refer to Appendix~\ref{appendix:exp}. PMD achieves the best performance in terms of total variation and cross entropy as shown in Figure~\ref{fig:syn_data} (3)(4). This experiment clearly indicates our algorithm is able to take advantages of nonparametric model to capture multiple modes.

\vspace{-3mm}
\paragraph{Bayesian Logistic Regression.} We test our algorithm on logistic regression with non-conjugate prior for handwritten digits classification on the {MNIST8M 8 vs. 6} dataset. The dataset contains about $1.6$M training samples and $1932$ testing samples. We initialize all algorithms with same prior and terminate the stochastic algorithms after 5 passes through the dataset.  We keep $1000$ samples for Monte Carlo based algorithms, except Gibbs sampling whose computation cost is unaffordable. We repeat the experiments $10$ times and the results are reported in Figure~\ref{fig:real_data}(1). Obviously, Gibbs sampling~\cite{HolHel06}, which needs to scan the whole dataset, is not suitable for large-scale problem. In this experiment, SVI performs best at first, which is expectable because learning in the Gaussian family is simpler comparing to nonparametric density family. Our algorithm achieves comparable performance in nonparametric form after fed with enough data, $98.8\%$, to SVI which relies on carefully designed lower bound of the log-likelihood~\cite{JaaJor98b}. SGD NPV is flexible with mixture models family, however, its speed becomes the bottleneck. For SGD NPV, the speed is dragged down for the use of L-BFGS to optimize the second-order approximation of ELBO.

\vspace{-3mm}
\paragraph{Sparse Gaussian Processes.} We use sparse GPs models  to predict the year of songs~\cite{BerEllWhiLam11}. In this task, we compare to the SVI for sparse GPs~\cite{HofBleWanPai13,HenFusLaw13}and one-pass SMC. We also included subset of data approximation~(SoD)~\cite{QuiRas05} as baseline. The data contains about $0.5$M songs, each represented by $90$-dimension features. We terminate the stochastic algorithms after 2 passes of dataset. We use $16$ particles in both SMC and PMD. The number of inducing inputs in sparse GP is set to be $2^{10}$, and other hyperparameters of sparse GP are fixed for all methods. We run experiments $10$ times and results are reported in Figure.~\ref{fig:real_data}(2). Our algorithm achieves the best RMSE $0.027$, significantly better than one-pass SMC and SVI. 

\paragraph{Latent Dirichlet Allocation.} We compare to SVI~\cite{HofBleWanPai13}, stochastic gradient Riemannian Langevin dynamic~(SGRLD)~\cite{PatTeh2013}, and SMC specially designed for LDA~\cite{CanShiGri09} on Wikipedia dataset~\cite{PatTeh2013}. The dataset contains $0.15$M documents, about $2$M words and $8000$ vocabulary. Since we evaluate their performances in terms of perplexity, which is integral over posterior, we do not need to recover the posterior, and therefore, we follow the same setting in~\cite{AhmAlyGonNaretal12,MimHofBle12}, where one particle is used in SMC and PMD to save the cost. We set topic number to $100$ and fix other hyperparameters to be fair to all algorithms.  We stop the stochastic algorithms after 5 passes of dataset. The results are reported in Figure~\ref{fig:real_data}(3). The top words from several topics found by our algorithm are illustrated in Appendix~\ref{appendix:exp}. Our algorithm achieves the best perplexity, significantly better than SGRLD and SVI. In this experiment, SMC performs well at the beginning  since it treats each documents equally and updates with full likelihood. However, SMC only uses each datum once, while the stochastic algorithms, \eg, SGRLD, SVI and our PMD, could further refine the solution by running the dataset multiple times.  

\section{Conclusion}\label{sec:discussion}
Our work contributes towards achieving better trade-off between \emph{efficiency}, \emph{flexibility} and \emph{provability} in approximate Bayesian inference from \emph{optimization perspective}. The proposed algorithm, \emph{Particle Mirror Descent}, successfully combines stochastic mirror descent and nonparametric density approximation.  Theoretically,  the algorithm enjoys a rate $O(1/\sqrt{m})$ in terms of both integral approximation and $KL$-divergence, with $O(m)$ particles. Practically, the algorithm achieves competitive performance to existing state-of-the-art inference algorithms in mixture models, logistic regression, sparse Gaussian processes and latent Dirichlet analysis on several large-scale datasets.

\section*{Acknowledgements}

We thank the anonymous referees for their valuable suggestions. This project was supported in part by NSF/NIH BIGDATA 1R01GM108341, ONR N00014-15-1-2340, NSF IIS-1218749, and NSF CAREER IIS-1350983.


\bibliographystyle{plainnat}
{

}


\clearpage
\newpage

\appendix
\addcontentsline{toc}{part}{\appendixname}

\textbf{\huge Appendix}
\section{Strong convexity}

As we discussed, the posterior from Bayes's rule could be viewed as the optimal of an optimization problem in Eq~\eq{eq:bayes_opt}. We will show that the objective function is strongly convex w.r.t $KL$-divergence.\\

\emph{Proof for Lemma \ref{lem:strongconvexity}. }
The lemma directly results from the generalized Pythagaras theorem for Bregman divergence.
Particularly, for $KL$-divergence, we have
\begin{eqnarray*}
KL(q_1||q) &=& KL(q_1||q_2) + KL(q_2||q) - \langle q_1 - q_2, \nabla \phi(q) -  \nabla \phi(q_2)\rangle_2
\end{eqnarray*}
where $\phi(q)$ is the entropy of $q$. 

Notice that $L(q) = KL(q||q^*)-\log Z$, where $q^* = \frac{p(\theta)\Pi_i^Np(x_i |\theta)}{Z}, \quad Z =\int p(\theta)\Pi_i^Np(x_i |\theta) $, we have
\begin{eqnarray*}
&&KL(q_1||q^*) - KL(q_2||q^*) - \langle q_1 - q_2, \nabla \phi(q_2) - \nabla \phi(q^*) \rangle_2 = KL(q_1||q_2)\\
&\Rightarrow& KL(q_1||q^*) - KL(q_2||q^*) - \langle q_1 - q_2, \log q_2 - \log q^* \rangle_2 = KL(q_1||q_2)\\
&\Rightarrow& KL(q_1||q^*) - KL(q_2||q^*) - \langle q_1 - q_2, \log q_2 - \log \big(p(\theta)\Pi_i^Np(x_i |\theta)\big)\rangle_2+\underbrace{\langle q_1 - q_2, \log Z \rangle_2}_{0} = KL(q_1||q_2)\\
&\Rightarrow& L(q_1) - L(q_2) - \langle q_1 - q_2, \nabla L(q_2)\rangle_2= KL(q_1||q_2)
\end{eqnarray*}
\hfill$\blacksquare$

\section{Finite Convergence of Stochastic Mirror Descent with Inexact Prox-Mapping in Density Space}\label{appendix:inexact_prox}

Since the prox-mapping of stochastic mirror descent is intractable when directly being applied to the optimization problem~\eq{eq:bayes_opt}, we propose the \emph{$\epsilon$-inexact prox-mapping} within the stochastic mirror descent framework in Section~\ref{sec:algorithm}. Instead of solving the prox-mapping exactly, we approximate the solution with $\epsilon$ error. In this section, we will show as long as the approximation error is tolerate, the stochastic mirror descent algorithm still converges.

{\bf Theorem~\ref{thm:main1} }
\emph{Denote $q^*=\argmin_{q\in\mathcal{P}} L(q)$, the stochastic mirror descent with inexact prox-mapping after $T$ steps gives
\begin{enumerate}[noitemsep,nolistsep] 
\item[(a)] the recurrence: $\forall t \leqslant T$, $\EE[KL(q^* ||  \qtil_{t+1})] \leqslant \epsilon_t+(1-\gamma_t) \EE[KL(q^* ||  \qtil_t)] + \frac{\gamma_t^2\EE\|g_t\|^2_\infty}{2}$
\item[(b)] the sub-optimality: 
  $
  \EE[KL( \bar q_{T}||q^* )] \leqslant  \EE[L(\bar q_{T})-L(q^*)] \leqslant \frac{\mathcal{M}^2\cdot\frac{1}{2}\sum_{t=1}^T\gamma_t^2+\sum_{t=1}^T\epsilon_t+D_1}{\sum_{t=1}^T\gamma_t}
  $
  where $\bar q_T=\sum_{t=1}^T\gamma_t \qtil_t/\sum_{t=1}^T\gamma_t$ and $D_1=KL(q^*|| \qtil_1)$ and $\mathcal{M}^2:=\max_{1\leq t\leq T}\EE\|g_t\|_\infty^2$.
\end{enumerate}
}

\noindent{\bf Remark.} Based on~\cite{NemJudLanSha09}, one can immediately see that, to guarantee the usual rate of convergence, the error $\epsilon_t$ can be of order $O(\gamma_t^2)$. The first recurrence implies an overall $O(1/T)$ rate of convergence for the $KL$-divergence when the stepsize $\gamma_t$ is as small as $O(1/t)$ and error $\epsilon_t$ is as small as $O(1/t^2)$. The second result implies an overall $O(1/\sqrt{T})$ rate of convergence for objective function when larger stepsize $\gamma_t=O(1/\sqrt{T})$ and larger error $\epsilon_t=O(1/t)$ are adopted.

\emph{Proof for Theorem \ref{thm:main1}. }
(a) By first-order optimality condition, $\qtil_{t+1}\in P_{\qtil_t}^{\epsilon_t}(\gamma_t g_t)$ is equivalent as
$$\langle \gamma_t g_t+\log(\qtil_{t+1})-\log(\qtil_t),\qtil_{t+1}-q\rangle_{L_2}\leq\epsilon_t,\forall q\in\mathcal{P}, $$
which implies that
\begin{eqnarray*}
\langle \gamma_t g_t, \qtil_{t+1}-q\rangle_{2}\leq \langle \log(\qtil_t)-\log(\qtil_{t+1}), \qtil_{t+1}-q\rangle_{2}+\epsilon_t
= KL(q||\qtil_t)-KL(q||\qtil_{t+1})-KL(\qtil_{t+1}||\qtil_t)+\epsilon_t
\end{eqnarray*}
Hence, 
\begin{eqnarray}\label{eq:fact1}
\langle \gamma_t g_t, \qtil_{t}-q\rangle_{2}\leq KL(q||\qtil_t)-KL(q||\qtil_{t+1})-KL(\qtil_{t+1}||\qtil_t)+\langle\gamma_t g_t, \qtil_{t}-\qtil_{t+1}\rangle_{2}+\epsilon_t.
\end{eqnarray}
By Young's inequality, we have
\begin{eqnarray}\label{eq:fact2}
\langle\gamma_t g_t, \qtil_{t}-\qtil_{t+1}\rangle_{2}\leq \frac{1}{2}\|\qtil_t-\qtil_{t+1}\|_{1}^2+\frac{\gamma_t^2}{2}\|g_t\|_{\infty}^2.
\end{eqnarray}
Also, from Pinsker's inequality, we have 
\begin{eqnarray}\label{eq:fact3}
KL(\qtil_{t+1}||\qtil_t)\geq \frac{1}{2}\|\qtil_t-\qtil_{t+1}\|_{1}^2.
\end{eqnarray}
Therefore, combining (\ref{eq:fact1}), (\ref{eq:fact2}), and (\ref{eq:fact3}), we have $\forall q\in \mathcal{P}$
\begin{eqnarray*}\label{eq:fact4}
\langle \gamma_t g_t, \qtil_{t}-q\rangle_{2}\leq \epsilon_t+{KL(q||\qtil_t)}-{KL(q||\qtil_{t+1})} +\frac{\gamma_t^2}{2}\|g_t\|_{\infty}^2
\end{eqnarray*}
Plugging $q^*$ and taking expectation on both sides, the LHS becomes
\begin{eqnarray*}
\EE_x\bigg[\langle \qtil_t - q^*, \gamma_t g_t\rangle\bigg] &=& \EE_x\bigg[\langle \qtil_{t} - q^*, \gamma_t\EE[ g_t]\rangle\bigg|x_{[t-1]}\bigg] = \EE_x\bigg[\langle \qtil_{t} - q^*, \gamma_t\nabla L(\qtil_t)\rangle\bigg],
\end{eqnarray*}
Therefore, we have
\begin{eqnarray}\label{eq:main}
\EE_x\bigg[\langle \qtil_{t} - q^*, \gamma_t\nabla L(\qtil_t)\rangle\bigg]\leq\epsilon_t+\EE_x\big[{KL(q^*||\qtil_t)}\big]
-\EE_x\big[{KL(q^*||\qtil_{t+1})}\big] +\frac{\gamma_t^2}{2}\EE_x\|g_t\|_{\infty}^2
\end{eqnarray}
Because the objective function is $1$-strongly convex w.r.t. $KL$-divergence, 
\begin{eqnarray*}
\langle q' - q, \nabla L(q') - \nabla L(q)\rangle = KL(q'||q) + KL(q || q'),
\end{eqnarray*}
and the optimality condition, we have
\begin{eqnarray*}
\langle \qtil_t - q^*, \nabla L(\qtil_t) \rangle \ge KL(q^* || \qtil_t)
\end{eqnarray*}
we obtain the recursion with inexact prox-mapping,
\begin{eqnarray*}
\EE_x[KL(q^* || \qtil_{t+1})] \le \epsilon_t+(1-\gamma_t) \EE_x[KL(q^* || \qtil_t)] + \frac{\gamma_t^2}{2}\mathcal{M}^2
\end{eqnarray*}

(b) Summing over $t=1,\ldots, T$ of equation (\ref{eq:main}), we get
\begin{eqnarray*}
\sum_{t=1}^T\EE_x[\langle \qtil_{t} - q^*, \gamma_t\nabla L(\qtil_t)\rangle]\leq\sum_{t=1}^T\epsilon_t+{KL(q^*||\qtil_1)}+\sum_{t=1}^T\frac{\gamma_t^2}{2}\mathcal{M}^2
\end{eqnarray*}
By convexity and optimality condition, this leads to
\begin{eqnarray*}
\left(\sum_{t=1}^T\gamma_t\right)\EE_x[L(\bar q_T)-L(q^*)]\leq\EE_x\left[\sum_{t=1}^T\gamma_t( L(\qtil_t)-L(q^*))\right]\leq\sum_{t=1}^T\epsilon_t+{KL(q^*||\qtil_1)}+\sum_{t=1}^T\frac{\gamma_t^2}{2}\mathcal{M}^2
\end{eqnarray*}
Furthermore, combined with the $1$-strongly-convexity, it immediately follows that 
\begin{eqnarray*}
\EE_x[KL(\bar q_{T} ||q^* )] &\leq& \EE_x[L(\bar q_T)-L(q^*)]\leq \frac{\frac{1}{2}\sum_{t=1}^T\gamma_t^2\mathcal{M}^2+\sum_{t=1}^T\epsilon_t+D_1}{\sum_{t=1}^T\gamma_t}.
\end{eqnarray*}
\hfill$\blacksquare$

\section{Convergence Analysis for Integral Approximation}\label{appendix:integral_convgence}

In this section, we provide the details of the convergence analysis of the proposed algorithm in terms of integral approximation w.r.t. the true posterior using a good initialization.

Assume that the prior $p(\theta)$ has support $\Omega$ cover true posterior distribution $q^*(\theta)$, then, we could represent 
$$
q^*(\theta) \in \Fcal = \bigg\{q(\theta) = \alpha(\theta)p(\theta), \int \alpha(\theta)p(\theta)d\theta = 1,  0\le \alpha(\theta)\le C \bigg\}.
$$ 
Therefore, one can show 

\begin{lemma}\label{lemma:finite_approximation}
$\forall q\in \Fcal$, let $\{\theta_i\}_{i=1}^m$ is \iid\, sampled from $p(\theta)$, we could construct $\hat q(\theta) = \sum_{i=1}^m \frac{\alpha(\theta_i)\delta(\theta_i)}{\sum_i^m\alpha(\theta_i)} $, such that $\forall f(\theta): \RR^d\rightarrow \RR$ bounded and integrable,
$$
\EE{\bigg[\bigg|\int \hat q(\theta)f(\theta)d\theta - \int q(\theta)f(\theta)d\theta\bigg| \bigg]}\le \frac{2\sqrt{C}\|f\|_\infty}{\sqrt{m}}.
$$
\end{lemma}

\begin{proof}

Given $q(\theta)$, we sample \iid $\{\theta_i\}_{i=1}^m$ from $p(\theta)$, and construct a function 
\begin{eqnarray*}
\hat q(\theta) = \frac{1}{m}\sum_{i=1}^m \alpha(\theta_i)\delta(\theta_i, \theta).
\end{eqnarray*}
It is obviously that
\begin{eqnarray*}
\EE_\theta[\hat q(\theta)] = \EE_\theta\bigg[\frac{1}{m}\sum_{i=1}^m \alpha(\theta_i)\delta(\theta_i, \theta)\bigg] = \frac{1}{m}\sum_{i=1}^{m}\EE_\theta\bigg[\alpha(\theta_i)\delta(\theta_i, \theta) \bigg] = q(\theta)
\end{eqnarray*}
and 
\begin{eqnarray*}
\EE_\theta\bigg[\int \hat q(\theta) f(\theta)d\theta\bigg] = \EE_\theta\bigg[\frac{1}{m}\sum_{i=1}^m \alpha(\theta_i)f(\theta_i)\bigg] = \frac{1}{m}\sum_{i=1}^{m}\EE_\theta\bigg[\alpha(\theta_i)f(\theta_i) \bigg] = \int q(\theta)f(\theta) d\theta
\end{eqnarray*}

Then, 
\begin{eqnarray*}
\EE_\theta\bigg[\bigg|\int \hat q(\theta) f(\theta)d\theta  - \int q(\theta) f(\theta)d\theta\bigg|^2\bigg] = \EE_\theta\bigg[\bigg|\int \hat q(\theta)f(\theta)d\theta - \EE_\theta\bigg[\int \hat q(\theta) f(\theta)d\theta\bigg]\bigg|^2\bigg] \\
= \frac{1}{m}\bigg(\EE_\theta\|\alpha(\theta_i) f(\theta_i)\|_2^2 - \|\EE_\theta[\alpha(\theta_i) f(\theta_i)]\|_2^2\bigg)\le \frac{1}{m}\EE_\theta\|\alpha(\theta_i) f(\theta_i)\|_2^2 = \frac{1}{m}\int \alpha(\theta)^2f(\theta)^2\pi(\theta)d\theta \\
 =\frac{1}{m}\int \alpha(\theta)f(\theta)^2 q(\theta)d\theta \le \frac{{C}}{m}\|f(\theta)\|_\infty^2\int\alpha(\theta)q(\theta)d\theta\le \frac{{C}}{m}\|f(\theta)\|_\infty^2\|\alpha(\theta)\|_\infty
\end{eqnarray*}
By Jensen's inequality, we have
\begin{eqnarray*}
\EE_\theta\bigg[\bigg|\int \hat q(\theta) f(\theta)d\theta  - \int q(\theta) f(\theta)d\theta\bigg|\bigg] \le \sqrt{\EE_\theta\bigg[\bigg|\int \hat q(\theta) f(\theta)d\theta  - \int q(\theta) f(\theta)d\theta\bigg|^2 \bigg]}\le  \frac{\sqrt{C}\|f(\theta)\|_\infty}{\sqrt{m}} 
\end{eqnarray*}

Apply the above conclusion to $f(\theta) = 1$, we have
\begin{eqnarray*}
\EE\bigg[\bigg|\frac{1}{m}\sum_i^m \alpha_i - 1\bigg|\bigg] \le \frac{\sqrt{C}}{\sqrt{m}}
\end{eqnarray*}

Let $\tilde q(\theta) = \frac{\sum_i^m \alpha(\theta_i)\delta(\theta_i, \cdot)}{\sum_i^m \alpha(\theta_i)} $, then $\sum_i^m  \frac{\alpha_i}{\sum_i^m \alpha_i}= 1$, and 
\begin{eqnarray*}
&&\EE_\theta\bigg[\bigg|\int \tilde q(\theta)f(\theta)d\theta - \int \hat q(\theta)f(\theta)d\theta\bigg|\bigg] =\EE_\theta\bigg[\bigg|\frac{1}{\sum_i^m \alpha(\theta_i)}\sum_i^m \alpha(\theta_i)f(\theta_i) - \frac{1}{m}\sum_i^m \alpha(\theta_i)f(\theta_i)\bigg| \bigg]\\
&=& \EE_\theta\bigg[\bigg|1 - \frac{\sum_i^m \alpha(\theta_i)}{m}\bigg|\bigg\|\frac{1}{\sum_i^m \alpha(\theta_i)}\sum_i^m \alpha(\theta_i)f(\theta_i)\bigg\|\bigg]\\
&=&\EE_\theta\bigg[\bigg|1 - \frac{\sum_i^m \alpha(\theta_i)}{m}\bigg|\frac{1}{\sum_i^m \alpha(\theta_i)}\sum_i^m \alpha(\theta_i)|f(\theta_i)|\bigg]\le \EE\bigg[\bigg|1 - \frac{\sum_i^m \alpha_i}{m}\|f(\theta)\|_\infty\bigg|\bigg] \le \frac{\sqrt{C}\|f(\theta)\|_\infty}{\sqrt{m}}
\end{eqnarray*}

Then, we have achieve our conclusion that 
\begin{eqnarray*}
&&\EE_\theta\bigg[\bigg|\int \tilde q(\theta)f(\theta)d\theta - \int q(\theta)f(\theta)d\theta\bigg|\bigg]\\
&\le& \EE_\theta\bigg[\bigg|\int \hat q(\theta) f(\theta)d\theta  - \int q(\theta) f(\theta)d\theta\bigg|\bigg]  + \EE_\theta\bigg[\bigg|\int \tilde q(\theta)f(\theta)d\theta - \int \hat q(\theta)f(\theta)d\theta\bigg|\bigg]\\
&\le& \frac{2\sqrt{C}\|f(\theta)\|_\infty}{\sqrt{m}}
\end{eqnarray*}
\end{proof}

With the knowledge of $p(\theta)$ and $q(\theta)$, we set $q_t(\theta) = \alpha_t(\theta)p(\theta)$, the PMD algorithm will reduce to adjust $\alpha(\theta_i)$ for samples $\{\theta_i\}_{i=1}^m\sim \pi(\theta)$ according to the stochastic gradient. Plug the gradient formula into the exact update rule, we have
\begin{eqnarray*}
q_{t+1}(\theta) = \frac{q_t(\theta)\exp(-\gamma_t g_t(\theta))}{Z} = \frac{\alpha_t(\theta)\exp(-\gamma_t g_t(\theta))p(\theta)}{Z} = {\alpha_{t+1}(\theta)p(\theta)} 
\end{eqnarray*}
where $\alpha_{t+1}(\theta) = \frac{\alpha_t(\theta)\exp(-\gamma_t g_t(\theta))}{Z}$. Since $Z$ is constant, ignoring it will not effect the multiplicative update. 

Given the fact that the objective function, $L(q)$, is 1-\emph{strongly convex} w.r.t. the $KL$-divergence, we can immediately arrive at the following convergence results as appeared in~\citet{NemJudLanSha09}, if we are able to compute the prox-mapping in~Eq.\eq{eq:solve_exact_prox} exactly.
\begin{lemma}\label{lemma:exact_prox_mapping}
  One prox-mapping step~Eq.\eq{eq:solve_exact_prox} reduces the error by
  \begin{eqnarray*}
    \EE[KL(q^* || q_{t+1})] \leqslant (1-\gamma_t) \EE[KL(q^* || q_t)] + \frac{\gamma_t^2\EE\|g_t\|^2_\infty}{2}.
  \end{eqnarray*}
  With stepsize $\gamma_t = \frac{\eta}{t}$, it implies 
  $$\EE[KL(q^*||q_{T})] \le \max\bigg\{KL(q^*||q_1), \frac{\eta^2\EE\|g\|^2_\infty}{2\eta - 1}\bigg\}\frac{1}{T}
  $$
\end{lemma}

\begin{proof}
We could obtain the recursion directly from Theorem~\ref{thm:main1} by setting $\epsilon= 0$, which means solving the prox-mapping exactly, and the rate of convergence rate could be obtained by solving the recursion as stated in~\cite{NemJudLanSha09}. 
\end{proof}

\begin{lemma}\label{lemma:exact_prox_integral}
Let $q_t$ is the exact solution of the prox-mapping at $t$-step, then $\forall f(\theta): \RR^d\rightarrow \RR$, which is bounded and integrable, we have
$$
\EE\bigg[\bigg|\int q_t(\theta) f(\theta)d\theta - \int q(\theta) f(\theta)d\theta\bigg|\bigg]\le \max\bigg\{\sqrt{KL(q^*||q_1)}, \frac{\eta\EE\|g\|_\infty}{\sqrt{2\eta - 1}}\bigg\}\frac{\|f\|_\infty}{\sqrt{t}}.
$$
\end{lemma}

\begin{proof}
\begin{eqnarray*}
&&\EE\bigg[\bigg|\int q_t(\theta) f(\theta)d\theta - \int q^*(\theta) f(\theta)d\theta\bigg|\bigg] = 
\EE\|\langle q_t(\theta) - q^*(\theta)\, , f(\theta)\rangle_{L_2}\|_2 \\
&&\le \EE[\| q_t(\theta) - q^*(\theta)\|_1\|f\|_\infty] \le \|f\|_\infty \EE[\| q_t(\theta) - q^*(\theta)\|_1]\le \|f\|_\infty \EE\bigg[\sqrt{\frac{1}{2}KL(q^*||q_t)}\bigg]\\
&&\le \max\bigg\{\sqrt{KL(q^*||q_1)}, \frac{\eta\EE\|g\|_\infty}{\sqrt{2\eta - 1}}\bigg\}\frac{\|f\|_\infty}{\sqrt{t}}
\end{eqnarray*}
The second last inequality comes from Pinsker's inequality.
\end{proof}

{\bf Theorem}~\ref{thm:integral_convergence}
\emph{Assume the particle proposal prior $p(\theta)$ has the same support as the true posterior $q^*(\theta)$, \ie, $0\le q^*(\theta)/p(\theta)\le C$. With further condition about the model $\|p(x|\theta)^N\|_\infty\le \rho, \forall x$, then $\forall f(\theta): \RR^d\rightarrow \RR$ bounded and integrable, with stepsize $\gamma_t = \frac{\eta}{t}$, the PMD algorithm return $m$ weighted particles after $T$ iteration such that 
\begin{eqnarray*}
&&\EE\sbr{\abr{\int \qtil_t(\theta)f(\theta)d\theta - \int q^*(\theta)f(\theta)d\theta}}\\
&\le& \frac{2\sqrt{\max\{C, \rho\exp(\|g(\theta)\|_\infty)\}}\|f\|_\infty}{\sqrt{m}} + \max\bigg\{\sqrt{KL(q^*||\pi)}, \frac{\eta\EE\|g\|_\infty}{\sqrt{2\eta - 1}}\bigg\}\frac{\|f\|_\infty}{\sqrt{T}}.
\end{eqnarray*}
}

\emph{Proof for Theorem \ref{thm:integral_convergence}. }

We first decompose the error into optimization error and finite approximation error.
\begin{eqnarray*}
&&\EE\sbr{\abr{\int \qtil_t(\theta)f(\theta)d\theta - \int q^*(\theta)f(\theta)d\theta}}\\
&\le& \underbrace{\EE\sbr{\abr{\int \qtil_t(\theta)f(\theta)d\theta - \int q_t(\theta)f(\theta)d\theta}}}_{\text{finite approximation error }\epsilon_1}+ 
\underbrace{\EE\sbr{\abr{\int q_t(\theta)f(\theta)d\theta - \int q^*(\theta)f(\theta)d\theta}}}_{\text{optimization error }\epsilon_2}
\end{eqnarray*}

For the optimization error, by lemma~\ref{lemma:exact_prox_integral}, we have 
$$
\epsilon_2 \le \max\bigg\{\sqrt{KL(q^*||q_1)}, \frac{\eta\EE\|g\|_\infty}{\sqrt{2\eta - 1}}\bigg\}\frac{\|f\|_\infty}{\sqrt{t}}.
$$ 

Recall that 
\begin{eqnarray*}
&& q_{t}(\theta) = \frac{q_{t-1}(\theta)\exp(-\gamma_{t-1} g_{t-1}(\theta))}{Z}\\
&=& \frac{\alpha_{t-1}(\theta)\pi(\theta)(\alpha_{t-1}^{-\gamma_{t-1}}(\theta)p(x|\theta)^{N{\gamma_{t-1}}})}{Z} = \alpha_{t-1}^{1-\gamma_{t-1}}(\theta)\pi(\theta)\frac{p(x|\theta)^{N{\gamma_{t-1}}}}{Z}
\end{eqnarray*}
which results the update $\alpha_t(\theta) = \frac{\alpha_{t-1}^{1-\gamma_{t-1}}(\theta)p(x|\theta)^{N{\gamma_{t-1}}}}{Z}$. Notice $Z=\int q_{t}(\theta)\exp(-\gamma_t g_t(\theta))d\theta$, we have $\exp(-\gamma_t \|g_t(\theta)\|_\infty) \leqslant Z \leqslant \exp(\gamma_t \|g_t(\theta)\|_\infty)$. By induction, it can be show that $\|\alpha_t\|_\infty\le \max\{C, \rho\exp(\|g_t(\theta)\|_\infty)\}\le \max\{C, \rho\exp(\|g(\theta)\|_\infty)\}$. Therefore, by lemma~\ref{lemma:finite_approximation}, we have 
$$
\epsilon_1 \le \frac{2\sqrt{\max\{C, \rho\exp(\|g(\theta)\|_\infty)}\}\|f\|_\infty}{\sqrt{m}}.
$$

Combine $\epsilon_1$ and $\epsilon_2$, we achieve the conclusion.
\hfill$\blacksquare$

\textbf{Remark:} Simply induction without the assumption from the update of $\alpha_t(\theta)$ will result the upper bound of sequence $\|\alpha_t\|_\infty$ growing. The growth of sequence $\|\alpha_t\|_\infty$ is also observed in the proof~\cite{CriDou02} for sequential Monte Carlo on dynamic models. To achieve the uniform convergence rate for SMC of inference on dynamic system, \citet{CriDou02, GlaOud04} require the models should satisfy i), $\epsilon \nu(\theta_i)\le p(x_i| \theta_i)p(\theta_i|\theta_{i-1})\leq \epsilon^{-1} \nu(\theta_i)$, $\forall x$ where $\nu(\theta)$ is a positive measure, and ii), $\frac{\sup_{\theta} p(x|\theta)}{\inf_{\mu\in \Pcal}\langle \mu(\theta)p(\cdot|\theta) p(x|\cdot)\rangle}\le \rho$. Such rate is only for SMC on dynamic system. For static model, the trandistiion distribution is unknown, and therefore, no guarantee is provided yet.  With much simpler and more generalized condition on the model, \ie, $\|p(x|\theta)^N\|_\infty\le \rho$, we also achieve the uniform convergence rate for static model. There are plenties of models satisfying such condition. We list several such models below. 
\begin{enumerate}
\item logistic regression, $p(y |x, w) = \frac{1}{1 + \exp(-yw^\top x)}$, and $\|p(y | x, w)\|_\infty\le 1$.
\item probit regression, $p(y = 1| x, w) = \Phi(w^\top x)$ where $\Phi(\cdot)$ is the cumulative distribution function of normal distribution. $\|p(y| x, w)\|_\infty\le 1$. 
\item multi-category logistic regression, $p(y = k|x, W) = \frac{\exp(w_k^\top x)}{\sum_{i=1}^K \exp(w_k^\top x)}$, and $\|p(y | x, W)\|_\infty\le 1$.
\item latent Dirichlet allocation, 
\begin{eqnarray*}
p(x_d |  \theta_d, \Phi) &=& \EE_{z_d\sim p(z_d|\theta_d) }[p(x_d|z_d, \Phi)]\\ 
p(x_d|z_d, \Phi) &=& \prod_{n=1}^{N_d}\prod_{w=1}^{W}\prod_{k=1}^{K} \Phi_{kw}^{z_{dnk}x_{dnw}}\\
p(z_d|\theta_d) &=& \prod_{n=1}^{N_d}\prod_{k=1}^{K} \theta_{dk} ^{z_{dnk}} 
\end{eqnarray*}
and $\|p(x_d | \theta_d, \Phi)\|_\infty \le \max_{z_d}\|p(x_d|z_d, \Phi)\|_\infty\le 1$.
\item linear regression, $p(y|w, x) = \frac{1}{\sigma\sqrt{2\pi}}\exp(-(y - w^\top x)^2/2\sigma^2)$, and $\|p(y|w, x)\|_\infty\le \frac{1}{\sigma\sqrt{2\pi}}$. 
\item Gaussian model and PCA, 
$p(x|\mu, \Sigma) = (2\pi\det(\Sigma))^{-\frac{1}{d}}\exp\bigg(-\frac{1}{2}(x - \mu)^\top \Sigma(x - \mu)\bigg)$, and $\|p(x|\mu, \Sigma)\|_\infty\le (2\pi\det(\Sigma))^{-\frac{1}{d}}$. 
\end{enumerate}

\section{Error Bound of Weighted Kernel Density Estimator}\label{appendix:bound_weighted_kde}

Before we start to prove the finite convergence in general case, we need to characterize the error induced by weighted kernel density estimator. In this section, we analyze the error in terms of both $L_1$ and $L_2$ norm, which are used for convergence analysis measured by $KL$-divergence in Appendix~\ref{appendix:density_convergence} .

\subsection{$L_1$-Error Bound of Weighted Kernel Density Estimator}

We approximate the density function $q(\theta)$ using the weighted kernel density estimator $\qtil(\theta)$ and would like to bound the $L_1$ error, i.e. $\|\qtil(\theta)-q(\theta)\|_1$ both in expectation and with high probability. We consider an unnormalized kernel density estimator as the intermediate quantity
\begin{align*}
  \varrho_m(\theta) = \frac{1}{m}\sum_{i=1}^m \omega(\theta_i)K_h(\theta,\theta_i)
\end{align*}
Note that $\EE[\varrho_m(\theta)] =\EE_{\theta_i}[\omega(\theta_i)K_h(\theta,\theta_i)]=q\star K_h$.
Then the error can be decomposed into three terms as 
\begin{equation*}
  \hspace{-3mm}
  \begin{array}{rcll}
    \epsilon := \EE \nbr{\qtil(\theta) - q(\theta)}_1 \leqslant \underbrace{\EE\nbr{\qtil(\theta) - \varrho_m(\theta)}_1}_{\text{normalization error}} + \underbrace{\EE\nbr{\varrho_m(\theta) - \EE\,\varrho_m(\theta)}_1 }_\text{sampling error (variance)} 
    + \underbrace{\nbr{\EE\,\varrho_m(\theta) - q(\theta)}_1}_\text{approximation error (bias)}
  \end{array}
\end{equation*}
We now present the proof for each of these error bounds.

To formally show that, we begin by giving the definition of a special class of kernels and H\"{o}lder classes of densities that we consider. 
\begin{definition}[$(\beta;\mu,\nu,\delta)$-valid density kernel]
We say a kernel function $K(\cdot)$ is a $(\beta;\mu,\nu)$-valid  density kernel, if $K(\theta, \theta) = K(\theta - \theta)$ is a bounded, 
compactly supported kernel such that
\begin{enumerate}[noitemsep,nolistsep]
\item[(i)] $\int K(z)dz = 1$
\item[(ii)] $\int |K(z)|^r dz\leq \infty$ for any $r\geq 1$, particularly, $\int K(z)^2\,dz\leq \mu^2$ for some $\mu>0$.
\item[(iii)]  $\int z^s K(z)dz=0$, for any $s=(s_1,\ldots,s_d)\in\mathbb{N}^d$ such that $1\leq |s|\leq \lfloor{\beta}\rfloor$. In addition, 
$\int \|z\|^\beta|K(z)|dz\leq\nu$  for some $\nu>0$. 
\end{enumerate}

\end{definition}
For simplicity, we sometimes call $K(\cdot)$ as a $\beta$-valid density kernel if the constants $\mu$ and $\nu$ are not specifically given.  Notice that all spherically symmetric compactly supported probability density and product kernels based on compactly supported symmetric univariate densities satisfy the  conditions. For instance, the kernel $K(\theta) = (2\pi)^{-d/2}\exp(-\nbr{\theta}^2/2)$ satisfies the conditions with $\beta = \infty$, and it is used through out our experiments. 
Furthermore, we will focus on a class of smooth densities 
\begin{definition}[$(\beta;\mathcal{L})$-H\"{o}lder density function]
We say a density function $q(\cdot)$ is a $(\beta;\mathcal{L})$-H\"{o}lder density function  if  function $q(\cdot)$ is $\lfloor \beta\rfloor$-times continuously differentiable on its support $\Omega$ and satisfies
\begin{enumerate}[noitemsep,nolistsep]
\item[(i)] for any $z_0$, there exists $L(z_0)>0$ such that
$$ |q(z)-q_{z_0}^{(\beta)}(z)|\leq L(z_0)\|z-z_0\|^{\beta},\forall z\in \Omega $$
where $q_{z_0}^{(\beta)} $ is the $\lfloor \beta\rfloor$-order Taylor approximation, i.e.
$$q_{z_0}^{(\beta)}(z):=\sum_{s=(s_1,\ldots,s_d): |s|\leq\lfloor\beta\rfloor}\frac{(z-z_0)^s}{s!}D^s q(z_0);$$
\item[(ii)] in addition, the integral $\int L(z) dz \leq \mathcal{L}$.
\end{enumerate}
$f\in C^{\beta}_\mathcal{L}(\Omega)$ means $f$ is $(\beta;\mathcal{L})$-H\"{o}lder density function.
\end{definition}

Then given the above setting for the kernel function and the smooth densities, we can characterize the error of the weighted kernel density estimator as follows.

\subsubsection{KDE error due to bias}

\begin{lemma}[Bias] \label{lem:bias}
If  $q(\cdot)\in C^{\beta}_{\mathcal{L}}(\Omega)$ and $K$ is a $(\beta;\mu,\nu)$-valid density kernel, then
  $$\nbr{q(\theta) - \EE[\varrho_m(\theta)]}_1 \leqslant  \nu \mathcal{L}h^{\beta}.$$
\end{lemma}

\emph{Proof}
The proof of this lemma follows directly from Chapter 4.3 in \cite{WanJon95}. 
\begin{eqnarray*}
|\EE[\varrho_m(\theta)]-q(\theta)|  &=& |q\star K_h(\theta)-q(\theta)|\\
&=& \int \frac{1}{h^d}K(\frac{z-\theta}{h})q(z)dz -q(\theta)\\
&=& \int \frac{1}{h^d}K(\frac{z}{h})[q(\theta+z)-q(\theta)]dz\\
&=&\int K(z)[q(\theta+hz)-q(\theta)]dz\\
&\leq&\left|\int K(z)[q(\theta+hz)-q^{(\beta)}_{\theta}(\theta+hz)]dz\right|+\int \left|K(z)[q^{(\beta)}_{\theta}(\theta+hz)-q(\theta)]dz\right|\\
&\leq&L(\theta)\int|K(z)\|hz\|^\beta dz+\left|\int K(z)[q^{(\beta)}_{\theta}(\theta+hz)-q(\theta)]dz\right|\\
\end{eqnarray*}
Note that $q^{(\beta)}_{\theta}(\theta+hz)-q(\theta)$ is a polynomial of degree at most $\lfloor\beta\rfloor$ with no constant, by the definition of $(\beta;\mu,\nu)$-valid density kernel, the second term is zero. Hence, we have $|\EE[\varrho_m(\theta)]-q(\theta)| \leq \nu L(\theta)h^{\beta}$, and therefore 
$$\|\EE[\varrho_m(\theta)]-q(\theta)\|_1\leq\nu h^\beta\int L(\theta)d\theta \leq \nu\mathcal{L} h^\beta.$$
\hfill$\blacksquare$

\subsubsection{KDE error due to variance}
The variance term can be bounded using similar techniques as in \cite{DevGyo85}.

\begin{lemma}[Variance] \label{lem:variance}
Assume $\omega\sqrt{p}\in L_1$ with bounded support, then
  $$\EE\nbr{\varrho_m(\theta)-\EE[\varrho_m(\theta)]}_1 \leqslant  \frac{\mu}{\sqrt{m}h^{\frac{d}{2}}}\int \omega\sqrt{p}\,d\theta +o((mh^d)^{-\frac{1}{2}}).$$
\end{lemma}

\emph{Proof}
For any $\theta$, we have
\begin{align*}
\sigma^2(\theta):&=\EE\left[(\varrho_m(\theta)-\EE[\varrho_m(\theta)])^2\right]\\
&= \frac{1}{m}\sum_{i=1}^m\EE[\omega^2(\theta_i)K_h^2(\theta,\theta_i)]- (q\star K_h)^2\leq\frac{(\omega^2q)\star K_h^2}{m}
\end{align*}
Denote $\mu(K):=\sqrt{\int K(\theta)^2\,d\theta}$ and kernel $K^+(\theta)=\frac{K^2(\theta)}{\mu(K)^2}$, then $\mu(K)\leq\mu$, $\int K^+d\theta=1$ and
$$K^+_h(\theta)=\frac{1}{h^d}K^+(\theta/d)=\frac{1}{h^d}\frac{K(\theta/h)K(\theta/h)}{\mu^2(K)}=\frac{h^d}{\mu^2(K)}K_h^2(\theta)$$
Hence,
$$\sigma^2(\theta)\leq \frac{\mu^2(K)(\omega^2p)\star K_h^+}{mh^d}\leq\frac{\mu^2[(\omega^2p)\star K_h^+-\omega^2p]}{mh^d}+\frac{\mu^2( \omega^2p)}{mh^d}.$$
Note that $\sigma(\theta)=\sqrt{\EE\left[(\varrho_m(\theta)-\EE[\varrho_m(\theta)])^2\right]}\geq \EE|\varrho_m(\theta)-\EE[\varrho_m(\theta)]|$, hence
\begin{align*}
&\quad\; \EE\nbr{\varrho_m(\theta)-\EE[\varrho_m(\theta)]}_1 \\
&=\int\EE|\varrho_m(\theta)-\EE[\varrho_m(\theta)]|\, d\theta\leq \int\sigma(\theta)\,d\theta\\
&\leq \int\sqrt{\frac{\mu^2(\omega^2p)\star K_h^+-\omega^2p]}{mh^d}}+\sqrt{\frac{\mu^2( \omega^2p)}{mh^d}}\,d\theta\\
&\leq \frac{\mu}{\sqrt{m}h^{d/2}}\left[\int \sqrt{\omega^2p}\,d\theta+\int \sqrt{ (\omega^2p)\star K_h^+-\omega^2p}\,d\theta\right]\\
&\leq \frac{\mu}{\sqrt{m}h^{d/2}}\bigg[\int \omega \sqrt{p}\,d\theta+  \sqrt{|\Omega|}\cdot\sqrt{\int \left|(\omega^2p)\star K_h^+-\omega^2p\right|\,d\theta}\bigg]
\end{align*}
From Theorem 2.1 in \cite{DevGyo85}, we have $\int \left|(\omega^2p)\star K_h^+-\omega^2p\right|\,d\theta=o(1)$. 
Therefore, we conclude that
  $$\EE\nbr{\varrho_m(\theta)-\EE[\varrho_m(\theta)]}_1 \leqslant  \frac{\mu}{\sqrt{m}h^{d/2}}\|\omega\sqrt{p}\|_1+o((mh^d)^{-\frac{1}{2}}).$$
\hfill$\blacksquare$

\subsubsection{KDE error due to normalization}
The normalization error term can be easily derived based on the variance. 

\begin{lemma}[Normalization error] \label{lem:normalization} 
Assume $\omega\sqrt{p}\in L_2$
  $$\EE\nbr{\qtil(\theta) - \varrho_m(\theta)}_1 \leqslant  \frac{1}{\sqrt{m}}\rbr{\int\omega^2(\theta)p(\theta)\,d\theta}^{1/2}.$$
  \vspace{-3mm}
\end{lemma}

\emph{Proof}
Denote $\omega_i:=\omega(\theta_i)$, then $\EE[\omega_i]=\int \omega(\theta)p(\theta)\,d\theta=1$ and $\EE[\omega_i^2]=\int\omega^2(\theta)p(\theta)\,d\theta$, for any $i=1,\ldots, m$. Hence, 
$$\EE|\frac{1}{m}\sum_{i=1}^m\omega_i-1|^2= \frac{1}{m}\int \omega^2(\theta)p(\theta)\,d\theta.$$
Recall that $\qtil(\theta) =\frac{1}{{\sum_{i=1}^m\omega_i}}\sum_{i=1}^m\omega_iK_h(\theta,\theta_i)$ and $\varrho_m(\theta)=\frac{1}{m}\sum_{i=1}^m\omega_iK_h(\theta,\theta_i)$. 
\begin{align*}
&\quad\; \EE\nbr{\qtil(\theta) - \varrho_m(\theta)}_1\\
&\leq \EE\nbr{\frac{1}{{\sum_{i=1}^m\omega_i}}\sum_{i=1}^m\omega_iK_h(\theta,\theta_i)-    \frac{1}{m}\sum_{i=1}^m\omega_iK_h(\theta,\theta_i) }_1\\
&\leq \EE\nbr{\left|1-\frac{\sum_{i=1}^m\omega_i}{m}\right| \frac{1}{{\sum_{i=1}^m\omega_i}}\sum_{i=1}^m\omega_iK_h(\theta,\theta_i) }_1\\
&\leq  \EE\left|1-\frac{\sum_{i=1}^m\omega_i}{m}\right|\cdot\|K_h(\theta)\|_1
\end{align*}
Since $\|K_h\|_1=\int \frac{1}{h^d}K(\theta/h)\,d\theta=\int K(\theta)\,d\theta=1$, we have 
\begin{eqnarray*}
\EE\nbr{\qtil(\theta) - \varrho_m(\theta)}_1\leq \frac{1}{\sqrt{m}}\sqrt{\int\omega^2(\theta)p(\theta)\,d\theta}= \frac{1}{\sqrt{m}}\|w\sqrt{p}\|_2
\end{eqnarray*}
\hfill$\blacksquare$

\subsubsection{KDE error in expectation and with high probability}
Based on the above there lemmas, namely, Lemma \ref{lem:bias} -  \ref{lem:normalization}, we can immediately arrive at the bound of the $L_1$ error in expectation as stated in Theorem \ref{thm:error1}. We now provide the proof for the high probability bound as stated below.\\
\begin{corollary}[Overall error in high probability]\label{cor:error2}
Besides the above assumption, let us also assume that $\omega(\theta)$ is bounded, i.e. there exists $0< B_1\leq B_2<\infty$ such that $B_1\leq \omega(\theta)\leq B_2, \forall \theta$. Then, with probability at least $1-\delta$,
\begin{eqnarray*}
\|\qtil(\theta) - q(\theta)\|_1\leq  \nu \mathcal{L}h^{\beta} + \frac{\mu}{\sqrt{m}h^{d/2}}\|\omega\sqrt{p}\|_1 + \frac{1}{\sqrt{m}}\|\omega\sqrt{p}\|_2
+\frac{1}{\sqrt{m}}\sqrt{8B_1B_2\log(1/\delta)}+o((mh^d)^{-\frac{1}{2}}).
\end{eqnarray*}
\end{corollary}

\emph{Proof}
We use McDiarmid's inequality to show that the  function $f(\Theta)= \nbr{\qtil(\theta) - q(\theta)}_1$,  defined on the random data $\Theta=(\theta_1,\ldots,\theta_m)$, is concentrated on the mean.  Let $\tilde\Theta=(\theta_1,\ldots,\tilde\theta_j,\ldots,\theta_m)$.  We denote $\omega=(\omega(\theta_1),\ldots,\omega(\theta_m))$ and $\tilde \omega=(\omega(\theta_1),\ldots,\omega(\tilde\theta_j),\ldots,\omega(\theta_m))$. Denote $k=(K_h(\theta,\theta_1),\ldots, K_h(\theta,\theta_m))$ and $\tilde k=(K_h(\theta,\theta_1),\ldots,K_h(\theta,\theta'_j),\ldots, K_h(\theta,\theta_m))$.
We first show that $|f(\Theta)-f(\Theta')|$ is bounded.  
\begin{eqnarray*}
&&|f(\Theta)-f(\Theta')\\
&=&\big|\nbr{\qtil_{\Theta}(\theta) - q(\theta)}_1-\nbr{\qtil_{\tilde\Theta}(\theta) - q(\theta)}_1\big|\\
&\leq& \nbr{\qtil_{\Theta}(\theta) - \qtil_{\tilde\Theta}(\theta)}_1\\
&=& \nbr{\frac{\sum_{i=1}^m\omega_ik_i}{\sum_{i=1}^m\omega_i}-\frac{\sum_{i=1}^m\tilde\omega_i\tilde k_i}{\sum_{i=1}^m\tilde\omega_i}}_1\\
&\leq&\nbr{\frac{(\tilde\omega_j-\omega_j)\cdot(\sum_{i=1}^m\omega_i k_i)
-(\sum_{i=1}^m\omega_i)(\tilde \omega_i \tilde k_i-\omega_j k_j)}
{(\sum_{i=1}^m\omega_i)\cdot(\sum_{i=1}^m\tilde\omega_i)}}_1\\
&\leq &\nbr{\frac{\tilde\omega_j-\omega_j}{(\sum_{i=1}^m\tilde\omega_i)}}_\infty+ \nbr{\frac{\tilde \omega_i \tilde k_i-\omega_j k_j}
{\sum_{i=1}^m\tilde\omega_i}}_1\\
&\leq& \frac{2B_1B_2}{m}+\frac{2B_1B_2}{m}\leq \frac{4B_1B_2}{m}
\end{eqnarray*}
Invoking the McDiamid's inequality, we have
\begin{eqnarray*}
\text{Pr} \left(f(\Theta)-\EE_{\Theta}[f(\Theta)]\geq \epsilon\right)\leq \exp\left\{-\frac{m\epsilon^2}{8B_1^2B_2^2}\right\}, \forall \epsilon>0
\end{eqnarray*}
which implies the corollary. \\

\hfill$\blacksquare$

\subsection{$L_2$-Error  Bound of  Weighted Kernel Density Estimator}

Following same argument yields also similar $L_2$-error bound of the weighted kernel density estimator, i.e. $\|\qtil(\theta)-q(\theta)\|_2$. For completeness and also for future reference, we provide the exact statement of the bound below in line with  Theorem \ref{thm:error1} and  Corollary \ref{cor:error2}.

\begin{theorem}[$L_2$-error in expectation]\label{thm:error3}
Let  $q=\omega p\in C^{\beta}_{\mathcal{L}}(\Omega)$ and $K$ be a $(\beta;\mu,\nu)$-valid density kernel. Assume that  $\omega^2p\in  L_2$ and has bounded support. Then
\begin{eqnarray*}
\EE \nbr{\qtil(\theta) - q(\theta)}_2^2 \leqslant  2(\nu h^{\beta}\mathcal{L})^2+ \frac{8\mu^2}{mh^d}\|\omega\sqrt{p}\|_2^2+ o((mh^d)^{-1}). 
\end{eqnarray*}
\end{theorem}

\emph{Proof for Theorem \ref{thm:error3}. }
The square $L_2$-error  can also be decomposed into three terms. 
\begin{equation*}
\EE \nbr{\qtil(\theta) - q(\theta)}_2^2 \leqslant  4\underbrace{\EE\nbr{\qtil(\theta) - \varrho_m(\theta)}_2}_{ \text{normalization error}}  + 4\underbrace{\EE\nbr{\varrho_m(\theta) - \EE\,\varrho_m(\theta)}_2^2 }_{\text{sampling error (variance)}} + 2\underbrace{ \nbr{\EE\,\varrho_m(\theta) - q(\theta)}_2^2}_{\text{approximation error (bias)}}
\end{equation*}
This uses the inequality $(a+b+c)^2\leq 2a^2+4b^2+4c^2$ for any $a,b,c$.  From Lemma \ref{lem:bias}, we already have  $|\EE[\varrho_m(\theta)]-q(\theta)| \leq L(\theta)\int|K(z)\|hz\|^\beta dz,\forall \theta$.
Hence,  
\begin{eqnarray}\label{eq:L2bias}
\|\EE[\varrho_m(\theta)]-q(\theta)\|_2^2 \leq \nu^2 h^{2\beta}\int L^2(\theta)d\theta\leq (\nu h^{\beta}\mathcal{L})^2.
\end{eqnarray}
From proof for  Lemma \ref{lem:variance}, we have 
\begin{eqnarray}\label{eq:L2var}
&\quad\; \EE\nbr{\varrho_m(\theta)-\EE[\varrho_m(\theta)]}_2^2=\int\EE|\varrho_m(\theta)-\EE[\varrho_m(\theta)]|^2\, d\theta\leq \int\sigma^2(\theta)\,d\theta\\
&\leq \int \frac{\mu^2[(\omega^2p)\star K_h^+-\omega^2p]}{mh^d}+\frac{\mu^2( \omega^2p)}{mh^d}\,d\theta\leq \frac{\mu^2}{mh^{d}}\|\omega\sqrt{p}\|_2^2+o((mh^d)^{-1})
\end{eqnarray}
In addition, we have for the normalization error term, 
\begin{eqnarray}\label{eq:L2norm}
& \EE\nbr{\qtil(\theta) - \varrho_m(\theta)}_2^2\leq \EE\nbr{\left(1-\frac{\sum_{i=1}^m\omega_i}{m}\right) \frac{\sum_{i=1}^m\omega_iK_h(\theta,\theta_i) }{{\sum_{i=1}^m\omega_i}}}_2^2\\ \nonumber
&\leq \EE\left|1-\frac{\sum_{i=1}^m\omega_i}{m}\right|^2\cdot\|K_h\|_2^2\leq \frac{\mu^2}{mh^d}\|\omega\sqrt{p}\|_2^2
\end{eqnarray}
Combining equation (\ref{eq:L2bias}) , (\ref{eq:L2var}) and (\ref{eq:L2norm}), it follows that
\begin{eqnarray*}
\EE \nbr{\qtil(\theta) - q(\theta)}_2^2 \leqslant  2(\nu h^{\beta}\mathcal{L})^2+ \frac{8\mu^2}{mh^d}\|\omega\sqrt{p}\|_2^2+ o((mh^d)^{-1}). 
\end{eqnarray*}
\hfill$\blacksquare$

\begin{corollary}[$L_2$-error in high probability]\label{cor:error4}
Besides the above assumption, let us also assume that $\omega(\theta)$ is bounded, i.e. there exists $0< B_1\leq B_2<\infty$ such that $B_1\leq \omega(\theta)\leq B_2, \forall \theta$. Then, with probability at least $1-\delta$, 
\begin{eqnarray*}
\|\qtil(\theta) - q(\theta)\|_2^2\leq  2(\nu h^{\beta}\mathcal{L})^2+ \frac{8\mu^2}{mh^d}\|\omega\sqrt{p}\|_2^2+ o((mh^d)^{-1})+\frac{16B_1B_2\mu^2}{m}\sqrt{\log(1/\delta)}.
\end{eqnarray*}
\end{corollary}

\emph{Proof for Theorem \ref{cor:error4}. } 
Use McDiarmid's inequality similar as proof for Corollary \ref{cor:error2}. 
\hfill$\blacksquare$

\section{Convergence Analysis for Density Approximation}\label{appendix:density_convergence}

In this section, we consider the rate of convergence for the entire density measured by $KL$-divergence. We start with the following lemma that show the renormalization does not effect the optimization in the sense of optimal, and we show the importance weight $\omega_t(\theta) = \frac{\exp(-\gamma_t g_t(\theta))}{Z}$ at each step are bounded under proper assumptions.  Moreover, the error of the prox-mapping at each step incurred by the weighted density kernel density estimation is bounded. 

\begin{lemma}\label{lem:renormalization}
Let $\zeta = \int_{\setminus\Omega} \qtil_{t}d\theta$, $\qhat_{t} = \frac{\qtil_{t}}{1-\zeta}$ is a valid density on $\Omega$, then, $\qtil_t^+ = \qhat_t^+$, where $\qtil_{t}^+ := \argmin_{q\in \Pcal(\Omega)} F_t(q; \qtil_t)$, $\qhat_{t}^+ := \argmin_{q\in \Pcal(\Omega)} F_t(q; \qhat_t)$, and $F_t(q; q') := \inner{q}{\gamma_t g}_{L_2} + KL(q\| q')$.
\end{lemma}

\emph{Proof for Lemma \ref{lem:renormalization}. }
The minima of prox-mapping is not effected by the renormalization. Indeed, such fact can be verified by comparing to $\qtil_{t}^+ = \argmin F_t(q; \qtil_t)$ and $\qhat_{t}^+ = \argmin F_t(q; \qhat_t)$, respectively. 
\begin{eqnarray*}
\qhat_t^+ = \frac{(\frac{1}{1-\zeta}\qtil_t)^{1-\gamma_t}p(\theta)^\gamma_tp(x_t|\theta)^{N\gamma_t}}{\int (\frac{1}{1-\zeta}\qtil_t)^{1-\gamma_t}p(\theta)^\gamma_tp(x_t|\theta)^{N\gamma_t}d\theta} = \frac{\qtil_t^{1-\gamma_t}p(\theta)^\gamma_tp(x_t|\theta)^{N\gamma_t}}{\int \qtil_t^{1-\gamma_t}p(\theta)^\gamma_tp(x_t|\theta)^{N\gamma_t}d\theta} = \qtil_t^+ 
\end{eqnarray*}

\hfill$\blacksquare$

Due to the fact, we use $\qtil_t^+$ following for consistency. Although the algorithm updates based on $\qtil_t$, it is implicitly doing renoramlization after each update. We will show that $\qhat_{t+1}$ is an $\epsilon$-inexact prox-mapping. 

\begin{lemma}
\label{lem:ratio_bound}
Assume for all mini-batch of examples $\|g_t(\theta)\|^2_\infty \leqslant M^2$, then we have
\begin{enumerate}
\item[(a)]$\exp(-2\gamma_t M)\leqslant\omega_t(\theta)=\frac{\qtil^+_{t}(\theta)}{\qhat_t(\theta)}\leqslant \exp(2\gamma_t M),$
\item[(b)] $\|\nabla F_t(\qtil^+_t; \qhat_t)\|_\infty \leqslant 3\gamma_tM.$
\end{enumerate}
\end{lemma}

\emph{Proof for Lemma \ref{lem:ratio_bound}. } Let $Z:=\int q_{t}(\theta)\exp(-\gamma_t g_t(\theta))d\theta$. We have $\exp(-\gamma_t M) \leqslant Z \leqslant \exp(\gamma_t M)$.\\
(a) Since $\|g_t(\theta)\|^2_\infty \leqslant M^2$, we have
\begin{eqnarray*}
 \exp(-2\gamma_t M)\leqslant\omega_t(\theta)=\frac{\qtil^+_{t}(\theta)}{\qhat_t(\theta)} =  \frac{\exp(-\gamma_t g_t(\theta))}{Z} \leq  \exp(2\gamma_t M).
\end{eqnarray*}
(b) Also, because $\nabla F_t(q_t^+)=\gamma_tg_t+\log\frac{\qtil^+_t}{\qhat_t}=\gamma_tg_t+\log(\omega_t)$, it immediately follows
\begin{eqnarray*}
\|\nabla F_t(\qtil^+_t; \qhat_t)\|_\infty&=&\|\gamma_t g_t + \log(\omega_t)\|_\infty \le \gamma_t\|g_t\|_\infty + \|\log(\omega_t)\|_\infty\le \gamma_t M + (2\gamma_tM)=3\gamma_tM.
\end{eqnarray*}
\hfill$\blacksquare$

\begin{lemma}\label{lem:error_bound}
Let $\epsilon_t:=F_t(\qhat_{t+1}; \qhat_t)-F_t(\qtil_t^+;\qhat_t)$, which implies $\qhat_{t+1}\in P_{\qtil_t}^{\epsilon_t}(\gamma_t g_t)$. Let the bandwidth at step $t$ satisfies
$$h_t=O(1)m_t^{-1/(d+2\beta)},$$  
one can guarantee that
\begin{eqnarray*}
\EE_{\theta}  [\epsilon_t|x_{[t-1]},\theta_{[t-1]}]\leq& O(1)(\mu^2+\nu^2\mathcal{L}^2)\mu^2\Delta m_t^{-\frac{2\beta}{d+2\beta}}+O(1)M(\mu+\nu\mathcal{L})\gamma_t m_t^{-\frac{\beta}{d+2\beta}}
\end{eqnarray*}
In addition, with probability at least $1-2\delta$ in $\theta_t|x_{[t-1]},\theta_{[t-1]}$, we have
 \begin{eqnarray*}
\epsilon_t\leq O(1)(\mu^2\sqrt{\log(1/\delta)}+\nu^2\mathcal{L}^2)\mu^2\Delta m_t^{-\frac{2\beta}{d+2\beta}}+O(1)M(\mu+\nu\mathcal{L}+\sqrt{\log(1/\delta)})\gamma_t m_t^{-\frac{\beta}{d+2\beta}}
\end{eqnarray*}
where $O(1)$ is some constant. 
\end{lemma}

\emph{Proof for Lemma \ref{lem:error_bound}. } 

Note that since $\qtil^+_{t}(\theta)=\qtil_{t}(\theta) \exp(-\gamma_t g_t(\theta)) / Z$, where $\qtil_{t}(\theta)=\sum_{i=1}^{m_{t}}\alpha_i K_{h_t}(\theta-\theta_i)$, and $g_t(\theta)=\log(\qtil_t)-\log(p)-N\log(p(x_t|\theta))$. By our assumption, we have $\qtil_t\in C^{\beta}_{\mathcal{L}}(\Omega)$ and $\exp(-\gamma_tg_t)\in C^{\beta}_{\mathcal{L}}(\Omega)$; hence, $\qtil^+_{t}\in C^{\beta}_{\mathcal{L}}(\Omega)$. Invoking the definition of function $F_t(\cdot; \, \qhat_t)$, we have
\begin{eqnarray*}
F_t(\qhat_{t+1}; \, \qhat_t)-F_t(\qtil_t^+; \, \qhat_t)&=&KL(\qhat_{t+1}||\qtil^+_t)+\langle\nabla F_t(\qtil^+_t; \, \qhat_t), \qhat_{t+1}-\qtil^+_t\rangle_{L_2}\\
&\leq&KL(\qhat_{t+1}||\qtil^+_t)+3\gamma_tM \|\qtil^+_t-\qhat_{t+1}\|_1 \\
&\leq&\int \frac{(\qhat_{t+1}-\qtil_t^+)^2}{\qtil_{t}^+}d\theta+3\gamma_tM \|\qtil^+_t-\qhat_{t+1}\|_1 \\
&\leq&{\Delta}\|\qhat_{t+1}-\qtil_t^+\|_2^2+3\gamma_tM \|\qtil^+_t-\qhat_{t+1}\|_1 
\end{eqnarray*}

Based on the definition of $\qhat_{t+1}$, we have 
\begin{eqnarray*}
\|\qtil^+_t-\qhat_{t+1}\|_1 &=& \bigg\|\frac{1}{1-\zeta} \qtil_{t+1} - \qtil_t^+\bigg\|_1 =  \frac{1}{1-\zeta}\|\qtil_{t+1} - \qtil_t^+ + \zeta\qtil_t^+\|_1\le \frac{1}{1-\zeta}\|\qtil_{t+1} - \qtil_t^+ \|_1 + \frac{\zeta}{1-\zeta}\\
&=& \|\qtil_{t+1} - \qtil_t^+ \|_1 +\zeta + o(\zeta  + \|\qtil_{t+1} - \qtil_t^+ \|_1).
\end{eqnarray*}
Similarly, 
\begin{eqnarray*}
\|\qtil^+_t-\qhat_{t+1}\|_2^2 &=&\bigg\|\frac{1}{1-\zeta}(\qtil_{t+1} - \qtil_t^+) + \frac{\zeta}{1 - \zeta}\qtil_t^+\bigg\|_2^2 \\
&\le& \frac{2}{(1-\zeta)^2}\|\qtil_{t+1} - \qtil_t^+\|_2^2 + \frac{2\zeta^2}{(1-\zeta)^2}\|\qtil_t^+\|_2^2\\
&\le& 2(1+\zeta)^2\|\qtil_{t+1} - \qtil_t^+\|_2^2 + 2\zeta^2\|\qtil_t^+\|_2^2 + o(\zeta^2\|\qtil_t^+\|_2^2+\zeta^2\|\qtil_{t+1} - \qtil_t^+\|_2^2)
\end{eqnarray*}

Recall $\zeta = 1 - \int_\Omega \qtil_{t+1} =\langle 1, \qtil_t^+ - \qtil_{t+1} \rangle \le \|\qtil_{t+1} - \qtil_t^+ \|_1$, we can simplify the $L_1$ and $L_2$ error as
\begin{eqnarray*}
\|\qhat_{t+1} - \qtil_t^+\|_1 &=& 2\|\qtil_{t+1} - \qtil_t^+ \|_1 + o(\|\qtil_{t+1} - \qtil_t^+ \|_1),\\
\|\qtil^+_t-\qhat_{t+1}\|_2^2 &\le & 2\|\qtil_{t+1} - \qtil_t^+\|_2^2 + 2\|\qtil_t^+\|_2^2\|\qtil_{t+1} - \qtil_t^+ \|_1^2 + o(\|\qtil_{t+1} - \qtil_t^+\|_2^2 + \|\qtil_t^+\|_2^2\|\qtil_{t+1} - \qtil_t^+ \|_1^2 ) \\
&\le & (2 + 2\|\qtil_t^+\|_2^2)\|\qtil_{t+1} - \qtil_t^+ \|_2^2 + o(\|\qtil_{t+1} - \qtil_t^+\|_2^2).
\end{eqnarray*}
The last inequality for $L_2$ error comes from Jensen's inequality. We argue that $\|\qtil_t^+\|_2^2$ is finite. Indeed, 
\begin{eqnarray*}
\|\qtil_t^+\|_2^2 &=& \int(\qtil_t^+)^2d\theta = \int \frac{\qtil_t^2\exp(-2\gamma_tg_t)}{Z^2}d\theta \le \bigg\|\frac{\exp(-2\gamma_tg_t)}{Z^2}\bigg\|_\infty\int\qtil_t^2d\theta\\
&\le&\exp(4\gamma_tM)\bigg(\sum_{i, j}\alpha_i^t\alpha_j^t \int K_h(\theta- \theta_i)K_h(\theta- \theta_j)d\theta\bigg)\\
&\le&\exp(4\gamma_tM)\bigg(\sum_{i, j}\alpha_i^t\alpha_j^t \|K_h(\theta- \theta_i)\|_2\|K_h(\theta- \theta_j)\|_2\bigg)
\le \exp(4\gamma_tM)\mu^2\|\alpha^t\|_1\|\alpha^t\|_\infty \le \exp(4\gamma_tM)\mu^2
\end{eqnarray*}

Therefore, we have
\begin{eqnarray*}
\epsilon_t &\le& (2\Delta+ 2\Delta \mu^2\exp(4\gamma_t M)) \|\qtil_{t+1} - \qtil_t^+\|_2^2 + 6\gamma_tM\|\qtil_{t+1} - \qtil_t^+\|_1\\
&&+ o(\|\qtil_{t+1} - \qtil_t^+\|_2^2 + \gamma_t\|\qtil_{t+1} - \qtil_t^+\|_1)
\end{eqnarray*}

Applying the result of Theorem \ref{thm:error1} and \ref{thm:error3} for $\qhat_{t+1}$ and $\qtil_t^+$ we have
\begin{eqnarray*}
&\EE_{\theta}  [\epsilon_t|x_{[t-1]},\theta_{[t-1]}]\leq& (2{\Delta} + 2\Delta \mu^2\exp(4\gamma_t M))\left[2(\nu h_t^{\beta}\mathcal{L})^2+ \frac{8\mu^2}{m_th_t^d}\|\omega_t\sqrt{\qtil_t}\|_2^2+ o((m_th_t^d)^{-1})\right]\\
&&+6\gamma_tM\left[ \nu \mathcal{L}h_t^{\beta} + \frac{\mu}{\sqrt{m_t}h_t^{d/2}}\|\omega_t\sqrt{\qtil_t}\|_1 + \frac{1}{\sqrt{m_t}}\|\omega_t\sqrt{\qtil_t}\|_2+o((m_th_t^d)^{-\frac{1}{2}})\right] \\ 
&& + o\bigg(2(\nu h_t^{\beta}\mathcal{L})^2+ \frac{8\mu^2}{m_th_t^d}\|\omega_t\sqrt{\qtil_t}\|_2^2 + \gamma_t [\nu \mathcal{L}h_t^{\beta} + \frac{\mu}{\sqrt{m_t}h_t^{d/2}}\|\omega_t\sqrt{\qtil_t}\|_1 + \frac{1}{\sqrt{m_t}}\|\omega_t\sqrt{\qtil_t}\|_2]\bigg)
\end{eqnarray*}
Under the Assumption C, we already proved that $|\omega_t|_\infty\leq\exp(2\gamma_t M)$, hence, $\|\omega_t\sqrt{\qtil_t}\|_2^2\leq \exp(4\gamma_t M)$. Without loss of generality, we can assume $\int \sqrt{\qtil_t(\theta)}d\theta\leq O(1)$ and $\gamma_tM\leq O(1)$  for all $t$, then we can simply write $\|\omega_t\sqrt{\qtil_t}\|_1\leq O(1)$ and $\|\omega_t\sqrt{\qtil_t}\|_2^2\leq O(1)$.
When $h_t=O(1)m_t^{-1/(d+2\beta)}$, the above result can be simplified as 
\begin{eqnarray*}
\EE_{\theta}  [\epsilon_t|x_{[t-1]},\theta_{[t-1]}]\leq& O(1)(\mu^2+\nu^2\mathcal{L}^2 )\mu^2\Delta m_t^{-\frac{2\beta}{d+2\beta}}+O(1)M(\mu+\nu\mathcal{L})\gamma_t m_t^{-\frac{\beta}{d+2\beta}}
\end{eqnarray*}
Similarly, combining the results of Corollary \ref{cor:error2} and \ref{cor:error4}, we have with probability at least $1-2\delta$,
 \begin{eqnarray*}
&\epsilon_t\leq & (2{\Delta} + 2\Delta \mu^2\exp(4\gamma_t M))\left[2(\nu h_t^{\beta}\mathcal{L})^2+ \frac{8\mu^2}{m_th_t^d}\|\omega_t\sqrt{\qtil_t}\|_2^2+ o((m_th_t^d)^{-1})+\frac{16B_1B_2\mu^2}{m_t}\sqrt{\log(1/\delta)}\right]\\
&&+6\gamma_tM\left[ \nu \mathcal{L}h^{\beta} + \frac{\mu}{\sqrt{m_t}h_t^{d/2}}\|\omega_t\sqrt{\qtil_t}\|_1 + \frac{1}{\sqrt{m_t}}\|\omega_t\sqrt{\qtil_t}\|_2+\frac{1}{\sqrt{m_t}}\sqrt{8B_1B_2\log(1/\delta)}+o((m_th_t^d)^{-\frac{1}{2}})\right]\\
&& + o\bigg(2(\nu h_t^{\beta}\mathcal{L})^2+ \frac{8\mu^2}{m_th_t^d}\|\omega_t\sqrt{\qtil_t}\|_2^2 + \gamma_t [\nu \mathcal{L}h_t^{\beta} + \frac{\mu}{\sqrt{m_t}h_t^{d/2}}\|\omega_t\sqrt{\qtil_t}\|_1 + \frac{1}{\sqrt{m_t}}\|\omega_t\sqrt{\qtil_t}\|_2]\bigg)
\end{eqnarray*}
which leads to the lemma. \\

\hfill$\blacksquare$

Our main Theorem \ref{thm:final} follows immediately by applying the results in the above lemma to Theorem \ref{thm:main1}.

\emph{Proof of Theorem \ref{thm:final}. } 
We first notice that 
\begin{eqnarray*}
\EE[KL(q^* || \qtil_T)] &=& \EE\left[\int q^*\log\frac{q^*}{\qtil_T}d\theta\right] = \EE\left[\int q^*\log\frac{q^*}{\qhat_T}d\theta + \int q^*\log\frac{\qhat_T}{\qtil_T}d\theta\right]\\
&=& \EE[KL(q^* || \qhat_T)] + \EE\left[\int q^*\log\frac{\qhat_T}{\qtil_T}d\theta\right].
\end{eqnarray*}

For the second term,
\begin{eqnarray*}
\EE\left[\int q^*\log\frac{\qhat_T}{\qtil_T}d\theta\right]&=&\EE\left[\langle q^*,  \log\frac{\frac{1}{1-\zeta_T}\qtil_T}{\qtil_T}\rangle\right] = \EE\left[\langle q^*, -\log(1-\zeta_T)\right]\\
&=&\EE[-\log(1-\zeta_T)]\le \zeta_T + o(\zeta_T)\le \EE\|\qtil_T - \qtil_{T-1}^+\|_1 + o(\EE\|\qtil_T - \qtil_{T-1}^+\|_1)
\end{eqnarray*}
By Theorem \ref{thm:error1} and setting $h_t=O(1)m_t^{-1/(d+2\beta)}$, we achieve the error bound 
$$
\EE\left[\int q^*\log\frac{\qhat_T}{\qtil_T}d\theta\right]\le \Ccal_2m_t^{-\frac{\beta}{d+2\beta}},
$$
where $\Ccal_2 := O(1)M(\mu+\nu\mathcal{L})$.

When setting $\gamma_t =\min\{\frac{2}{t+1}, \frac{\Delta}{Mm_t^{\beta/(d+2\beta)}}\}$
invoking the above lemma, we have
$$\EE_{\theta}  [\epsilon_t|x_{[t-1]},\theta_{[t-1]}]\leq \mathcal{C}_1m_t^{-2\beta/(d+2\beta)},$$
where $\mathcal{C}_1:=O(1)(\mu+\nu\mathcal{L})^2\mu^2\Delta$. 
Expanding the result from Theorem \ref{thm:main1}, it follows that 
  \begin{eqnarray*}
  \EE_{x,\theta}[KL(q^* ||  \qhat_{t+1})] \leqslant (1-\gamma_t) \EE_{x,\theta}[KL(q^* ||  \qhat_t)]+\mathcal{C}_1m_t^{-2\beta/(d+2\beta)} +\frac{\gamma_t^2}{2}M^2
  \end{eqnarray*} 
The above recursion leads to the convergence result for the second term, 
\begin{eqnarray*}
\EE[KL(q^* || \qhat_T)] &\leq& \frac{2\max\left\{D_1, M^2\right\}}{T} + \mathcal{C}_1\frac{\sum_{t=1}^Tt^2 m_t^{-\frac{2\beta}{d+2\beta}}}{T^2}.
\end{eqnarray*} 

Combine these two results, we achieve the desired result
\begin{eqnarray*}
\EE[KL(q^* || \qtil_T)] &\leq& \frac{2\max\left\{D_1, M^2\right\}}{T} + \mathcal{C}_1\frac{\sum_{t=1}^Tt^2 m_t^{-\frac{2\beta}{d+2\beta}}}{T^2} + \Ccal_2m_t^{-\frac{\beta}{d+2\beta}}.
\end{eqnarray*} 

\hfill$\blacksquare$

{\bf Remark.} The convergence in terms of $KL$-divergence is measuring the entire density and much more stringent compared to integral approximation. For the last iterate, an overall $O(\frac{1}{T})$ convergence rate can be achieved when $m_t=O(t^{2+d/\beta})$. Similar to Lemma~\ref{lemma:exact_prox_integral}, with Pinsker's inequality, we could easily obtain the the rate of convergence in terms of integral approximation from Theorem~\ref{thm:final}. After $T$ steps, in general cases, the PMD algorithm converges in terms of integral approximation in rate $O(1/\sqrt{T})$ by choosing $O(1/t)$-decaying stepsizes and $O(t^{2+\frac{d}{2\beta}})$-growing samples.

\section{Derivation Details for Sparse Gaussian Processes and Latent Dirichlet Allocation}\label{appendix:GP_LDA}

We apply the Particle Mirror Descent algorithm to sparse Gaussian processes and latent Dirichlet allocation. For these two models, we decompose the latent variables and incorporate the structure of posterior into the algorithm. The derivation details are presented below.

\subsection{Sparse Gaussian Processes}

Given data $X = \{x_i\}_{i=1}^n,\quad x_i\in \RR^{d \times 1}$ and $y = \{y_i\}_{i=1}^n$. The sparse GP introduce a set of inducing variables, $Z = \{z_i\}_{i=1}^m,\quad z_i\in \RR^{d \times 1}$ and the model is specified as
\begin{eqnarray*}
p(y_n |\ub, Z) &=& \Ncal(y_n| K_{nm}K^{-1}_{mm}\ub, \tilde K)\\
p(\ub|Z) &=& \Ncal(\ub|\mathbf{0}, K_{mm}).
\end{eqnarray*}
where $K_{mm} = [k(z_i, z_j)]_{i, j = 1,\ldots, m}$, $K_{nm} = [k(x_i, z_j)]_{i=1,\ldots, n; j = 1,\ldots, m}$. For different $\tilde K$, there are different sparse approximations for GPs. Please refer~\cite{QuiRas05} for details. We test algorithms on the sparse GP model with $\tilde K = \beta^{-1}I$. We modify the stochastic variational inference for Gaussian processes~\cite{HenFusLaw13} for this model. We also apply our algorithm on the same model. However, it should be noticed that our algorithm could be easily extended to other sparse approximations~\cite{QuiRas05}.

We treat the inducing variables as the latent variables with uniform prior in sparse Gaussian processes. Then, the posterior of $Z, \ub$ could be thought as the solution to the optimization problem
\begin{eqnarray}\label{eq:sparse_gp}
\min_{q(Z, \ub)} \int q(Z, \ub)\log\frac{q(Z, \ub)}{p(Z)p(\ub)}\ub dZ - \sum_{i=1}^n \int q(Z, \ub)\log p (y_i | x_i, \ub, Z)d\ub dZ
\end{eqnarray}

The stochastic gradient of Eq.(\ref{eq:sparse_gp}) w.r.t. $q(Z,\ub)$ will be
$$
g(q(Z, \ub)) = \frac{1}{n}\log{q(Z, \ub)} - \frac{1}{n}\log{p(Z)p(\ub)} - \log p (y_i | x_i, \ub, Z)
$$
and therefore, the prox-mapping in $t$-step is
\begin{eqnarray*}
\min_{q(Z, \ub)} \int q(Z, \ub)\log\frac{q(Z, \ub)}{q_t(Z, \ub)^{1-\gamma_t/n}p(Z, \ub)^{\gamma_t/n}}\ub dZ - \gamma_t\int q(Z, \ub)\log p (y_i | x_i, \ub, Z)d\ub dZ
\end{eqnarray*}
which could be re-written as 
\begin{eqnarray*}
&&\min_{q(Z)q(\ub|Z)} \int q(Z)\bigg\{\log\frac{q(Z)}{q_t(Z)^{1-\gamma_t/n}p(Z)^{\gamma_t/n}} \\
&+& \underbrace{\int q(\ub|Z)\bigg[\log\frac{q(\ub|Z)}{q_t(\ub|Z)^{1-\gamma_t/n}p(\ub|Z)^{\gamma_t/n}} -\gamma_t\log p(y_i | x_i, \ub, Z) \bigg]d\ub}_{L(q(\ub|Z))}\bigg\}dZ
\end{eqnarray*}
We update $q_{t+1}(\ub|Z)$ to be the optimal of $L(q(\ub|Z))$ as
\begin{eqnarray*}
q_{t+1}(\ub|Z)&\propto& q_{t}(\ub|Z)^{1-\gamma_t/n}p(\ub|Z)^{\gamma_t/n}p(y_i|x_i, \ub, Z)^{\gamma_t}\\
&=& \Ncal(\ub | m_t, \delta_t^{-1})^{1-\gamma_t/n}\Ncal(\ub | \mathbf{0}, K_{mm})^{\gamma_t/n}\Ncal(y_i | K_{im}K_{mm}^{-1}\ub, \Gamma)^{\gamma_t}\\
&=&\Ncal(\ub|m_{t+1}, \delta_{t+1}^{-1})
\end{eqnarray*}
where $\Gamma = diag(\tilde K_{ii} - Q_{ii}) + \beta^{-1}I,\,\, Q_{ii} = K_{im}K_{mm}^{-1}K_{mi}$, 
\begin{eqnarray*}
\delta_{t+1} = (1 - \gamma_t/n)\delta_t + \gamma_t/n K_{mm}^{-1} + \gamma_t K_{im}K_{mm}^{-1}\Gamma^{-1}K_{mm}^{-1}K_{mi}\\
m_{t+1} = \delta_{t+1}^{-1}\bigg((1 - \gamma_t/n)\delta_{t}^{-1}m_t + \gamma_t/n K_{mm}^{-1}m_0 + \gamma_t K_{mm}^{-1}K_{mi}\Gamma^{-1}y\bigg)
\end{eqnarray*}

Plug this into the $L(q(\ub|Z))$, we have
\begin{eqnarray*}
L(q(u|Z)) = \int q(\ub|Z)\bigg[\log\frac{q(\ub|Z)}{q_t(\ub|Z)^{1-\gamma_t/n}p(\ub|Z)^{\gamma_t/n}} -\gamma_t\log p(y_i | x_i, \ub, Z) \bigg]d_\ub= -\log\tilde p(y_i|x_i, Z)
\end{eqnarray*}
where 
\begin{eqnarray*}
\tilde p(y_i|x_i, Z) &=& \int q_t(\ub|Z)^{1-\gamma_t/n}p(\ub|Z)^{\gamma_t/n}p(y_i|x_i, \ub, Z)^{\gamma_t}d\ub\\
&=& \int \Ncal(\ub | m_t, \delta_t^{-1})^{1 - \gamma_t / n}\Ncal(\ub|0, K_{mm})^{\gamma_t /n}\Ncal(y_i | K_{im}K^{-1}_{mm}\ub, \Gamma)^{\gamma_t}d\ub\\
&=& \Ncal(y_i | K_{im}K^{-1}_{mm}c, \Sigma)
\end{eqnarray*}
where 
\begin{eqnarray*}
\bar\delta_{t+1} &=& (1  - \gamma_t/n)\delta_t + \gamma_t /n K_{mm}^{-1}\\
c &=& \bar \delta_{t+1}^{-1}\bigg((1 - \gamma_t /n)\delta_tm_t + \gamma_t / n K_{mm}^{-1}m_0\bigg)\\
\Sigma &=& K_{im}K^{-1}_{mm}\bar \delta^{-1}_{t+1}K^{-1}_{mm}K_{mi} + \frac{1}{\gamma_t}\Gamma
\end{eqnarray*}

Solve
\begin{eqnarray*}
\min_{q(Z)}\int q(Z)\log\frac{q(Z)}{q_t(Z)^{1-\gamma_t/n}p(Z)^{\gamma_t /n }}dZ -\int q(Z)\log\tilde p(y_i|x_i, Z)dZ
\end{eqnarray*}
will result the update rule for $q(Z)$,
\begin{eqnarray*}
q_{t+1}(Z) \propto q_t(Z)^{1-\gamma_t/n}p(Z)^{\gamma_t/n}\tilde p(y_i|x_i, Z)
\end{eqnarray*}
We approximate the $q(Z)$ with particles, i.e., $q(Z) = \sum_{j=1}^l w^j\delta(Z^j)$. The update rule for $w^j$ is
$$
w_{t+1}^j = \frac{w_t^j\exp(-\gamma_t/n\log(w_t^j) + \gamma_t/n \log p(Z^j) + \log\tilde p(y_i|x_i, Z^j))}{\sum_j^l w_t^j\exp(-\gamma_t/n\log(w_t^j) + \gamma_t/n \log p(Z^j) + \log\tilde p(y_i|x_i, Z^j))
}
$$

\subsection{Latent Dirichlet Allocations}

In LDA, the topics $\Phi\in \RR^{K \times W}$ are $K$ distributions on the words $W$ in the text corpora. The text corpora contains $D$ documents, the length of the $d$-th document is $N_d$. The document is modeled by a mixture of topics, with the mixing proportion $\theta_d\in \RR^{1 \times K}$. The words generating process for $X_d$ is following: first drawing a topic assignment $z_{dn}$, which is $1$-by-$K$ indicator vector, \iid from $\theta_d$ for word $x_{dn}$ which is $1$-by-$W$ indicator vector, and then drawing the word $x_{dn}$ from the corresponding topic $\Phi_{z_{dn}}$. We denote $z_d = \{z_{dn}\}_{n=1}^{N_d} \in \RR^{N_d \times K}$, $ x_d = \{x_{dn}\}_{n=1}^{N_d}\in \RR^{N_d \times W}$ and $X = \{x_d\}_{d=1}^D$,$Z = \{Z_d\}_{d=1}^D$ . Specifically, the joint probability is 
\begin{eqnarray}
p(x_d, z_d, \theta_d, \Phi) &=& p(x_d|z_d, \Phi) p(z_d|\theta_d) p(\theta_d) p(\Phi)\\ \nonumber
p(x_d|z_d, \Phi) &=& \prod_{n=1}^{N_d}\prod_{w=1}^{W}\prod_{k=1}^{K} \Phi_{kw}^{z_{dnk}x_{dnw}}\\
p(z_d|\theta_d) &=& \prod_{n=1}^{N_d}\prod_{k=1}^{K} \theta_{dk} ^{z_{dnk}} \nonumber
\end{eqnarray}
The $p(\Phi)$ and $p(\theta)$ are the priors for parameters, $p(\theta_d|\alpha) = \frac{\Gamma(K\alpha)}{\Gamma(\alpha)^K}\prod_{k}^K \theta_{dk}^{\alpha-1}$ and $p(\Phi|\beta_0) = \prod_{k}^K \frac{\Gamma(W\beta_0)}{\Gamma(\beta_0)^W} \prod_{w}^W \Phi_{wk}^{\beta_0-1}$, both are Dirichlet distributions.

We incorporate the special structure into the proposed algorithm. Instead of modeling the $p(\Phi)$ solely, we model the $Z = \{Z\}_{d=1}^D$ and $\Phi$ together as $q(Z, \Phi)$. Based on the model, given $Z$, the $q(\Phi|Z)$ will be Dirichlet distribution and could be obtained in closed-form.

The posterior of $Z, \Phi$ is the solution to
\begin{eqnarray*}
\min_{q(Z, \Phi)} \frac{1}{D}\int q(Z, \Phi)\log\frac{q(Z, \Phi)}{p(Z|\alpha)p(\Phi|\beta)}dZd\Phi - \frac{1}{D}\sum_{d=1}^D\int q(Z, \Phi)\log p(x_d|z_d,\Phi)dZd\Phi
\end{eqnarray*}

We approximate the finite summation by expectation, then the objective function becomes
\begin{eqnarray}\label{eq:LDA}
\min_{q(Z, \Phi)} \frac{1}{D}\int q(Z, \Phi)\log\frac{q(Z, \Phi)}{p(Z|\alpha)p(\Phi|\beta)}dZd\Phi - \EE_x\bigg[\int q(Z, \Phi)\log p(x_d|z_d,\Phi)dZd\Phi\bigg]
\end{eqnarray}

We approximate the $q(Z) \approx \sum_{i=1}^m w^i\delta(Z^i)$ by particles, and therefore, $q(Z, \Phi) \approx \sum_{i=1}^m w^i P(\Phi|Z^i)$ where $P(\Phi|Z^i)$ is the Dirichlet distribution as we discussed. It should be noticed that from the objective function, we do not need to instantiate the $z_d$ until we visit the $x_d$. By this property, we could first construct the particles $\{Z^i\}_{i=1}^m$ `conceptually' and assign the value to $\{z_d^i\}_{i=1}^m$ when we need it. The gradient of Eq.(\ref{eq:LDA}) w.r.t. $q(\Phi, Z)$ is 
$$
g(q(Z, \Phi)) = \frac{1}{D}\log{q(Z, \Phi)}- \frac{1}{D}\log{p(\Phi)p(Z)}- \EE_x[\log p(x_d|\Phi, z_d)]
$$

Then, the SGD prox-mapping is 
\begin{eqnarray*}
\min_{q(Z, \Phi)} \int q(Z, \Phi)\log\frac{q(Z, \Phi)}{q_t(Z, \Phi)} + \gamma_t\int q(Z, \Phi) \bigg[\log{q_t(Z, \Phi)}/D- \log{p(\Phi)p(Z)}/D - \log p(x_d|\Phi, z_d) \bigg]dZd\Phi
\end{eqnarray*}

We rearrange the prox-mapping,
\begin{eqnarray*}
\min_{q(Z)q(\Phi|Z)}&& \int q(Z)q(\Phi|Z)\log\frac{q(Z)q(\Phi|Z)}{q_t(Z)^{1-\gamma_t/D}q_t(\Phi|Z)^{1-\gamma_t/D}} \\
&-& \gamma_t\int q(Z)q(\Phi|Z) \bigg[\log{p(\Phi)p(Z)}/D + \log p(x_d|\Phi, z_d) \bigg]dZd\Phi
\end{eqnarray*}

\begin{eqnarray*}
\min_{q(Z)q(\Phi|Z)} &&\int q(Z)\bigg\{\log\frac{q(Z)}{q_t(Z)^{1-\gamma_t/D}p(Z)^{\gamma_t/D}}\\
&+& \underbrace{\int q(\Phi|Z)\bigg[\log\frac{q(\Phi|Z)}{q_t(\Phi|Z)^{1-\gamma_t/D}p(\Phi)^{\gamma_t/D}} - \gamma_t \log p(x_d|\Phi, z_d) \bigg]d\Phi}_{L(q(\Phi|Z))} \bigg\}dZ
\end{eqnarray*}

The stochastic functional gradient update for $q(\Phi|Z^i)$ is 
\begin{eqnarray*}
q_{t+1}(\Phi|Z^i)\propto q_{t}(\Phi|Z^i)^{1-\gamma_t/D}p(\Phi)^{\gamma_t/D} p(x_d|\Phi, z_d)^{\gamma_t}
\end{eqnarray*}
Let $q_{t}(\Phi|Z^i) = \Dcal ir(\beta_t^i)$, then, the $q_{t+1}(\Phi|Z^i)$ is also Dirichlet distribution
\begin{eqnarray*}
q_{t+1}(\Phi|Z^i)\propto \Dcal ir(\beta_t^i)^{1-\tilde \gamma_t}\Dcal ir(\beta_0)^{\tilde \gamma_t} \bigg(\prod_k\prod_w \Phi_{kw}^{\sum_n^{N_d}\delta(z_{dnk}=1, x_{dnw}=1)}\bigg)^{D\tilde \gamma_t} = \Dcal ir(\beta_{t+1}^i)
\end{eqnarray*}
where $\tilde \gamma_t = \gamma_t / D$ and 
$$
[\beta_{t+1}^i]_{kw} = (1-\tilde \gamma_t)[\beta_t^i]_{kw} + \tilde \gamma_t\beta_0 + D\tilde \gamma_t \sum_n^{N_d}\delta(z_{dnk}=1, x_{dnw}=1).
$$
In mini-batch setting, the updating will be 
$$
[\beta_{t+1}^i]_{kw} = (1-\tilde\gamma_t)[\beta_t^i]_{kw} + \tilde \gamma_t\beta_0 + \frac{D}{B}\tilde \gamma_t \sum_{d=1}^B\sum_n^{N_d}\delta(z_{dnk}=1, x_{dnw}=1).
$$

Plug the $q_{t+1}(\Phi|Z^i)$ into prox-mapping, we have 
\begin{eqnarray*}
L(q(\Phi|Z)) &=& \int q(\Phi|Z)\bigg[\log\frac{q(\Phi|Z)}{q_t(\Phi|Z)^{1-\tilde \gamma_t}p(\Phi)^{\tilde\gamma_t}} - D\tilde \gamma_t \log p(x_d|\Phi, z_d) \bigg]d\Phi\\
&=& -\log\tilde p(x_d|z_d, Z)
\end{eqnarray*}
where $\tilde p(x_d|z_d, Z^i) = \int_\Phi q_t(\Phi|Z^i)^{1-\tilde \gamma_t}p(\Phi)^{\tilde \gamma_t} p(x_d|\Phi, z_d)^{D\tilde \gamma_t}d\Phi$ which have closed-form
\begin{eqnarray*}
\tilde p(x_d|z_d, Z^i) &=& \int_\Phi q_t(\Phi|Z^i)^{1-\tilde \gamma_t}p(\Phi)^{\tilde \gamma_t} p(x_d|\Phi, z_d)^{D\tilde \gamma_t}d\Phi\\
&=& \int \Dcal ir(\beta_{t}^i)^{1-\tilde \gamma_t}\Dcal ir(\beta_{0}^i)^{\tilde \gamma_t}\bigg(\prod_k\prod_w \Phi_{kw}^{\sum_n^{N_d}\delta(z_{dnk}=1, x_{dnw}=1)}\bigg)^{D\tilde \gamma_t}d\Phi\\
&=&\prod_k\bigg(\frac{\Gamma(\sum_w^W [\beta_t^i]_{kw})}{\prod_w\Gamma([\beta_t^i]_{kw})}\bigg)^{1-\tilde \gamma_t}\bigg(\frac{\Gamma(W\beta_0)}{\Gamma(\beta_0)^W}\bigg)^{\tilde \gamma_t}\frac{\prod_w \Gamma([\beta_{t+1}^i]_{kw})}{\Gamma(\sum_w[\beta_{t+1}^i]_{kw})}
\end{eqnarray*}
and
\begin{eqnarray*}
\log \tilde p(x_d|z_d, Z^i) &\propto& \sum_k\bigg((1-\tilde \gamma_t)\log\Gamma(\sum_w^W [\beta_t^i]_{kw}) + \sum_w\log\Gamma([\beta_{t+1}^i]_{kw}) \\
&-& \log \Gamma(\sum_w[\beta_{t+1}^i]_{kw})- (1-\tilde \gamma_t)\sum_w\log \Gamma([\beta_t^i]_{kw})\bigg)
\end{eqnarray*}
Then, we could update $q_t(Z) = \sum_i^m w^i\delta(Z^i)$ by 
\begin{eqnarray*}
q_{t+1}(Z^i) \propto q_{t}(Z^i)\exp\bigg(-\frac{\gamma_t}{D}\log q_{t}(Z^i) + \frac{\gamma_t}{D}\log p(Z^i|\alpha) + \log \tilde p(x_d|z_d, Z^i)\bigg)
\end{eqnarray*}
If we set $\alpha = 1$, $p(Z^i)$ will be uniformly distributed which has no effect to the update. For general setting, to compute $\log p(Z^i|\alpha)$, we need prefix all the $\{z_d^i\}_{d = 1}^D$. However, when $D$ is huge, the second term will be small and we could ignore it approximately.

Till now, we almost complete the algorithm except the how to assign $z_d$ when we visit $x_d$. We could assign the $z_d$ randomly. However, considering the requirement for the $z_d^i$ assignment that the $q(z^i_d|Z^i_{\setminus d})>0$, which means the assignment should be consistent, an better way is using the average or sampling proportional to $\int p(x_d|\Phi, z_d) q_t(\Phi |Z^i)p(z_d|Z^i_{1\ldots, d-1}) d\Phi$ where $p(z_d|Z^i_{1\ldots, d-1})  = \int p(z_d|\alpha)p(\alpha|Z^i_{1\ldots, d-1})d\alpha$, or $\int p(x_d|\Phi, z_d) q_t(\Phi |Z^i)p(z_d|\alpha) d\Phi$.

\section{More Related Work}\label{appendix:more_related_work}

Besides the most related two inference algorithms we discussed in Section~(\ref{sec:related_work}), \ie, stochastic variational inference~\cite{HofBleWanPai13} and static sequential Monte Carlo~\cite{Chopin02,BalMad06}, there are several other inference algorithms connect to the PMD from algorithm, stochastic approximation, or representation aspects, respectively.

From algorithmic aspect, our algorithm scheme shares some similarities to annealed importance sampling~(AIS)~\cite{Neal01} in the sense that both algorithms are sampling from a series of densities and reweighting the samples to approximate the target distribution. The most important difference is the way to construct the intermediate densities. In AIS, the density at each iteration is a weighted product of the joint distribution of \emph{all the data} and a \emph{fixed} proposal distribution, while the densities in PMD are a weighted product of \emph{previous step solution} and the stochastic functional gradient on \emph{partial data}. Moreover, the choice of the temperature parameter (fractional power) in AIS is heuristic, while in our algorithm, we have a principle way to select the stepsize with quantitative analysis. The difference in intermediate densities results the sampling step in these two algorithms is also different: the AIS might need MCMC to generate samples from the intermediate densities, while we only samples from a KDE which is more efficient. These differences make our method could handle large-scale dataset while AIS cannot.

Sequential Monte-Carlo sampler~\cite{DelDouJas06} provides a unified view of SMC in Bayesian inference by adopting different forward/backward kernels, including the variants proposed in~\cite{Chopin02,BalMad06} as special cases. There are subtle and important differences between the PMD and the SMC samplers. In the SMC samplers, the introduced finite forward/backward Markov kernels are used to construct a distribution over the auxiliary variables. To make the SMC samplers valid, it is required that the marginal distribution of the constructed density by integrating out the auxiliary variables must be the \emph{exact} posterior. However, there is no such requirement in PMD. In fact, the PMD algorithm only \emph{approaches} the posterior with controllable error by iterating the dataset \emph{many times}. Therefore, although the proposed PMD and the SMC sampler bare some similarities operationally, they are essentially different algorithms.

Stochastic approximation becomes a popular trick in extending the classic Bayesian inference methods to large-scale datasets recently. Besides stochastic variational inference, which incorporates stochastic gradient descent into variational inference, the stochastic gradient Langevin dynamics~(SGLD)~\citet{WelTeh11}, and its derivatives~\cite{AhnKorWel12, ChenFoxGue14, DinFanBabCheetal14} combine ideas from stochastic optimization and Hamiltonian Monte Carlo sampling. Although both PMD and the SGLD use the stochastic gradient information to guide next step sampling, the optimization variable in these two algorithms are different which results the completely different updates and properties. In PMD, we directly update the density utilizing \emph{functional gradient in density space}, while the SGLD perturbs the \emph{stochastic gradient in parameter space}. Because of the difference in optimization variables, the mechanism of these algorithms are totally different. The SGLD generates a trajectory of \emph{dependent} samples whose stationary distribution approximates the posterior, the PMD keeps an approximation of the posterior represented by \emph{independent} particles or their weighted kernel density estimator. In fact, their different properties we discussed in Table~\ref{table:methods_survey} solely due to this essential difference.

A number of generalized variational inference approaches are proposed trying to relax the constraints on the density space with flexible densities. Nonparametric density family is a natural choice\footnote{Although~\cite{SudIhlFreWil03, GerHofBle12} named their methods as ``nonparametric'' belief propagation and ``nonparametric'' variational inference, they indeed use mixture of Gaussians, which is still a parametric model.}. \citep{SonGreBicLowGue11} and \cite{IhlMcA09, LieTehDou15} extend the belief propagation algorithm with nonparametric models by kernel embedding and particle approximation, respectively. The most important difference between these algorithms and PMD is that they originate from different sources and are designed for different settings. Both the kernel BP~\citet{SonGreBicLowGue11} and particle BP~\citet{IhlMcA09, LieTehDou15} are based on belief propagation optimizing \emph{local objective} and designed for the problem with \emph{one sample} $X$ in which observations are highly dependent, while the PMD is optimizing the \emph{global objective}, therefore, more similar to mean-field inference, for the inference problems with \emph{many} \iid\, samples. 

After the comprehensive review about the similarities and differences between PMD and the existing related approximate Bayesian inference methods from algorithm, stochastic approximation and representation perspectives, we can see the position of the proposed PMD clearly. The PMD connects variation inference and Monte Carlo approximation, which seem two orthogonal paradigms in approximate Bayesian inference, and achieves a balance in trade-off between efficiency, flexibility and provability. 

\section{Experiments Details}\label{appendix:exp}

\subsection{Mixture Models}
\begin{figure*}[!ht]
\centering
  \begin{tabular}{ccc}
  \includegraphics[width=0.32\textwidth]{figures/true_posterior-crop} &
  \includegraphics[width=0.32\textwidth]{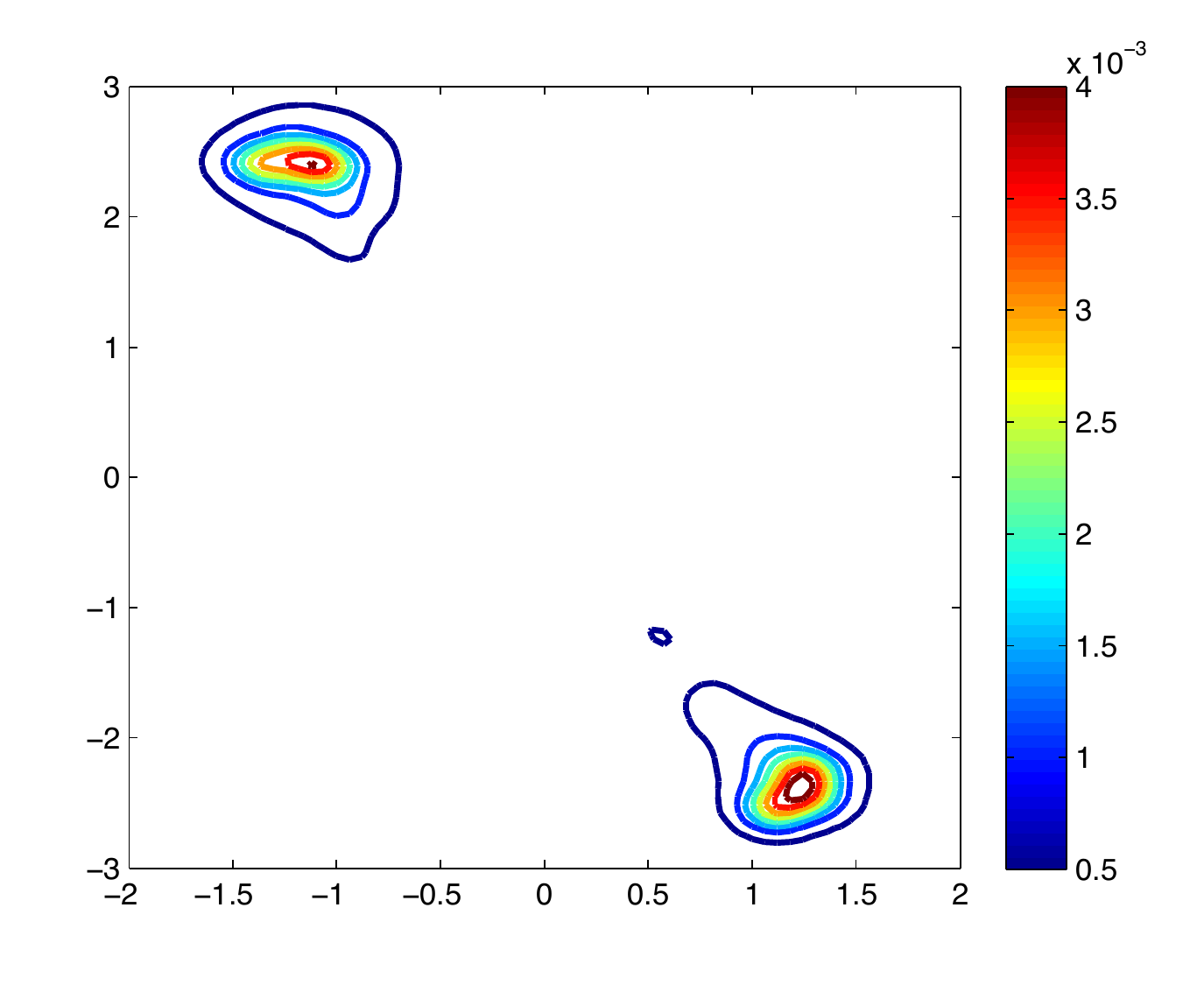} &
  \includegraphics[width=0.32\textwidth]{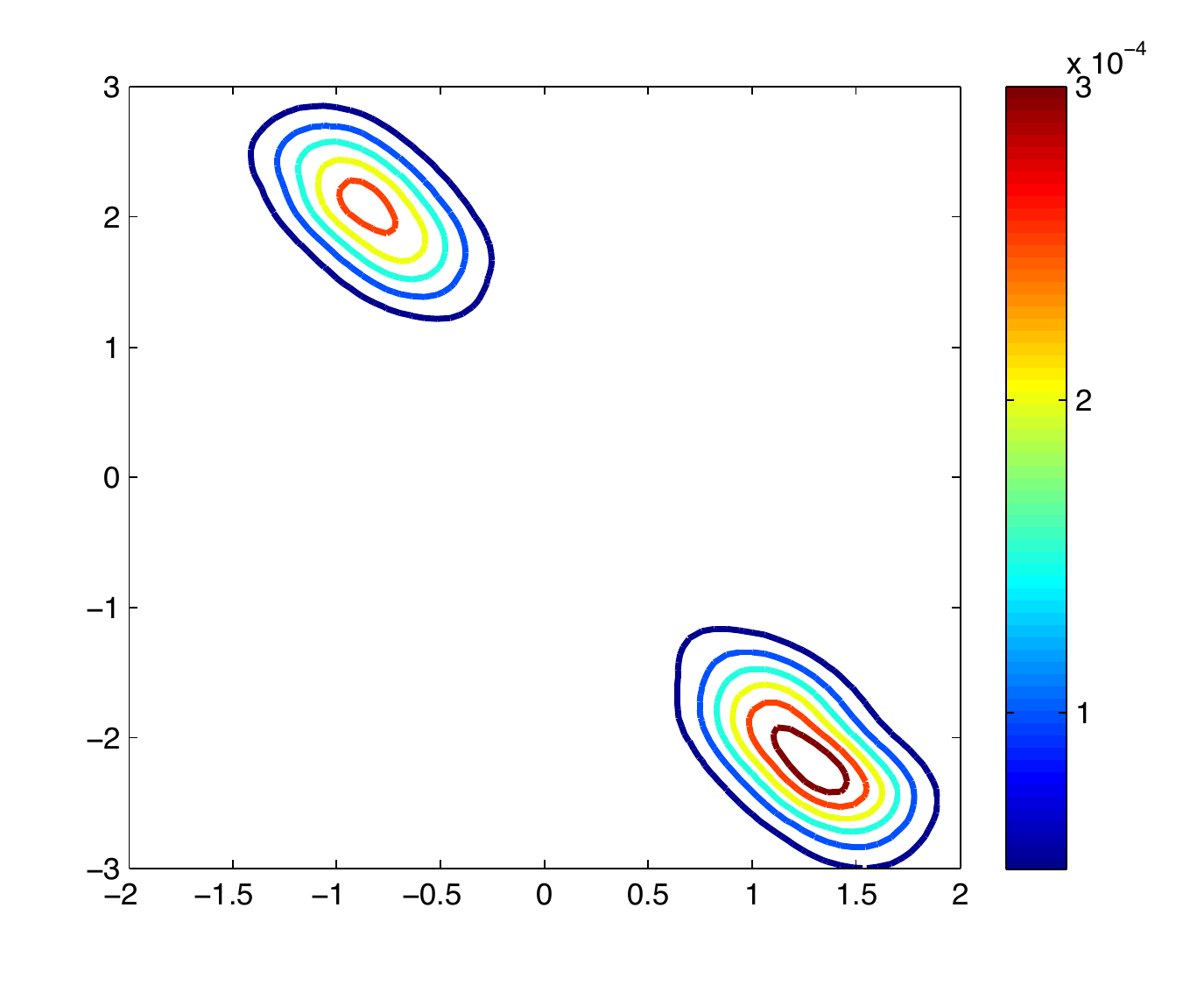} \\
  (1) True Posterior & (2) SGD NPV & (3) One-pass SMC
  \end{tabular}
  \begin{tabular}{ccc}
  \includegraphics[width=0.32\textwidth]{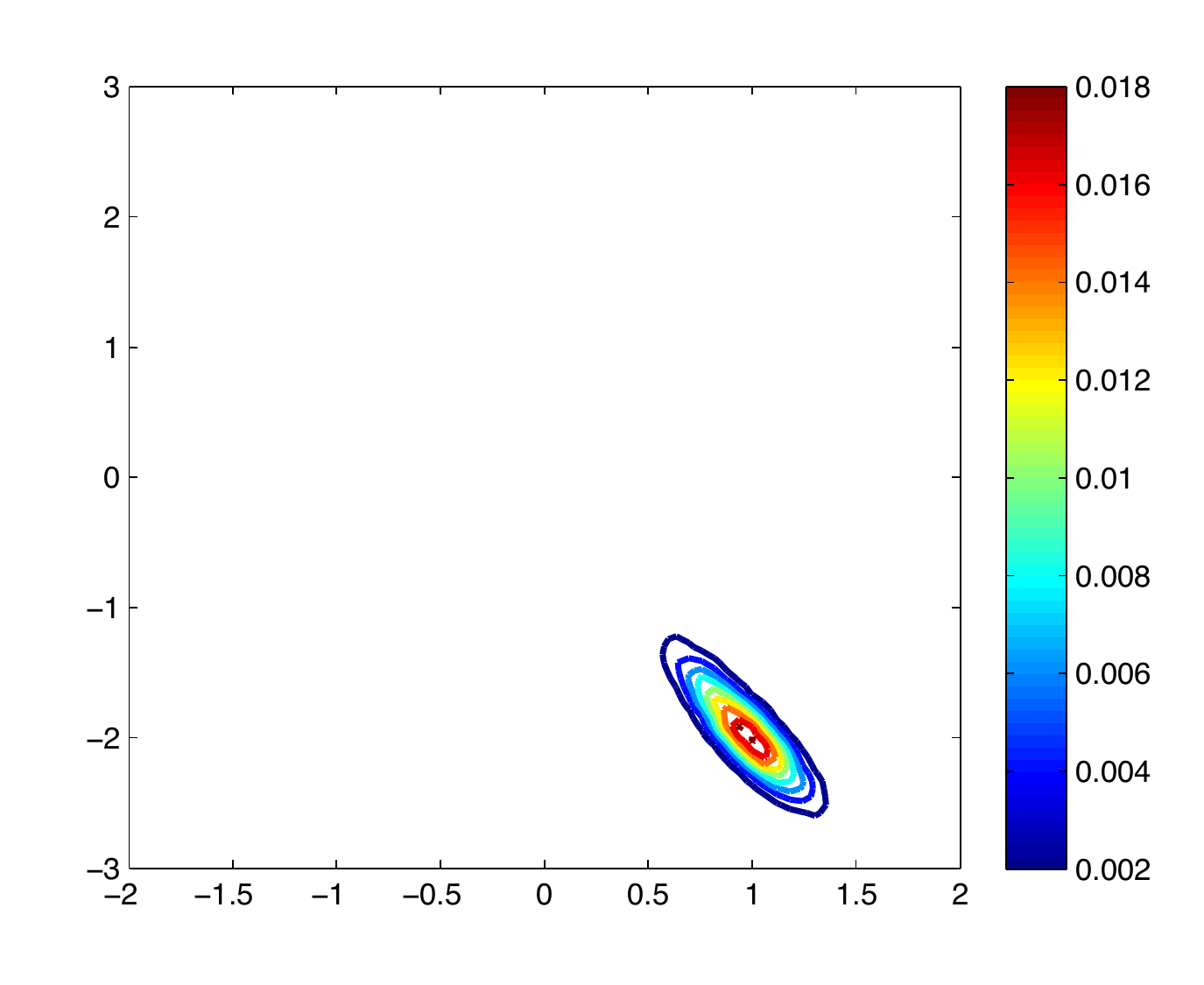} & 
  \includegraphics[width=0.32\textwidth]{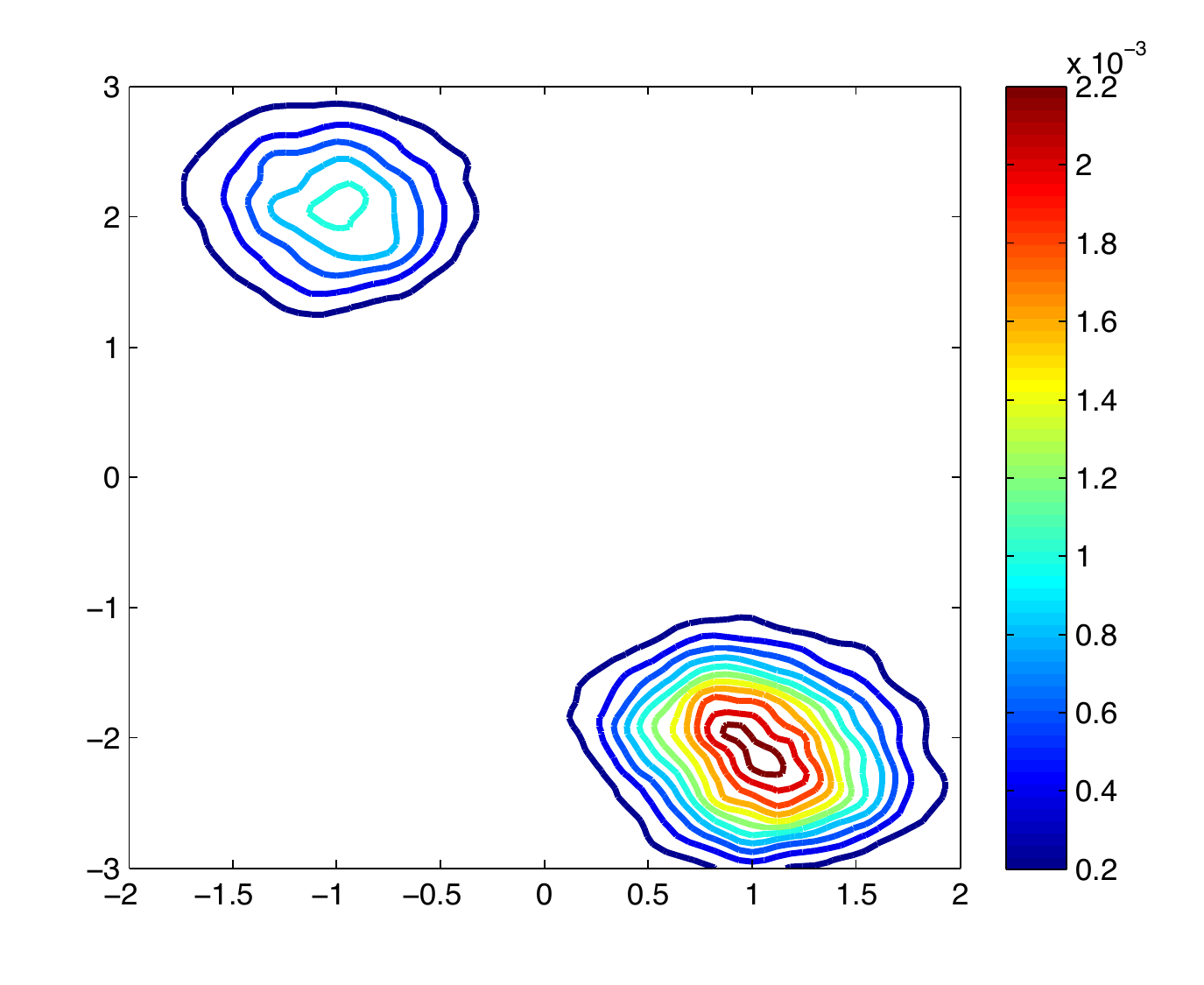}&
  \includegraphics[width=0.32\textwidth]{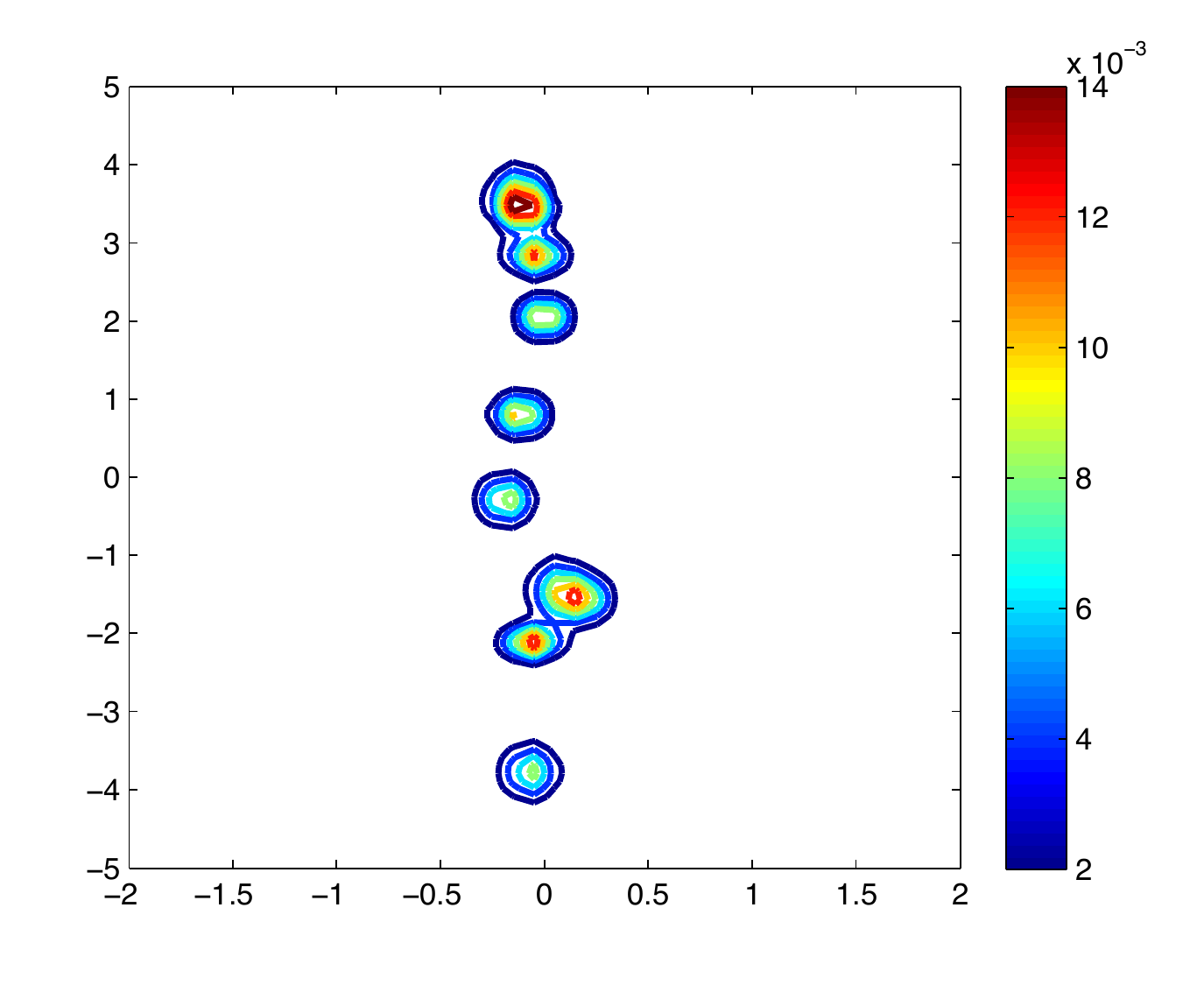}\\
  (4) Gibbs Sampling  & (5) SGD Langevin &(6) SVI
  \end{tabular}
  \caption{Visualization of posteriors of mixture model on synthetic dataset obtained by several inference methods. }
  \label{fig:syn_data_competitors}
\end{figure*}

We use the normalized Gaussian kernel in this experiment. For one-pass SMC, we use the suggested kernel bandwidth in~\cite{BalMad06}. For our method, since we increase the samples, the kernel bandwidth is shrunk in rate of $O(m^{-\frac{1}{2}})$ as the theorem suggested. The batch size for stochastic algorithms and one-pass SMC is set to be $10$. The total number of particles for the Monte Carlo based competitors, \ie, SMC, SGD Langevin, Gibbs sampling, and our method is $1500$ in total. We also keep $1500$ Gaussian components in SGD NPV. The burn-in period for Gibbs sampling and stochastic Langevin dynamics are $50$ and $1000$ respectively. 

The visualization of $10$ runs average posteriors obtained by the alternative methods are plotted in Figure~\ref{fig:syn_data_competitors}. From these figures, we could have a direct understand about the behaviors for each competitors. The Gibbs sampling and stochastic gradient Langevin dynamics sampling stuck in one local mode in each run. Gibbs sampler could fit one of the contour quite well, better than the stochastic Langevin dynamics. It should be noticed that this is the average solution, the two contours in the result of stochastic gradient Langevin dynamics did not mean it finds both modes simultaneously. The one-pass sequential Monte Carlo and stochastic nonparametric variational inference are able to location multiple modes. However, their shapes are not as good as ours. Because of the multiple modes and the highly dependent variables in posterior, the stochastic variational inference fails to converge to the correct modes.

To compare these different kinds of algorithms in a fair way, we evaluate their performances using total variation and cross entropy of the solution against the true potential functions versus the number of observations visited. In order to evaluate the total variation and the cross entropy between the true posterior and the estimated one, we use both kernel density estimation and Gaussian estimation to approximate the posterior density and report the better one for Gibbs sampling and stochastic Langevin dynamics. The kernel bandwidth is set to be $0.1$ times the median of pairwise distances between data points (median trick).

In Figure~\ref{fig:syn_data}(3)(4), the one-pass SMC performs similar to our algorithm at beginning. However, it cannot utilize the dataset effectively, therefore, it stopped with high error. It should be noticed that the one-pass SMC starts with more particles while our algorithm only requires the same number of particles at final stage. The reason that Gibbs sampling and the stochastic gradient Langevin dynamics perform worse is that they stuck in one mode. It is reasonable that Gibbs sampling fits the single mode better than stochastic gradient Langevin dynamics since it generates one new sample by scanning the whole dataset. For the stochastic nonparametric variational inference, it could locate both modes, however, it optimizes a non-convex objective which makes its variance much larger than our algorithm. The stochastic variational inference fails because of the highly dependent variables and multimodality in posterior.

\subsection{Bayesian Logistic Regression}
The likelihood function is 
$$
p(y|x, w) = \frac{1}{1 + \exp(-yw^\top x)}
$$ 
with $w$ as the latent variables. We use Gaussian prior for $w$ with identity covariance matrix. 

We first reduce the dimension to $50$ by PCA. The batch size is set to be $100$ and the step size is set to be $\frac{1}{100 + \sqrt{t}}$. We stop the stochastic algorithms after they pass through the whole dataset $5$ times. The burn-in period for stochastic Langevin dynamic is set to be $1000$. We rerun the experiments $10$ times. 

Although the stochastic variant of nonparametric variational inference performs comparable to our algorithm with fewer components, its speed is bottleneck when applied to large-scale problems. The gain from using stochastic gradient is dragged down by using L-BFGS to optimize the second-order approximation of the evidence lower bound.

\subsection{Sparse Gaussian Processes}

\subsubsection{1D Synthetic Dataset}

\begin{figure*}[!ht]

  \hspace{-0.3in}
  \begin{tabular}{cccc}
  \includegraphics[width=0.24\textwidth, trim = 40 40 40 40]{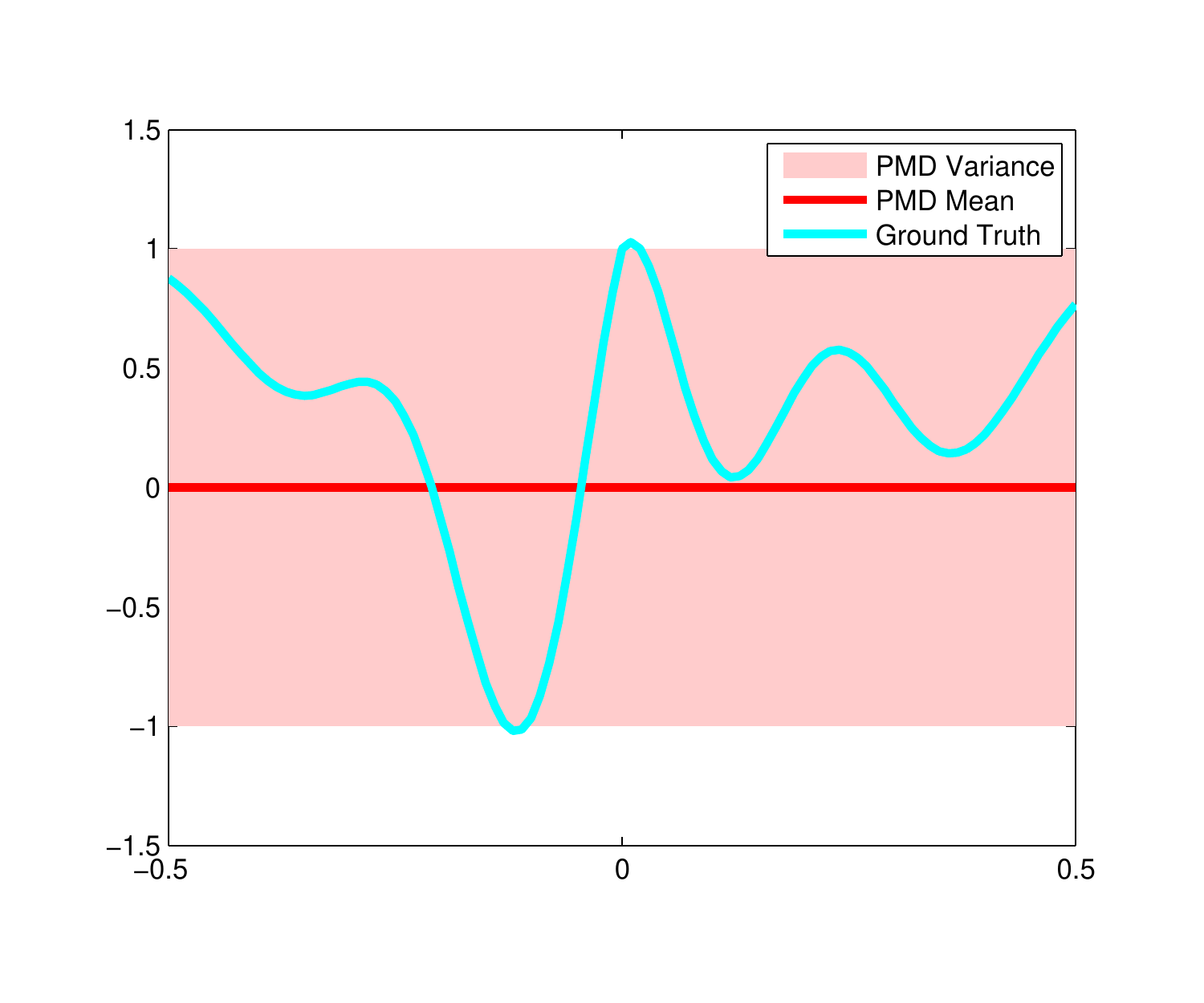} &
  \includegraphics[width=0.24\textwidth, trim = 40 40 40 40]{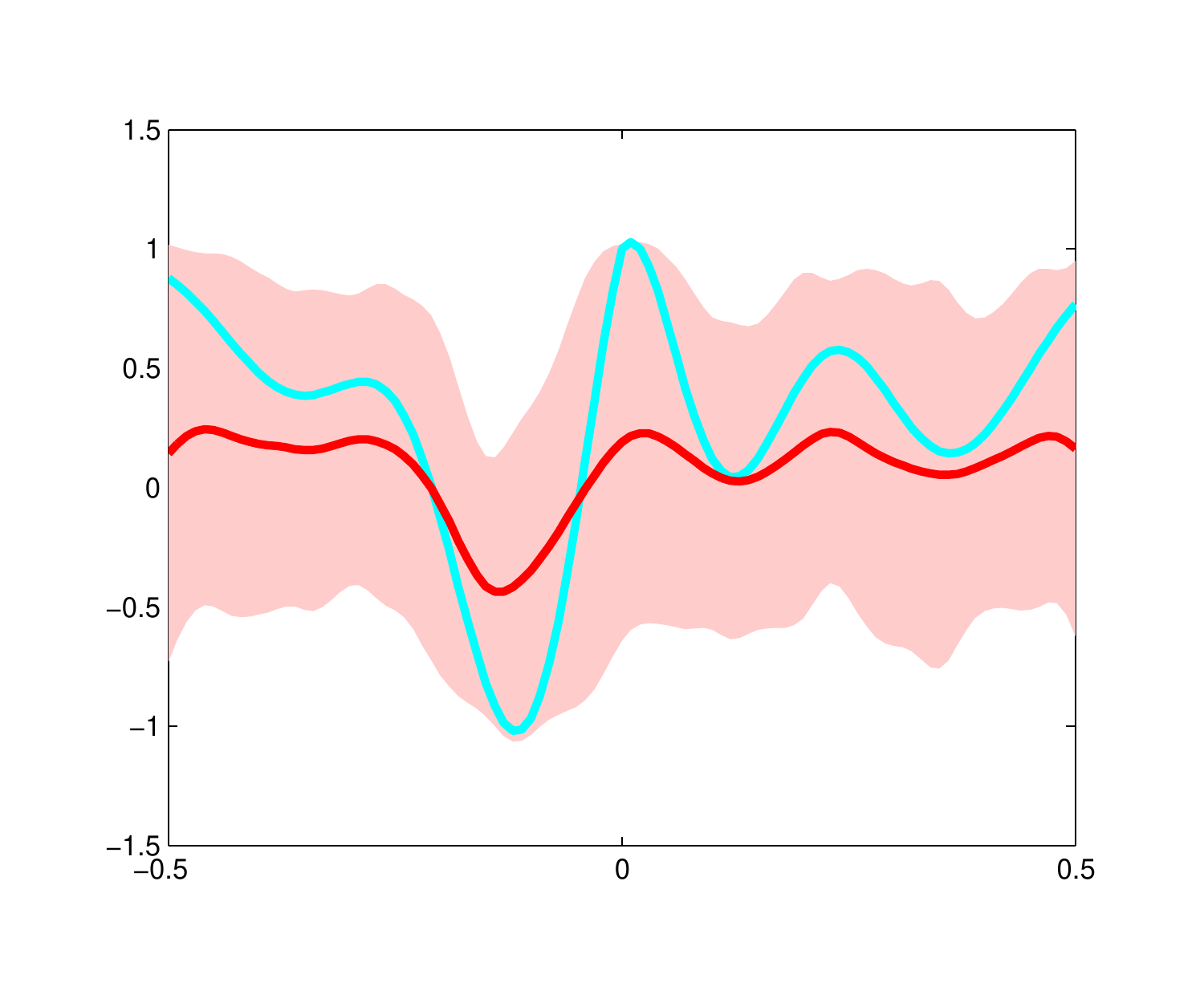} &
  \includegraphics[width=0.24\textwidth, trim = 40 40 40 40]{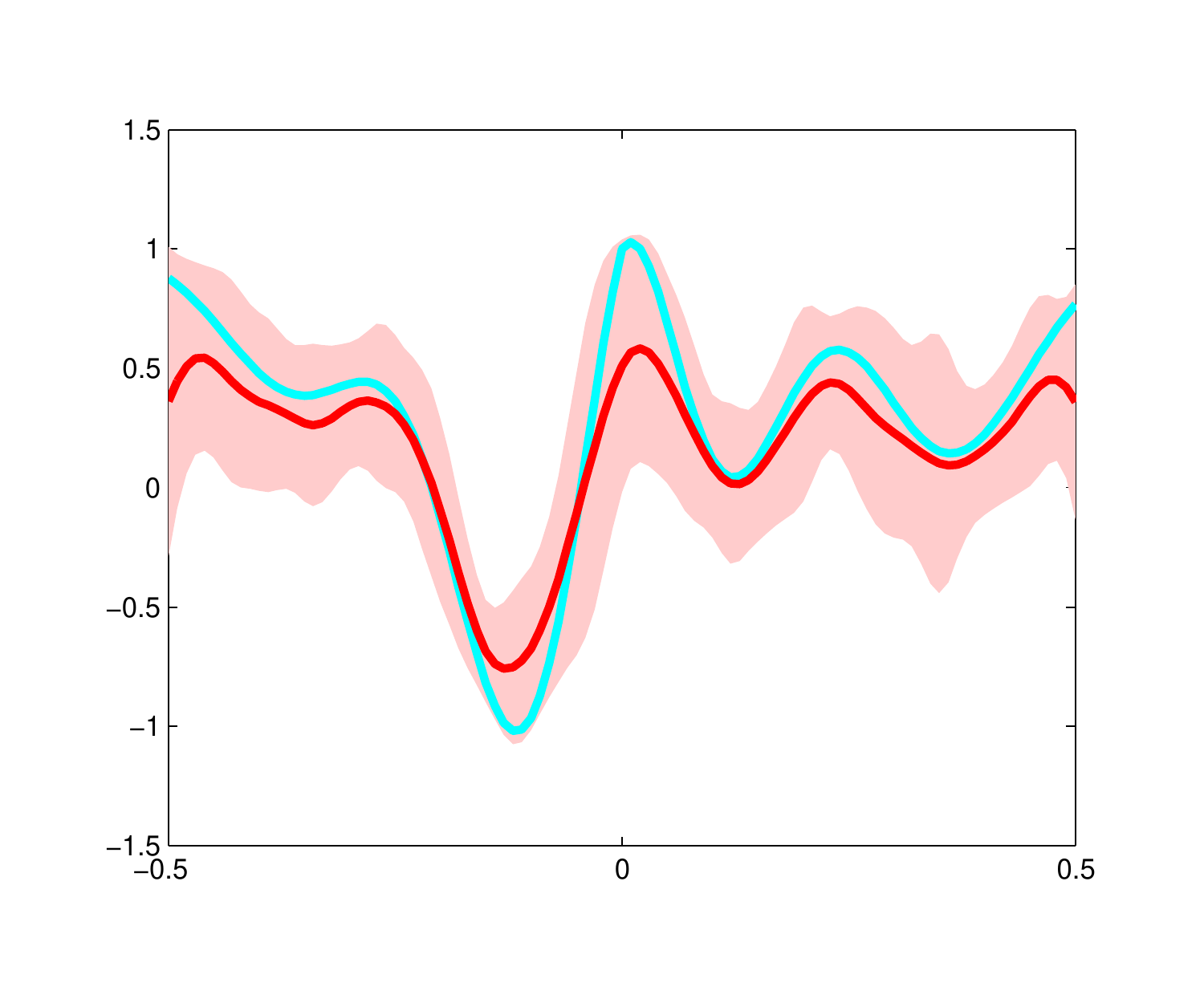} &
  \includegraphics[width=0.24\textwidth, trim = 40 40 40 40]{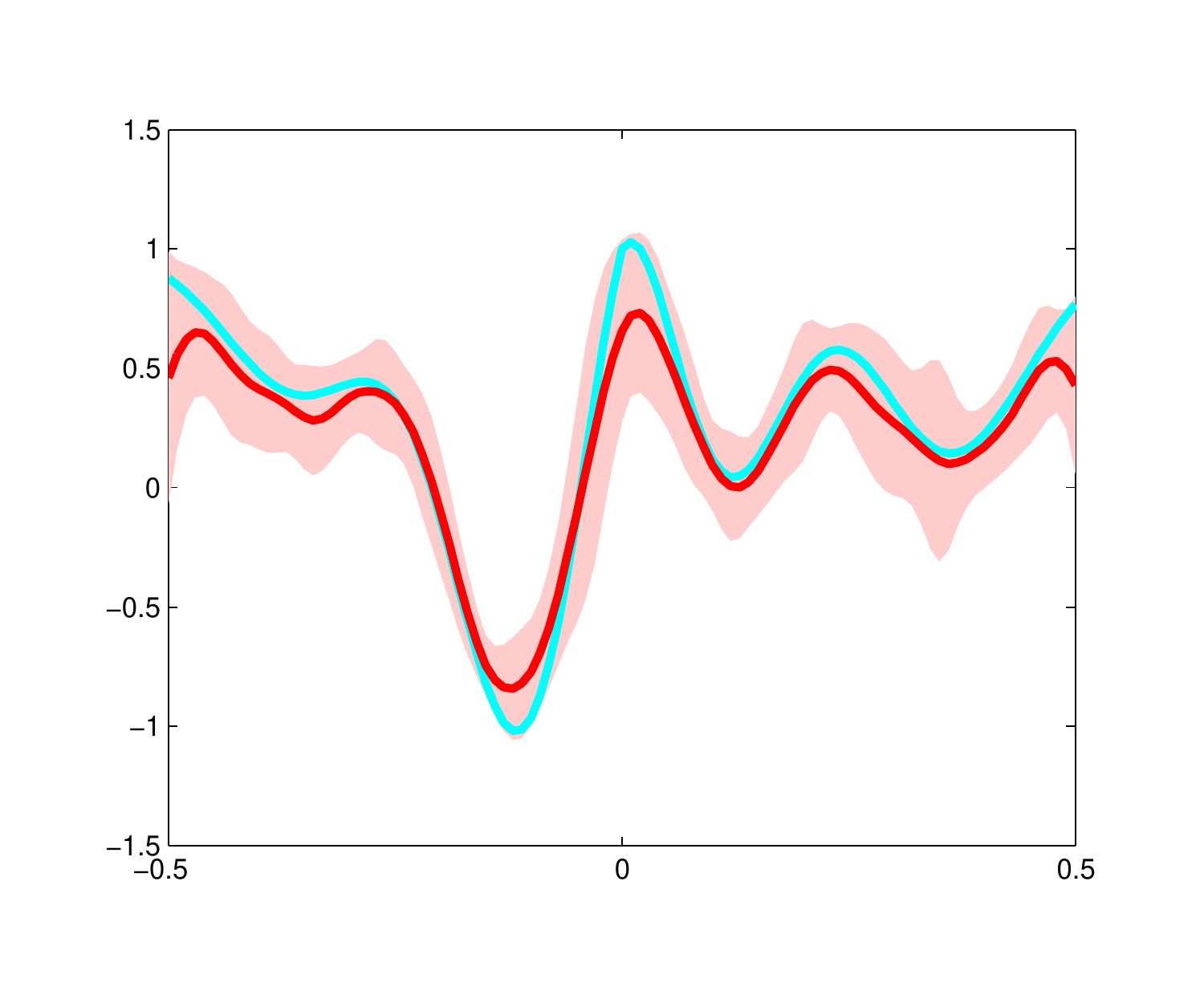} \vspace{-0.1in}\\
  (1) Iteration 0 & (2) Iteration 1 & (3) Iteration 5 &(4) Iteration 10
  \end{tabular}

  \hspace{-0.3in}
  \begin{tabular}{cccc}  
  \includegraphics[width=0.24\textwidth, trim = 40 40 40 40]{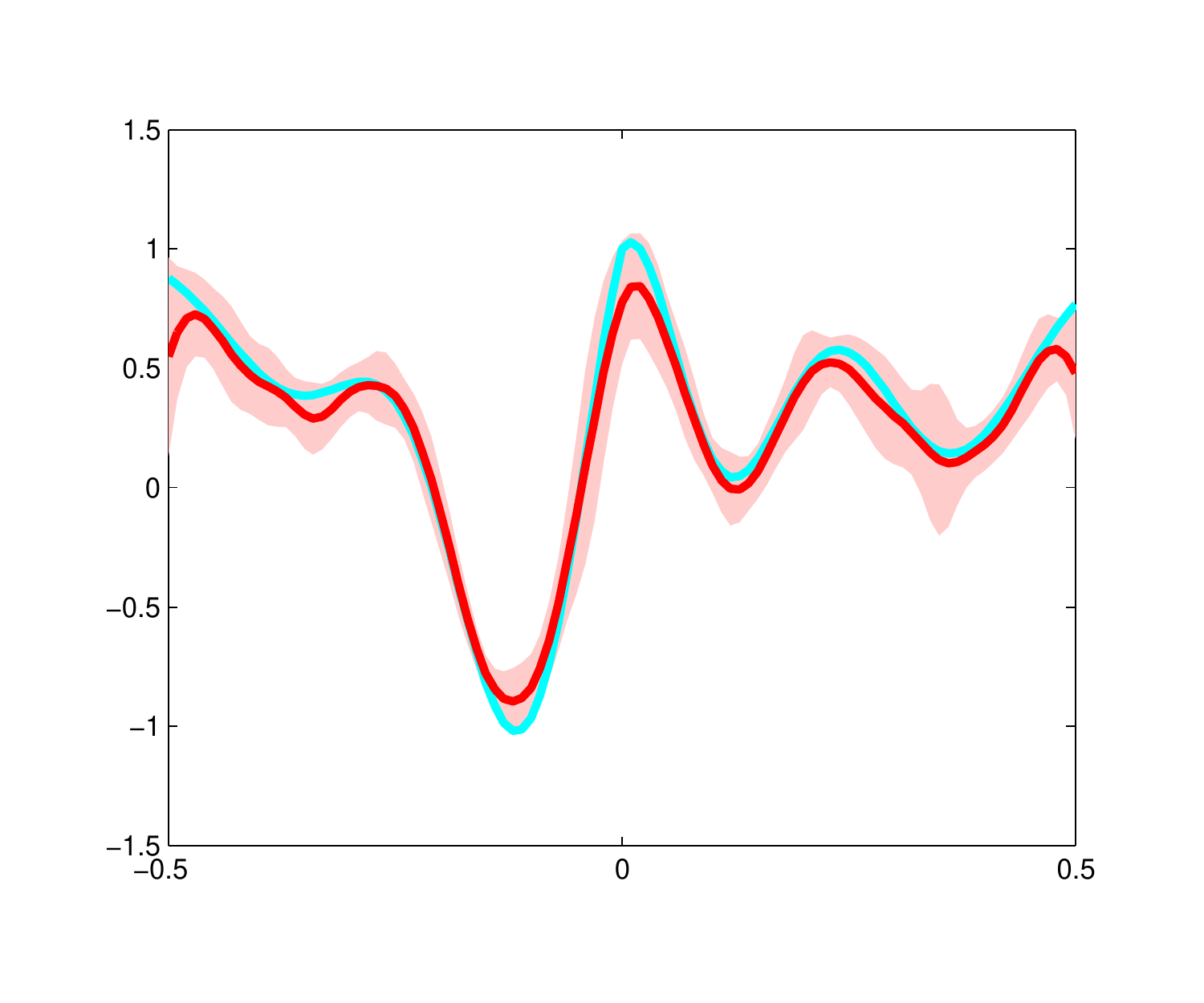} &
  \includegraphics[width=0.24\textwidth, trim = 40 40 40 40]{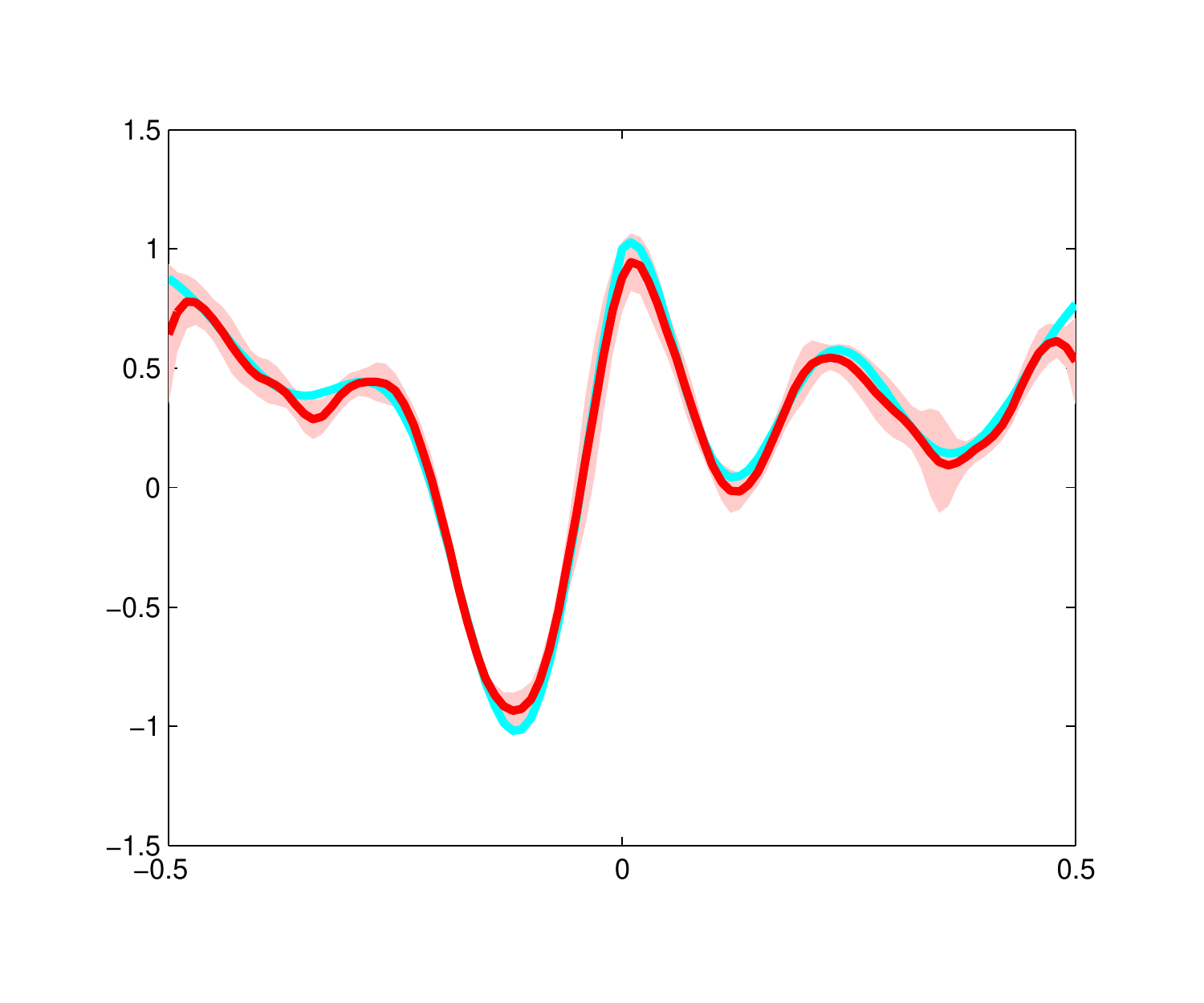} &
  \includegraphics[width=0.24\textwidth, trim = 40 40 40 40]{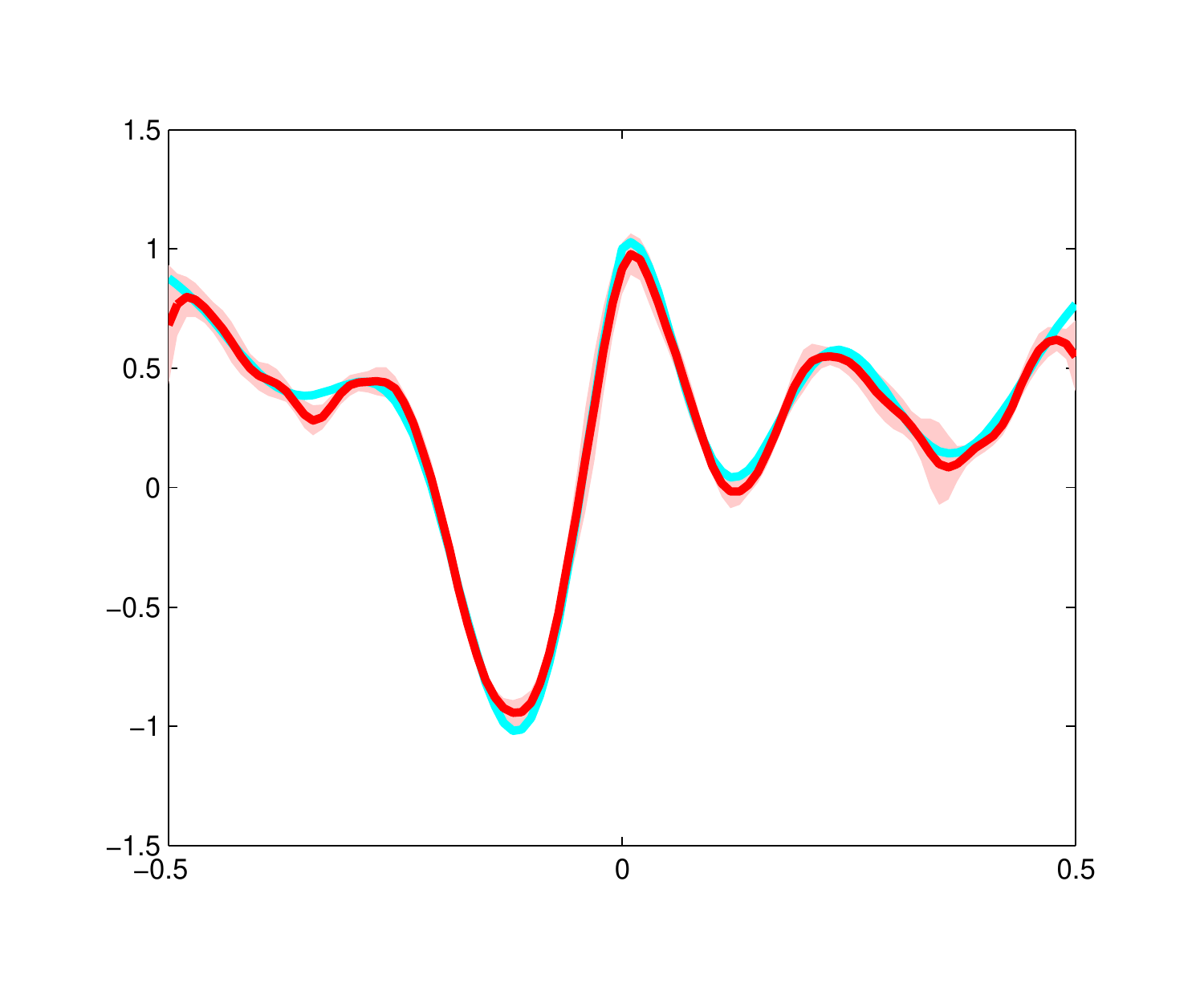} &
  \includegraphics[width=0.24\textwidth, trim = 40 40 40 40]{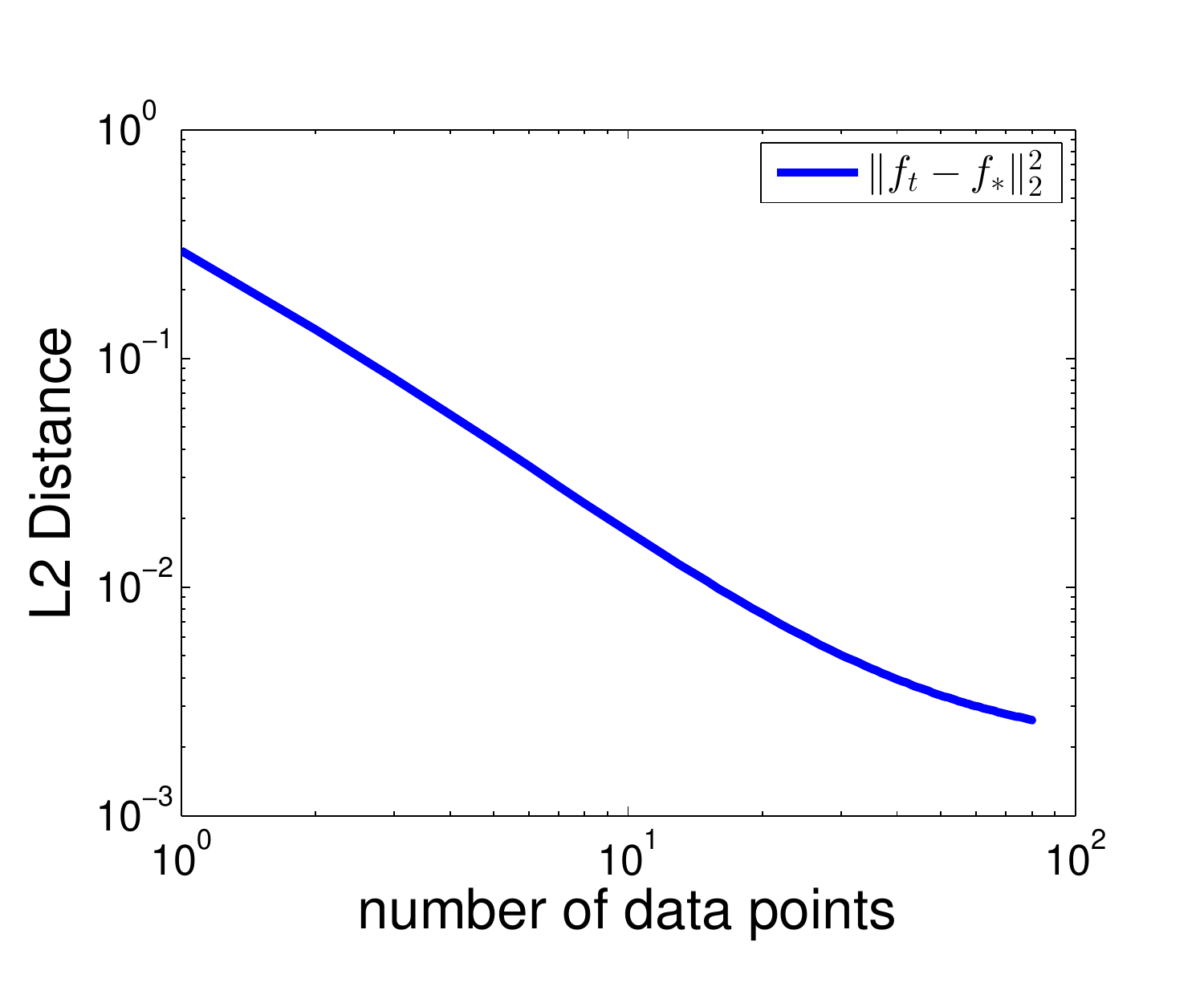} \vspace{-0.1in}\\

  (5) Iteration 20 & (6) Iteration 50 & (7) Iteration 80 &\\
          &         &         &  (8)Posterior Convergence
  \end{tabular}
  \caption{Visualization of posterior prediction distribution. The red curve is the mean function and the pale red region is the variance of the posterior. The cyan curve the ground truth. The last one shows convergence of the posterior mean to the ground truth.}
  \label{fig:syn_data_GP}
\end{figure*}

We test the proposed algorithm on 1D synthetic data. The data are generated by 
$$
y = 3x^2 + (\sin(3.53\pi x) + \cos(7.7\pi x))\exp(-1.6\pi|x|) + 0.1e
$$ 
where $x\in[-0.5, 0.5]$ and $e\sim\Ncal(0, 1)$. The dataset contains $2048$ observations which is small enough to run the exact GP regression. We use Gaussian RBF kernel in Gaussian processes and sparse Gaussian processes. Since we are comparing different inference algorithms on the same model, we use the same hyperparameters for all the inference algorithms. We set the kernel bandwidth $\sigma$ to be $0.1$ times the median of pairwise distances between data points (median trick), and $\beta^{-1} = 0.001$. We set the stepsize in the form of $\frac{\eta}{n_0 + \sqrt{t}}$ for both PMD and SVI and the batch size to be $128$. Figure.~\ref{fig:syn_data_GP} illustrates the evolving of the posterior provided by PMD with $16$ particles and $128$ inducing variables when the algorithms visit more and more data. To illustrate the convergence of the posterior provided by PMD, we initialize the $\ub = 0$ in PMD. Later, we will see we could make the samples in PMD more efficient.

\subsubsection{Music Year Prediction}
We randomly selected $463,715$ songs to train the model and test on $5,163$ songs. As in~\cite{BerEllWhiLam11}, the year values are linearly mapped into $[0, 1]$. The data is standardized before regression. Gaussian RBF kernel is used in the model. Since we are comparing the inference algorithms, for fairness, we fixed the model parameters for all the inference algorithms, \ie, the kernel bandwidth is set to be the median of pairwise distances between data points and the observations precision $\beta^{-1} = 0.01$. We set the number of inducing inputs to be $2^{10}$ and batch size to be $512$. The stepsize for both PMD and SVI are in the form of $\frac{\eta}{n_0+\sqrt{t}}$. To demonstrate the advantages of PMD comparing to SMC, we initialize PMD with prior while SMC with the SoD solution. We rerun the experiments $10$ times. We use both $16$ particles in SMC and PMD.  We stop the stochastic algorithms after they pass through the whole dataset $2$ times.

\subsection{Latent Dirichlet Allocation}

\begin{figure*}[!ht]
\centering
  \begin{tabular}{ccc}
  \includegraphics[width=0.33\textwidth]{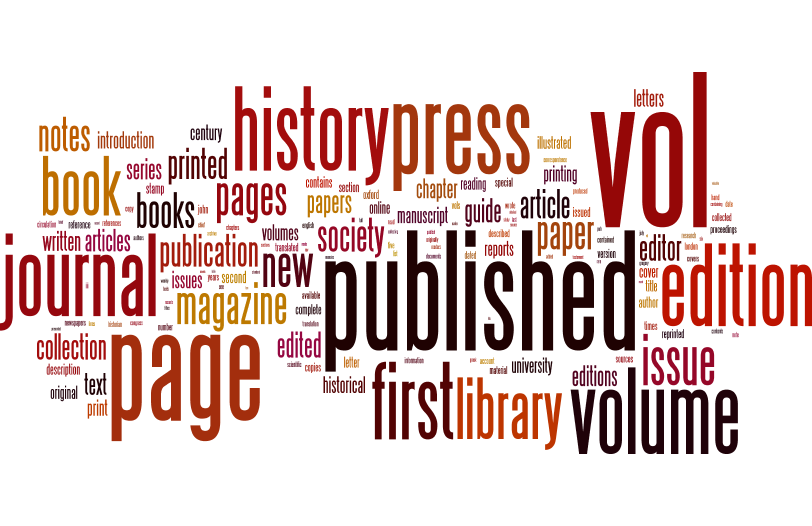} &
  \includegraphics[width=0.33\textwidth]{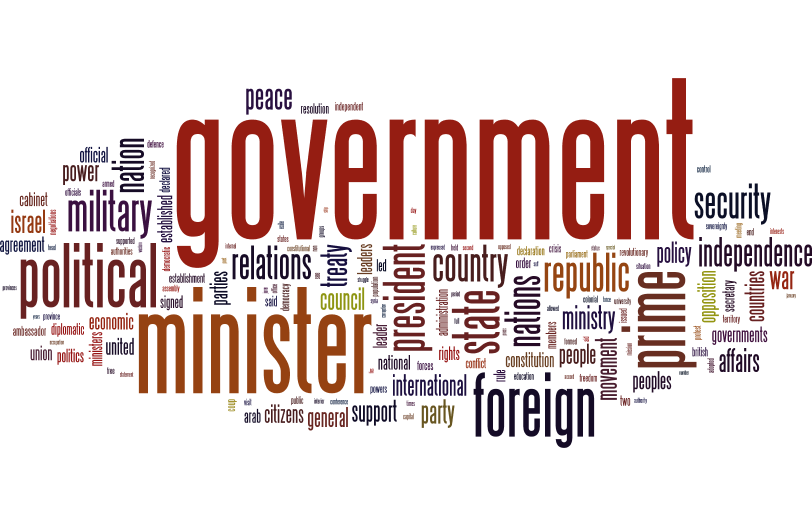} &
  \includegraphics[width=0.33\textwidth]{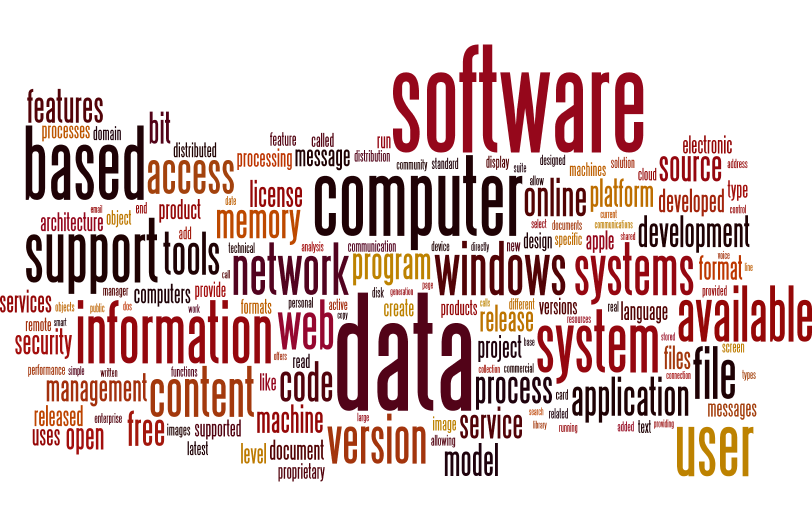}\\
  \includegraphics[width=0.33\textwidth]{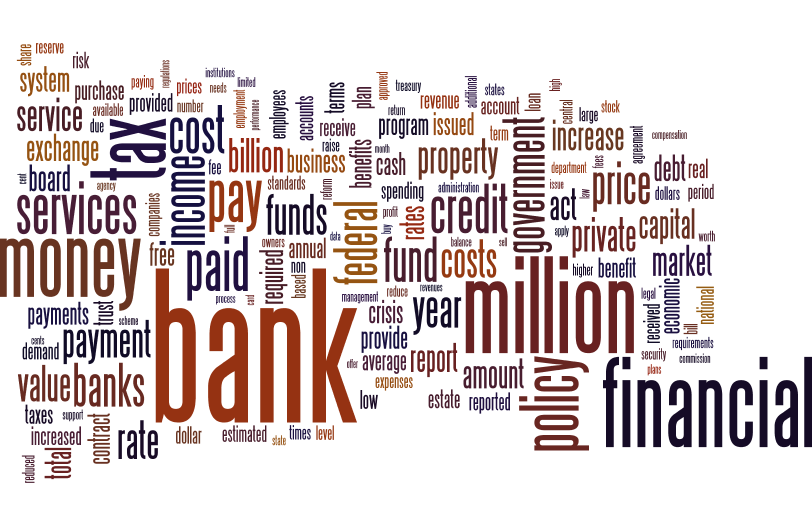} &
  \includegraphics[width=0.33\textwidth]{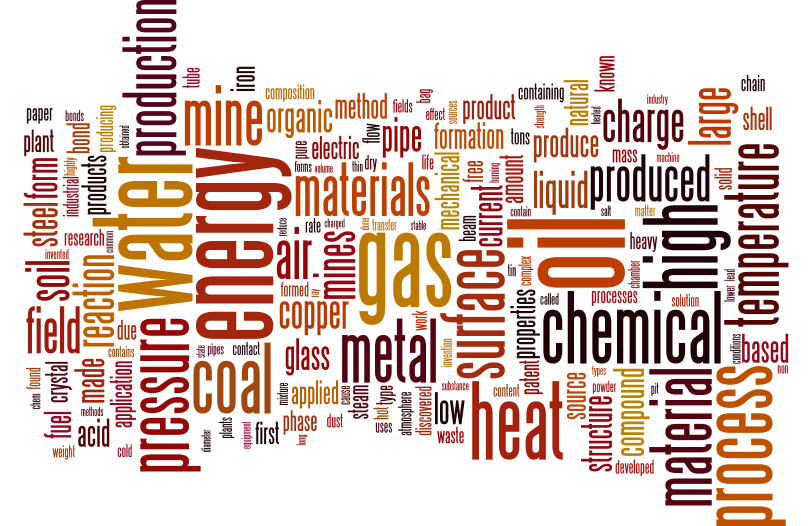} &
  \includegraphics[width=0.33\textwidth]{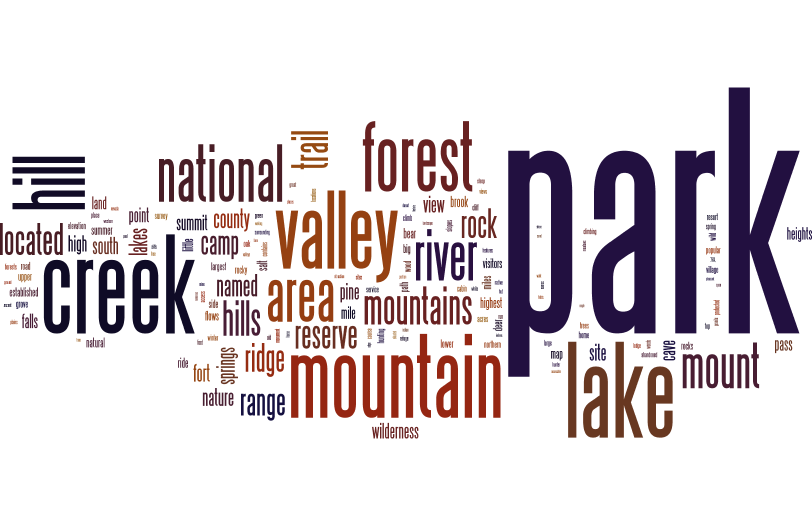}\\ 
  \end{tabular}
  \caption{Several topics learnd by LDA with PMD}
  \label{fig:LDA_topics}
\end{figure*}

We fix the hyper-parameter $\alpha = 0.1$, $\beta=0.01$, and $K=100$. The batchsize is set to be $100$. We use stepsize $\frac{\eta}{n_0 + {t}^\kappa}$ for PMD, stochastic variational inference and stochastic Riemannian Langevin dynamic. For each algorithm a grid-search was run on step-size parameters and the best performance is reported. We stop the stochastic algorithms after they pass through the whole dataset $5$ times.

The log-perplexity was estimated using the methods discussed in~\cite{PatTeh2013} on a separate holdout set with $1000$ documents. For a document $x_d$ in holdout set, the perplexity is computed by
$$
\text{perp}(x_d|X, \alpha, \beta) = \exp\bigg(-\frac{\sum_{n=1}^{N_d}\log p(x_{dn}|X, \alpha, \beta)}{N_d}\bigg)
$$
where 
\begin{eqnarray}\label{eq:lda_perplexity}
p(x_{dn}|X, \alpha, \beta) = \EE_{\theta_d, \Phi}\bigg[\theta_d^\top \Phi_{\cdot, x_{dn}}\bigg].
\end{eqnarray}

We separate the documents in testing set into two non-overlapped parts, $x_d^{\text{estimation}}$ and $x_d^{\text{evaluation}}$. We first evaluate the $\theta_d$ based on the $x_d^{\text{estimation}}$. For different inference methods, we use the corresponding strategies in learning algorithm to obtain the distribution of $\theta_d$ based on $x_d^{\text{estimation}}$. We evaluate $p(x_{dn}|X, \alpha, \beta)$ on $x_d^{\text{evaluation}}$ with the obtained distribution of $\theta_d$. Specifically, 
\begin{eqnarray*}
p(x_{dn}^{\text{evaluation}}|X, \alpha, \beta) = \EE_{\Phi|X, \beta}\EE_{\theta_d^{\text{evaluation}}|\Phi, \alpha, x_d^{\text{estimation}}} \bigg[\theta_d^\top \Phi_{\cdot, x_{dn}}\bigg]
\end{eqnarray*}

For PMD, SMC and stochastic Langevin dynamics, 
\begin{eqnarray*}
\theta_{dk}^{\text{evaluation}} = \frac{\sum_{n=1}^{N_d^{\text{estimation}}}\delta(z_{dnk}^{\text{estimation}} = 1) + \alpha}{N_d^{\text{estimation}} + K\alpha}
\end{eqnarray*}
For stochastic variational inference, $q(\theta_d)$ is updated as in the learning procedure.

We illustrate several topics learned by LDA with our algorithm in Figure.\ref{fig:LDA_topics}.

\end{document}